%% file: arxiv_how_much_data.tex
\documentclass[11pt]{article}

\usepackage{vitercik}
\usepackage{bm}
\usepackage{amsmath,amssymb}
\usepackage{graphicx}
\usepackage{color}
\usepackage{subcaption}
\usepackage{todonotes}
\usepackage[hidelinks]{hyperref}
\usepackage{bibentry}

\newcommand \ind [1]{\mathbb{I}_{\left\{#1\right\}}}
\newcommand \configs {\pazocal{P}}

\newcommand \cY {\pazocal{Y}}
\newcommand \cL {\pazocal{L}}
\newcommand{\score}{\texttt{score}}

\newcommand{\pspace}{\pazocal{P}}

\newcommand{\reals}{\mathbb{R}}

\newcommand{\alignment}{L}

\newcommand{\tree}{\pazocal{T}}

\newcommand{\slin}{u}
\def\qp{A}

\newcommand \mt {\textsc{mt}}
\newcommand \ms {\textsc{ms}}
\newcommand \id {\textsc{id}}
\newcommand \gp {\textsc{gp}}

\newcommand \numseqs {\kappa}
\newcommand \gapchar {-}
\newcommand \height [1]{\operatorname{height}\left(#1\right)}

	\title{How much data is sufficient to learn high-performing algorithms?
		Generalization guarantees for data-driven algorithm design}
\author{
	Maria-Florina Balcan \\ \small Carnegie Mellon University \\ \small \texttt{ninamf@cs.cmu.edu}
	\and 
	Dan DeBlasio \\ \small University of Texas at El Paso \\ \small \texttt{dfdeblasio@utep.edu}
	\and 
	Travis Dick \\ \small University of Pennsylvania \\ \small \texttt{tbd@seas.upenn.edu}
	\and 
	Carl Kingsford \\ \small Carnegie Mellon University \\ \small Ocean Genomics, Inc. \\ \small \texttt{carlk@cs.cmu.edu}
	\and 
	Tuomas Sandholm \\ \small Carnegie Mellon University \\
	\small Optimized Markets, Inc.\\
	\small Strategic Machine, Inc.\\
	\small Strategy Robot, Inc.\\
	\small \texttt{sandholm@cs.cmu.edu}
	\and 
	Ellen Vitercik \\ \small Carnegie Mellon University \\ \small \texttt{vitercik@cs.cmu.edu}}

\begin{document}

\maketitle
\begin{abstract}
	\input{abstract}
\end{abstract}

\section{Introduction}
\input{intro}

\section{Notation and problem statement}\label{sec:prelim}
\input{prelim}

\section{Generalization guarantees for data-driven algorithm design}\label{sec:general}
\input{general_theorem_new}

\section{Parameterized computational biology algorithms}\label{sec:bio}

\input{bio_intro}
\subsection{Sequence alignment}\label{sec:sequence}
\input{sequence}
\subsubsection{Progressive multiple sequence alignment}\label{sec:multisequence}
\input{multisequence}

\subsection{RNA folding}\label{sec:RNA}
\input{RNA}
\subsection{Prediction of topologically associating domains}\label{sec:TAD}
\input{TAD}
\section{Parameterized voting mechanisms}\label{sec:econ}
\input{econ}
\section{Subsumption of prior research on generalization guarantees}\label{sec:connections}
\input{prior_research}

\section{Experiments}\label{sec:experiments}
\input{experiments}
	
	\subsection{Sequence alignment experiments}\label{sec:experiments_sequence}
	\input{experiments_bio}
	\subsection{Mechanism design experiments}\label{sec:mechanism}
	\input{experiments_auctions}

\section{Conclusions}
\input{conclusion}

\paragraph{Acknowledgments.}
This research is funded in part by
the Gordon and Betty Moore Foundation's Data-Driven Discovery Initiative (GBMF4554 to C.K.),
the US National Institutes of Health (R01GM122935 to C.K.),
the US National Science Foundation (a Graduate Research Fellowship to E.V. and grants IIS-1901403 to M.B. and T.S., IIS-1618714, CCF-1535967, CCF-1910321, and SES-1919453 to M.B., IIS-1718457, IIS-1617590, and CCF-1733556 to T.S., and DBI-1937540 to C.K.),
the US Army Research Office (W911NF-17-1-0082 and W911NF2010081 to T.S.),
the Defense Advanced Research Projects Agency under cooperative agreement HR00112020003 to M.B.,
an AWS Machine Learning Research Award to M.B.,
an Amazon Research Award to M.B.,
a Microsoft Research Faculty Fellowship to M.B.,
a Bloomberg Research Grant to M.B.,
a fellowship from Carnegie Mellon University’s Center for Machine Learning and Health to E.V.,
and by the generosity of Eric and Wendy Schmidt by recommendation of the Schmidt Futures program.

\bibliographystyle{plainnat}
\bibliography{bibliography.bib}

\appendix

\section{Helpful lemmas}
\input{appendix_helpful}

\section{Additional details about our general theorem}\label{app:general}
\input{appendix_general}

\section{Additional details about sequence alignment (Section~\ref{sec:sequence})}\label{app:sequence}
\input{appendix_sequence}
\subsection{Additional details about progressive multiple sequence alignment (Section~\ref{sec:multisequence})}\label{app:multisequence}
\input{appendix_multisequence}

\end{document}

%% file: abstract.tex
Algorithms often have tunable parameters that impact performance metrics such as runtime and solution quality. For many algorithms used in practice, no parameter settings admit meaningful worst-case bounds, so the parameters are made available for the user to tune. Alternatively, parameters may be tuned implicitly within the proof of a worst-case approximation ratio or runtime bound. Worst-case instances, however, may be rare or nonexistent in practice. A growing body of research has demonstrated that \emph{data-driven algorithm design} can lead to significant improvements in performance. This approach uses a \emph{training set} of problem instances sampled from an unknown, application-specific distribution and returns a parameter setting with strong average performance on the training set.

We provide a broadly applicable theory for deriving \emph{generalization guarantees} that bound the difference between the algorithm's average performance over the training set and its expected performance on the unknown distribution. Our results apply no matter how the parameters are tuned, be it via an automated or manual approach. The challenge is that for many types of algorithms, performance is a volatile function of the parameters: slightly perturbing the parameters can cause a large change in behavior. Prior research~\citep[e.g.,][]{Gupta17:PAC,Balcan17:Learning,Balcan18:General,Balcan18:Learning} has proved generalization bounds by employing case-by-case analyses of greedy algorithms, clustering algorithms, integer programming algorithms, and selling mechanisms. We uncover a unifying structure which we use to prove extremely general guarantees, yet we recover the bounds from prior research. Our guarantees, which are tight up to logarithmic factors in the worst case, apply whenever an algorithm's performance is a piecewise-constant, -linear, or---more generally---\emph{piecewise-structured} function of its parameters. Our theory also implies novel bounds for voting mechanisms and dynamic programming algorithms from computational biology.

%% file: intro.tex
	Algorithms often have tunable parameters that impact performance metrics such as runtime, solution quality, and memory usage.
These parameters may be set explicitly, as is often the case in applied disciplines.  For example, integer programming solvers expose over one hundred parameters for the user to tune.
There may not be parameter settings that admit meaningful worst-case bounds, but after careful parameter tuning, these algorithms can quickly find solutions to computationally challenging problems. However, applied approaches to parameter tuning have rarely come with provable guarantees.
	Alternatively, an algorithm's parameters may be set implicitly, as is often the case in theoretical computer science: a proof may implicitly optimize over a parameterized family of algorithms in order to guarantee a worst-case approximation factor or runtime bound. Worst-case bounds, however, can be overly pessimistic in practice.
	A growing body of research (surveyed in a book chapter by~\citet{Balcan20:Data}) has demonstrated the power of \emph{data-driven algorithm design}, where machine learning is used to find parameter settings that work particularly well on problems from the application domain at hand.

We present a broadly applicable theory for proving \emph{generalization guarantees} in the context of data-driven algorithm design. We adopt a natural learning-theoretic model of data-driven algorithm design introduced by \citet{Gupta17:PAC}. As in the applied literature on automated algorithm configuration~\cite[e.g.,][]{Horvitz01:Bayesian,Sandholm13:Very-Large-Scale,Xu08:SATzilla,Hutter09:Paramils,Leyton-Brown09:Empirical,Kadioglu10:ISAC,Xu11:Hydra-MIP}, we assume there is an unknown, application-specific distribution over the algorithm's input instances.  A learning procedure receives a \emph{training set} sampled from this distribution and returns a parameter setting---or \emph{configuration}---with strong average performance over the training set. If the training set is too small, this configuration may have poor expected performance. \emph{Generalization guarantees} bound the difference between  average performance over the training set and actual expected performance. Our guarantees apply no matter how the parameters are optimized, via an algorithmic search as in \emph{automated algorithm configuration}~\cite[e.g.,][]{Sandholm13:Very-Large-Scale,Xu08:SATzilla,Xu11:Hydra-MIP,DeBlasio18:Parameter}, or manually as in \emph{experimental algorithmics}~\citep[e.g.,][]{Bentley84:Some,Mcgeoch12:Guide,Iyer01:Experimental}.

Across many types of algorithms---for example, combinatorial algorithms, integer programs, and dynamic programs---the algorithm's performance is a volatile function of its parameters. This is a key challenge that distinguishes our results from prior research on generalization guarantees. For well-understood functions in machine learning theory, there is generally a simple connection between a function's parameters and the value of the function. Meanwhile, slightly perturbing an algorithm's parameters can cause significant changes in its behavior and performance. To provide generalization bounds, we uncover structure that governs these volatile performance functions.

The structure we discover depends on the relationship between \emph{primal} and \emph{dual} functions~\citep{Assouad83:Densite}. To derive generalization bounds, a common strategy is to calculate the \emph{intrinsic complexity} of a function class $\cU$ which we refer to as the \emph{primal class}. Every function $u_{\vec{\rho}}\in \cU$ is defined by a parameter setting $\vec{\rho} \in \R^d$ and $u_{\vec{\rho}}(x) \in \R$ measures the performance of the algorithm parameterized by $\vec{\rho}$ given the input $x$.
We measure intrinsic complexity using the classic notion of \emph{pseudo-dimension}~\citep{Pollard84:Convergence}.
This is a challenging task because the domain of every function in $\cU$ is a set of problem instances, so there are no obvious notions of Lipschitz continuity or smoothness on which we can rely.
Instead, we use structure exhibited by the \emph{dual class} $\cU^*$. Every \emph{dual function} $u_x^* \in \cU^*$ is defined by a problem instance $x$ and measures the algorithm's performance as a function of its parameters given $x$ as input. The dual functions have a simple, Euclidean domain $\R^d$ and we demonstrate that they have ample structure which we can use to bound the pseudo-dimension of $\cU$.

\subsection{Our contributions}\label{sec:contributions}

Our results apply to any parameterized algorithm with dual functions that exhibit a clear-cut, ubiquitous structural property: the duals are piecewise constant, piecewise linear, or---more broadly---piecewise \emph{structured}. The parameter space decomposes into a small number of regions such that within each region,
the algorithm's performance is ``well behaved.''
\begin{figure}
	\centering
	\includegraphics[scale =1]{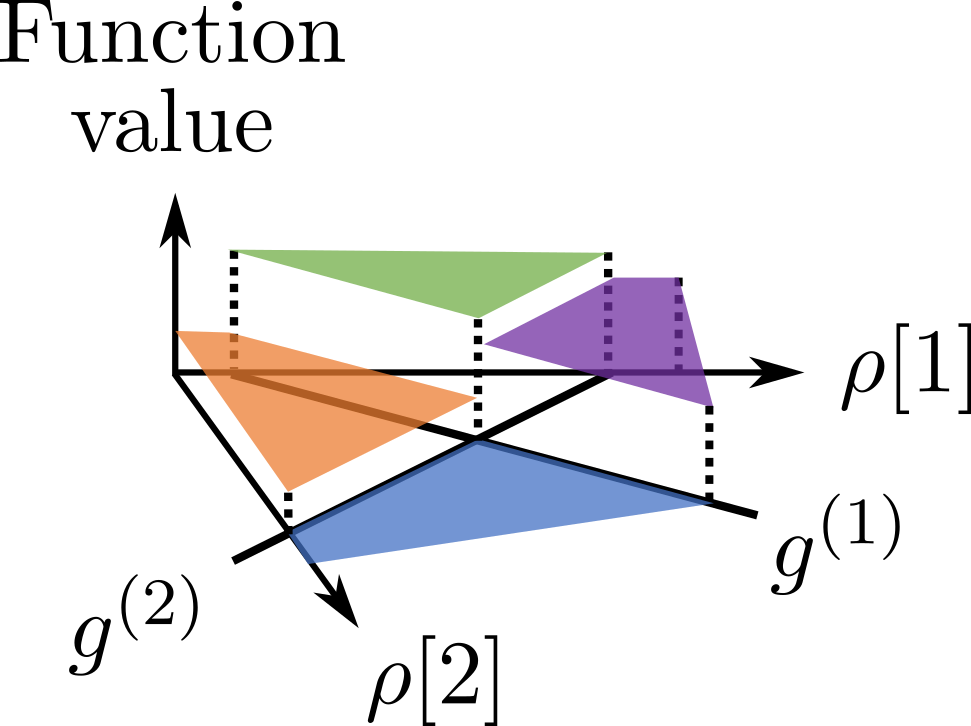}
	\caption{A piecewise-constant function over $\R^2_{\geq 0}$ with linear boundary functions $g^{(1)}$ and $g^{(2)}$.}
	\label{fig:decomp}
\end{figure}
As an example, Figure~\ref{fig:decomp} illustrates a piecewise-structured
function of two parameters $\rho[1]$ and $\rho[2]$. There are two functions $g^{(1)}$ and $g^{(2)}$ that define a partition of the parameter space and four constant functions that define the function value on each subset from this partition.

More formally, the dual class $\cU^*$ is \emph{$(\cF, \cG, k)$-piecewise decomposable} if for every problem instance, there are at most $k$ boundary functions from a set $\cG$ (for example, the set of linear separators) that partition the parameter space into regions such that within each region, algorithmic performance is defined by a function from a set $\cF$ (for example, the set of constant functions).
We bound the pseudo-dimension of $\cU$ in terms of the pseudo- and VC-dimensions of the dual classes $\cF^*$ and $\cG^*$, denoted $\pdim\left(\cF^*\right)$ and $\VC\left(\cG^*\right)$. This yields our main theorem: if $[0,H]$ is the range of the functions in $\cU$, then with probability $1-\delta$ over the draw of $N$
	training instances, for any parameter setting, the difference between the algorithm's average performance over the training set and its expected performance is $\tilde O\left(H \sqrt{\frac{1}{N}\left(\pdim\left(\cF^*\right) + \VC\left(\cG^*\right) \ln k + \ln\frac{1}{\delta}\right)}\right).$ Specifically, we prove that $\pdim(\cU) = \tilde O\left(\pdim\left(\cF^*\right) + \VC\left(\cG^*\right)\ln k\right)$ and that this bound is tight up to log factors. The classes $\cF$ and $\cG$ are often so well structured that bounding $\pdim\left(\cF^*\right)$ and $\VC\left(\cG^*\right)$ is straightforward.

	This is the most broadly applicable generalization bound for data-driven algorithm design in the distributional learning model that applies to arbitrary input distributions.
	A nascent line of research~\cite{Gupta17:PAC,Balcan17:Learning,Balcan18:General,Balcan18:Learning,Balcan18:Dispersion,Balcan20:LearningToLink}
	provides generalization bounds for a selection of parameterized algorithms.
	Unlike the results in this paper, those papers analyze each algorithm individually,
	case by case.
	Our approach recovers those bounds, implying guarantees for configuring greedy algorithms~\cite{Gupta17:PAC}, clustering algorithms~\citep{Balcan17:Learning}, and integer programming algorithms~\citep{Balcan17:Learning,Balcan18:Learning}, as well as mechanism design for revenue maximization~\citep{Balcan18:General}.	We also derive novel generalization bounds for computational biology algorithms and voting mechanisms.
	
\paragraph{Proof insights.} At a high level, we prove this guarantee by counting the number of parameter settings with significantly different performance over any set $\sample$ of problem instances. To do so, we first count the number of regions induced by the $|\sample|k$ boundary functions  that correspond to these problem instances. This step subtly depends not on the VC-dimension of the class of boundary functions $\cG$, but rather on $\VC\left(\cG^*\right)$. These $|\sample|k$ boundary functions partition the parameter space into regions where across all instances $x$ in $\sample$, the dual functions $u_x^*$ are simultaneously structured. Within any one region, we use the pseudo-dimension of the dual class $\cF^*$ to count the number of parameter settings in that region with significantly different performance. We aggregate these bounds over all regions to bound the pseudo-dimension of $\cU$.

\paragraph{Parameterized dynamic programming algorithms from computational biology.}
Our results imply bounds for a variety of computational biology algorithms that are used in practice.
We analyze parameterized sequence alignment algorithms~\citep{Gotoh87:Improved,Gusfield94:Parametric, FernandezBaca04:Parametric, Pachter04:Tropical, Pachter04:Parametric} as well as RNA folding algorithms~\cite{Nussinov80:Fast}, which predict how an input RNA strand  would  naturally  fold,  offering  insight  into  the   molecule’s  function.   We also provide guarantees for algorithms that predict topologically associating domains in DNA sequences~\cite{Filippova14:Identification}, which shed light on how DNA wraps into three-dimensional structures that influence genome function.

\paragraph{Parameterized voting mechanisms.} A \emph{mechanism} is a special type of algorithm designed to help a set of agents come to a collective decision. For example, a town's residents may want to build a public resource such as a park, pool, or skating rink, and a mechanism would help them decide which to build. We analyze \emph{neutral affine maximizers}~\cite{Roberts79:Characterization,Mishra12:Roberts,Nath19:Efficiency}, a well-studied family of parameterized mechanisms. The parameters can be tuned to maximize social welfare, which is the sum of the agents' values for the mechanism's outcome.

\subsection{Additional related research}
A growing body of theoretical research investigates how machine learning can be incorporated in the process of algorithm design~\citep{Chawla20:Pandoras,Weisz18:LeapsAndBounds,Weisz19:CapsAndRuns,Kleinberg17:Efficiency,Kleinberg19:Procrastinating,Alabi19:Learning,Lykouris18:Competitive,Purohit18:Improving,Mitzenmacher18:Model,Hsu19:Learning,Blum21:Learning,Gupta17:PAC,Balcan17:Learning,Balcan18:General,Balcan18:Learning,Balcan18:Dispersion,Balcan20:LearningToLink,Bamas20:Primal,Wei20:Optimal,Bamas20:Learning,Garg18:Supervising,Balcan20:Semi}.
A chapter by \citet{Balcan20:Data} provides a survey. We highlight a few of the papers that are most related to ours below.
\subsubsection{Prior research}

\paragraph{Runtime optimization with provable guarantees.}
\citet{Kleinberg17:Efficiency,Kleinberg19:Procrastinating} and \citet{Weisz18:LeapsAndBounds,Weisz19:CapsAndRuns} provide configuration procedures with provable guarantees when the goal is to minimize runtime.
In contrast, our bounds apply to arbitrary performance metrics, such as solution quality as well as runtime.
Also, their procedures are designed for the case where the set of parameter settings is finite (although they can still offer some guarantees when the parameter space is infinite by first sampling a finite set of parameter settings and then running the configuration procedure; \citet{Balcan18:Learning,Balcan20:Learning} study what kinds of guarantees discretization approaches can and cannot provide). In contrast, our guarantees apply immediately to infinite parameter spaces.  Finally, unlike our results, the guarantees from this prior research are not configuration-procedure-agnostic: they apply only to the specific procedures that are proposed.

\paragraph{Learning-augmented algorithms.}
A related line of research has designed algorithms that replace some steps of a classic worst-case algorithm with a machine-learned oracle that makes predictions about structural aspects of the input~\citep{Lykouris18:Competitive,Purohit18:Improving,Mitzenmacher18:Model,Hsu19:Learning}.
If the prediction is accurate, the algorithm's performance (for example, its error or runtime) is superior to the original worst-case algorithm, and if the prediction is incorrect, the algorithm performs as well as that worst-case algorithm. Though similar, our approach to data-driven algorithm design is different because we are not attempting to learn structural aspects of the input; rather, we are optimizing the algorithm's parameters directly using the training set. Moreover, we can also compete with the best-known worst-case algorithm by including it in the algorithm class over which we optimize. Just adding that one extra algorithm---however different---does not increase our sample complexity bounds.  That best-in-the-worst-case algorithm does not have to be a special case of the parameterized algorithm.

\paragraph{Dispersion.}
Prior research by \citet*{Balcan18:Dispersion} as well as concurrent research by~\citet*{Balcan20:Semi} provides provable guarantees for algorithm configuration, with a particular focus on online learning and privacy-preserving algorithm configuration. These tasks are impossible in the worst case, so these papers identify a property of the dual functions under which online and private configuration are possible. This condition is \emph{dispersion}, which, roughly speaking, requires that the discontinuities of the dual functions are not too concentrated in any ball.  Online learning guarantees imply sample complexity guarantees due to online-to-batch conversion, and \citet{Balcan18:Dispersion} also provide sample complexity guarantees based on dispersion using Rademacher complexity.

To prove that dispersion holds, one typically needs to show that under the distribution over problem instances, the dual functions' discontinuities do not concentrate. This argument is typically made by assuming that the distribution is sufficiently nice or---when applicable---by appealing to the random nature of the parameterized algorithm. Thus, for arbitrary distributions and deterministic algorithms, dispersion does not necessarily hold. In contrast, our results hold even when the discontinuities concentrate, and thus are applicable to a broader set of problems in the distributional learning model. In other words, the results from this paper cannot be recovered using the techniques of \citet{Balcan18:Dispersion,Balcan20:Semi}.

\subsubsection{Concurrent and subsequent research}

Subsequently to the appearance of the original version of this paper in 2019~\citep{Balcan19:How}, an extensive body of research has developed that studies the use of machine learning in the context of algorithm design, as we highlight below.

\paragraph{Learning-augmented algorithms.} The literature on learning-augmented algorithms (summarized in the previous section) has continued to flourish in subsequent research~\citep{Bamas20:Primal,Wei20:Optimal,Bamas20:Learning,Eden21:Learning,Dutting20:Secretaries,Indyk20:Online,Lavastida20:Learnable,Lavastida20:Learnable}. Some of these papers make explicit connections to the types of parameter optimization problems we study in this paper, such as research by \citet{Lavastida20:Learnable}, who study online flow allocation and makespan minimization problems. They formulate the machine-learned predictions as a set of parameters and study the sample complexity of learning a good parameter setting. An interesting direction for future research is to investigate which other problems from this literature can be formulated as parameter optimization algorithms, and whether the techniques in this paper can be used to derive tighter or novel guarantees.

\paragraph{Sample complexity bounds for data-driven algorithm design.} \citet{Chawla20:Pandoras} study a data-driven algorithm design problem for the \emph{Pandora's box problem}, where there is a set of alternatives with costs drawn from an unknown distribution. A search algorithm observes the alternatives' costs one-by-one, eventually stopping and selecting one alternative. 
The authors show how to learn an algorithm that minimizes the selected alternative's expected cost, plus the number of  alternatives the algorithm examines. The primary contributions of that paper are 1) identifying a finite subset of algorithms that compete with the optimal algorithm, and 2) showing how to efficiently optimize over that finite subset of algorithms. Since the authors prove that they only need to optimize over a finite subset of algorithms, the sample complexity of this approach follows from a Hoeffding and union bound.

\citet{Blum21:Learning} study a data-driven approach to learning a nearly optimal cooling schedule for the simulated annealing algorithm. They provide upper and lower sample complexity bounds, with their upper bound following from a careful covering number argument. We leave as an open question whether our techniques can be combined with theirs to match their sample complexity lower bound of $\tilde \Omega\left(\sqrt[3]{m}\right)$, where $m$ is the cooling schedule length.

\paragraph{Machine learning for combinatorial optimization.} A growing body of applied research has developed machine learning approaches to discrete optimization, largely with the aim of improving classic optimization algorithms such as branch-and-bound~\citep{Furian21:Machine, Prouvost20:Ecole,Tang20:Reinforcement,Selsam19:Guiding,Shen19:Lorm,Ferber20:Mipaal,Zarpellon21:Parameterizing,Etheve20:Reinforcement,Wattez20:Learning,Song20:General,Kotary21:End}. For example, \citet{Chmiela21:Learning} present data-driven approaches to scheduling heuristics in branch-and-bound, and they leave as an open question whether the techniques in this paper can be used to provide provable guarantees.

%% file: prelim.tex
We study algorithms parameterized by a set $\configs \subseteq
\R^d$. 
As a concrete example, parameterized algorithms are often used for sequence alignment~\citep{Gotoh87:Improved}. There are many features of an alignment one might wish to optimize, such as the number of matches, mismatches, or indels (defined in Section~\ref{sec:sequence}). A parameterized objective function is defined by weighting these features.
As another example, hierarchical clustering algorithms often use linkage routines such as single, complete, and average linkage. Parameters can be used to interpolate between these three classic procedures~\citep{Balcan17:Learning}, which can be outperformed with a careful parameter tuning~\citep{Balcan20:LearningToLink}.

We use $\cX$ to denote the set of problem
instances the algorithm takes as input. We measure the performance of the algorithm parameterized by $\vec{\rho} = (\rho[1], \dots, \rho[d]) \in \R^d$ via a utility function $u_{\vec{\rho}} :
\cX \to [0,H]$, with $\cU = \left\{u_{\vec{\rho}} : \vec{\rho} \in \configs\right\}$ denoting the set of all such functions.
We assume there is an unknown, application-specific distribution $\dist$ over $\cX$.

Our goal is to find a parameter vector in $\configs$ with high performance in expectation over the distribution $\dist$. We provide \emph{generalization guarantees} for this problem.
Given a training set of problem instances $\sample$ sampled from $\dist$, a generalization guarantee bounds the difference---for any choice of the parameters $\vec{\rho}$---between the average performance of the algorithm over $\sample$ and its expected performance.

Specifically, our main technical contribution is a bound on the \emph{pseudo-dimension}~\citep{Pollard84:Convergence} of the set $\cU$.  For any arbitrary set of functions $\cH$ that map an abstract domain $\cY$ to $\R$, the \emph{pseudo-dimension of $\cH$}, denoted $\pdim(\cH)$, is the size of the largest set $\left\{y_1, \dots, y_N\right\} \subseteq \cY$ such that for some set of \emph{targets} $z_1, \dots, z_N \in \R$, \begin{equation}\left|\left\{
	\left(\sign\left(h\left(y_1\right) - z_1\right), \dots,
	\sign\left(h\left(y_N\right) - z_N\right)\right)
	\,\middle\vert\,
	h \in \cH
	\right\}\right| = 2^N.\label{eq:pdim}\end{equation}  Classic results from learning theory~\citep{Pollard84:Convergence} translate pseudo-dimension bounds into generalization guarantees. For example, suppose $[0,H]$ is the range of the functions in $\cH$. For any $\delta \in (0,1)$ and any distribution $\dist$ over $\cY$, with probability $1-\delta$
over the draw of $\sample\sim \dist^N$, for
all functions $h \in \cH$, the difference between the average value of $h$ over $\sample$ and its expected value is bounded as follows:
\begin{equation}\left|\frac{1}{N}\sum_{y \in \sample} h(y) - \E_{y \sim \dist}\left[h(y)\right]\right| = O\left(H \sqrt{\frac{1}{N}\left(\pdim(\cH) + \ln \frac{1}{\delta}\right)}\right).\label{eq:pollard}\end{equation}
When $\cH$ is a set of binary-valued functions that map $\cY$ to $\{0,1\}$, the pseudo-dimension of $\cH$ is more commonly referred to as the \emph{VC-dimension of $\cH$}~\citep{Vapnik71:Uniform}, denoted $\VC(\cH)$. 

%% file: general_theorem_new.tex
In data-driven algorithm design, there are two closely related function classes. First, for each parameter setting $\vec{\rho} \in \configs$, $u_{\vec{\rho}} : \cX \to \R$ measures performance as a function of the input $x\in \cX$.
Similarly, for each input $x$, there is a function $u_{x} :
\configs \to \R$ defined as $u_x(\vec{\rho}) = u_{\vec{\rho}}(x)$ that
measures performance as a function of the parameter vector $\vec{\rho}$. The set $\left\{u_x   \mid x \in \cX\right\}$ is equivalent to \citeauthor{Assouad83:Densite}'s notion of the \emph{dual class}~\citep{Assouad83:Densite}.

\begin{definition}[Dual class~\citep{Assouad83:Densite}]\label{def:dual}
	For any domain $\cY$ and set of functions $\cH \subseteq \R^{\cY}$,
	the \emph{dual class} of $\cH$ is defined as $\cH^* = \left\{h^*_y : \cH \to \reals \mid y \in \cY\right\}$
	where $h^*_y(h) = h(y)$. Each function $h^*_y \in \cH^*$ fixes an
	input $y \in \cY$ and maps each function $h \in \cH$ to $h(y)$. We refer to the class $\cH$ as the \emph{primal class.}
\end{definition}
The set of functions $\left\{u_x  \mid x \in \cX\right\}$ is equivalent to the dual class $\cU^* = \left\{u_x^* : \cU \to [0,H] \mid x \in \cX\right\}$ in the sense that for every parameter vector $\vec{\rho} \in \cP$ and every problem instance $x \in \cX$, $u_x(\vec{\rho}) = u_x^*\left(u_{\vec{\rho}}\right)$.

Many combinatorial algorithms share a clear-cut, useful structure: for
each instance $x \in \cX$, the function $u_x$ is \emph{piecewise structured}. For example, each function $u_x$ might be piecewise constant with a
small number of pieces. Given the equivalence of the functions $\left\{u_x \mid x \in \cX\right\}$ and the dual class $\cU^*$, the dual class exhibits this piecewise structure as well.
We use this structure to bound the pseudo-dimension of the primal class $\cU$.

Intuitively, a function $h : \cY \to \reals$ is piecewise structured if
we can partition the domain $\cY$ into subsets $\cY_1, \dots,
\cY_M$ so that when we restrict $h$ to one piece $\cY_i$, $h$ equals some well-structured function $f : \cY \to \reals$.  In other words, for all
$y \in \cY_i$, $h(y) = f(y)$. We define the partition $\cY_1, \dots, \cY_M$ using \emph{boundary functions} $g^{(1)}, \dots, g^{(k)} : \cY \to
\{0,1\}$. Each function $g^{(i)}$ divides the domain $\cY$ into two
sets: the points it labels 0 and the points it labels 1.
\begin{figure}
	\centering
\includegraphics[scale =1]{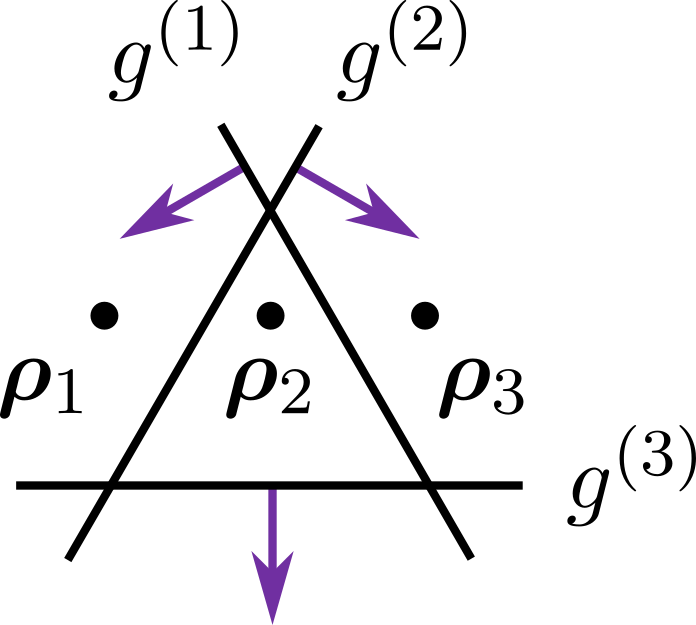}
	\caption{Boundary functions partitioning $\R^2$. The arrows indicate on which side of each function $g^{(i)}(\vec{\rho}) = 0$ and on which side $g^{(i)}(\vec{\rho}) = 1$. For example, $g^{(1)}\left(\vec{\rho}_1\right) = 1$, $g^{(1)}\left(\vec{\rho}_2\right) = 1$, and $g^{(1)}\left(\vec{\rho}_3\right) = 0$.}
	\label{fig:hyperplanes}
\end{figure}
Figure~\ref{fig:hyperplanes} illustrates a partition of $\R^2$ by boundary functions. Together, the $k$
boundary functions partition the domain $\cY$ into at most $2^k$ regions, each
one corresponding to a bit vector $\vec{b} \in \{0,1\}^k$ that describes on which
side of each boundary the region belongs. For each region, we specify a
\emph{piece function} $f_{\vec{b}} : \cY \to \reals$ that defines the function
values of $h$ restricted to that region. Figure~\ref{fig:decomp}
shows an example of a piecewise-structured
function with two boundary functions and four piece functions.

For many parameterized algorithms, every function in the dual class is piecewise structured. Moreover, across dual functions, the boundary functions come from a single, fixed class, as do the piece functions. For example, the
boundary functions might always be halfspace indicator functions, while
the piece functions might always be linear functions. The following
definition captures this structure.

\begin{definition}\label{def:decomposable_sign}
  A function class $\cH \subseteq \reals^{\cY}$ that maps a domain
  $\cY$ to $\reals$ is \emph{$(\cF, \cG, k)$-piecewise decomposable} for a
  class $\cG \subseteq \{0,1\}^{\cY}$ of boundary functions and a class $\cF
  \subseteq \reals^{\cY}$ of piece functions if the following holds: for
  every $h \in \cH$, there are $k$ boundary functions $g^{(1)},
  \dots, g^{(k)} \in \cG$ and a piece function $f_{\vec{b}} \in \cF$ for each
  bit vector $\vec{b} \in \{0,1\}^k$ such that for all $y \in \cY$,
$h(y) = f_{\vec{b}_y}(y)$ where $\vec{b}_y = \left(g^{(1)}(y), \dots, g^{(k)}(y)\right) \in \{0,1\}^k.$
\end{definition}

Our main theorem shows that whenever the dual class $\cU^*$ is $(\cF, \cG,
k)$-piecewise decomposable, we can bound the pseudo-dimension of $\cU$ in terms of the VC-dimension of $\cG^*$ and the pseudo-dimension of
$\cF^*$.
Later, we show that for many common  classes $\cF$ and $\cG$, we can easily bound the complexity of their duals.

\begin{restatable}{theorem}{main}\label{thm:main}
Suppose that the dual function class $\cU^*$ is $(\cF, \cG, k)$-piecewise decomposable with boundary functions $\cG \subseteq \{0,1\}^{\cU}$ and piece functions $\cF \subseteq \R^{\cU}$.
The pseudo-dimension of $\cU$ is bounded as follows:
\[\pdim(\cU) = O\left(\left(\pdim(\cF^*) + \VC(\cG^*)\right)\ln \left(\pdim(\cF^*) + \VC(\cG^*)\right) + \VC(\cG^*)\ln k\right).\]
\end{restatable}

To help make the proof of Theorem~\ref{thm:main} succinct, we extract a key insight in the following lemma. Given a set of functions $\cH$ that map a domain $\cY$ to $\{0,1\}$, Lemma~\ref{lem:regions} bounds the number of binary vectors \begin{equation}\left(h_1(y), \dots, h_N(y)\right)\label{eq:vector}\end{equation} we can obtain for any $N$ functions $h_1, \dots, h_N \in \cH$ as we vary the input $y \in \cY$.  Pictorially, if we partition $\R^2$ using the functions $g^{(1)}$, $g^{(2)}$, and $g^{(3)}$ from Figure~\ref{fig:hyperplanes} for example, Lemma~\ref{lem:regions} bounds the number of regions in the partition.
This bound depends not on the VC-dimension of the class $\cH$, but rather on that of its dual $\cH^*$. We use a classic lemma by~\citet{Sauer72:Density} to prove Lemma~\ref{lem:regions}. \citeauthor{Sauer72:Density}'s lemma~\citep{Sauer72:Density} bounds the number of binary vectors of the form $\left(h\left(y_1\right), \dots, h\left(y_N\right)\right)$ we can obtain for any $N$ elements $y_1, \dots, y_N \in \cY$ as we vary the function $h \in \cH$ by $(eN)^{\VC(\cH)}$. Therefore, it does not immediately imply a bound on the number of vectors from Equation~\eqref{eq:vector}. In order to apply \citeauthor{Sauer72:Density}'s lemma, we must transition to the dual class.

\begin{lemma}\label{lem:regions}
	Let $\cH$ be a set of functions that map a domain $\cY$ to $\{0,1\}$. For any functions $h_1, \dots, h_N \in \cH$, the number of binary vectors $\left(h_1(y), \dots, h_N(y)\right)$ obtained by varying the input $y \in \cY$ is bounded as follows: \begin{equation}\left|\left\{\left(h_1(y), \dots, h_N(y)\right)\,\middle\vert\, y \in \cY \right\}\right| \leq (eN)^{\VC\left(\cH^*\right)}.\label{eq:regions}\end{equation}
	\end{lemma}

\begin{proof}
We rewrite the left-hand-side of Equation~\eqref{eq:regions} as $\left|\left\{\left(h_y^*\left(h_1\right), \dots, h_y^*\left(h_N\right)\right)\,\middle\vert\, y \in \cY \right\}\right|.$ Since we fix $N$ inputs and vary the function $h_y^*$, the lemma statement follows from \citeauthor{Sauer72:Density}'s lemma~\citep{Sauer72:Density}.
		\end{proof}
	
	We now prove Theorem~\ref{thm:main}.

\begin{proof}[Proof of Theorem~\ref{thm:main}]
Fix an arbitrary set of problem instances $x_1, \dots, x_N \in \cX$ and targets $z_1, \dots, z_N
			\in \reals$. We bound the number of ways that $\cU$ can label the problem instances
			$x_1, \dots, x_N$ with respect to the target thresholds $z_1, \dots, z_N \in \R$.
			In other words, as per Equation~\eqref{eq:pdim}, we bound the size of the set \begin{align}\left|\left\{
				\begin{pmatrix}
					\sign\left(u_{\vec{\rho}}\left(x_1\right) - z_1\right)\\
					\vdots\\
					\sign\left(u_{\vec{\rho}}\left(x_N\right) - z_N\right)
				\end{pmatrix}
				\,\middle\vert\,
				\vec{\rho} \in \cP
				\right\}\right| =\left|\left\{
				\begin{pmatrix}
					\sign\left(u_{x_1}^*\left(u_{\vec{\rho}}\right) - z_1\right)\\
					\vdots\\
					\sign\left(u_{x_N}^*\left(u_{\vec{\rho}}\right) - z_N\right)
				\end{pmatrix}
				\,\middle\vert\,
				\vec{\rho} \in \cP
				\right\}\right|\label{eq:first_transition_to_dual}\end{align} by
		$(ekN)^{\VC(\cG^*)} (eN)^{\pdim(\cF^*)}$.
		Then solving for the largest $N$ such
		that \[2^N \leq (ekN)^{\VC(\cG^*)} (eN)^{\pdim(\cF^*)}\] gives a bound on the
		pseudo-dimension of $\cU$. Our bound on Equation~\eqref{eq:first_transition_to_dual} has two main steps:
		\begin{enumerate}
			\item In Claim~\ref{claim:partition}, we show that there are $M <
			\left(ekN\right)^{\VC(\cG^*)}$ subsets $\cP_1, \dots, \cP_M$ partitioning the
			parameter space $\cP$ such that within any one subset, the dual functions
			$u_{x_1}^*, \dots, u_{x_N}^*$ are simultaneously structured. In
			particular, for each subset $\cP_j$, there exist piece functions $f_1,
			\dots, f_N \in \cF$ such that $u_{x_i}^*\left(u_{\vec{\rho}}\right) = f_i\left(u_{\vec{\rho}}\right)$ for all parameter settings $\vec{\rho} \in \cP_j$
			and $i \in [N]$. This is the partition of $\cP$ induced by aggregating all
			of the boundary functions corresponding to the dual functions $u_{x_1}^*,
			\dots, u_{x_N}^*$.
			\item We then show that for any region $\cP_j$ of the partition, as we vary the parameter vector $\vec{\rho} \in \cP_j$, $u_{\vec{\rho}}$ can label the problem instances $x_1, \dots, x_N$   in at
			most $(eN)^{\pdim(\cF^*)}$ ways with respect to the target thresholds $z_1, \dots, z_N$. It follows that the total number of ways that
			$\cU$ can label the problem instances $x_1, \dots, x_N$ is bounded by
			$(ekN)^{\VC(\cG^*)}(eN)^{\pdim(\cF^*)}$.
		\end{enumerate}
		
		We now prove the first claim.
		
		\begin{claim}\label{claim:partition}
			There are $M < \left(ekN\right)^{\VC(\cG^*)}$ subsets $\cP_1, \dots, \cP_M$
			partitioning the parameter space $\cP$ such that within any one subset, the
			dual functions $u_{x_1}^*, \dots, u_{x_N}^*$ are simultaneously structured. In
			particular, for each subset $\cP_j$, there exist piece functions $f_1, \dots,
			f_N \in \cF$ such that $u_{x_i}^*\left(u_{\vec{\rho}}\right) = f_i\left(u_{\vec{\rho}}\right)$ for all parameter settings $\vec{\rho} \in \cP_j$ and $i
			\in [N]$.
		\end{claim}
		
		\begin{proof}[Proof of Claim~\ref{claim:partition}]
			Let $u_{x_1}^*, \dots, u_{x_N}^* \in \cU^*$ be the dual functions
			corresponding to the problem instances $x_1, \dots, x_N$. Since $\cU^*$ is
			$(\cF,\cG,k)$-piecewise decomposable, we know that for each function
			$u_{x_i}^*$, there are $k$ boundary functions $g_i^{(1)}, \dots, g_i^{(k)} \in
			\cG\subseteq \{0,1\}^{\cU}$ that define its piecewise decomposition. Let
			$\hat{\cG} = \bigcup_{i = 1}^N \left\{g_i^{(1)}, \dots, g_i^{(k)}\right\}$ be
			the union of these boundary functions across all $i \in [N]$. For ease of
			notation, we relabel the functions in $\hat{\cG}$, calling them $g_1, \dots,
			g_{kN}$. Let $M$ be the total number of $kN$-dimensional vectors we can obtain
			by applying the functions in $\hat{\cG} \subseteq \{0,1\}^{\cU}$ to elements
			of $\cU$:
			\begin{equation}M:= \left|\left\{\begin{pmatrix} g_1\left(u_{\vec{\rho}}\right)\\
					\vdots\\
					g_{kN}\left(u_{\vec{\rho}}\right)\end{pmatrix} : \vec{\rho} \in \cP\right\}\right|.\label{eq:Mdef}\end{equation}
			By Lemma~\ref{lem:regions}, $M <
			\left(ekN\right)^{\VC(\cG^*)}.$ Let $\vec{b}_1, \dots, \vec{b}_M$ be the binary vectors in the set
			from Equation~\eqref{eq:Mdef}. For each $i \in [M]$, let $\cP_j = \left\{\vec{\rho}
			\mid \left(g_1\left(u_{\vec{\rho}}\right), \dots, g_{kN}\left(u_{\vec{\rho}}\right)\right) = \vec{b}_j\right\}.$
			By construction, for each set $\cP_j$, the values of all the boundary functions $g_1\left(u_{\vec{\rho}}\right), \dots,
			g_{kN}\left(u_{\vec{\rho}}\right)$ are constant as we vary $\vec{\rho} \in \cP_j$. Therefore,
			there is a fixed set of piece functions $f_1, \dots,
			f_N \in \cF$ so that $u_{x_i}^*\left(u_{\vec{\rho}}\right) = f_i\left(u_{\vec{\rho}}\right)$ for all parameter vectors $\vec{\rho} \in \cP_j$
			and indices $i \in [N]$. Therefore, the claim holds.
		\end{proof}
		
Claim~\ref{claim:partition} and Equation~\eqref{eq:first_transition_to_dual} imply that for every subset $\cP_j$ of the partition, \begin{align}\left|\left\{
				\begin{pmatrix}\sign\left(u_{\vec{\rho}}\left(x_1\right) - z_1\right)\\
					\vdots\\
					\sign\left(u_{\vec{\rho}}\left(x_N\right) - z_N\right)
				\end{pmatrix}
				\,\middle\vert\,
				\vec{\rho} \in \cP_j
				\right\}\right| = \left|\left\{
				\begin{pmatrix}
					\sign\left(f_1\left(u_{\vec{\rho}}\right) - z_1\right)\\
					\vdots\\
					\sign\left(f_N\left(u_{\vec{\rho}}\right) - z_N\right)
				\end{pmatrix}
				\,\middle\vert\,
				\vec{\rho} \in \cP_j
				\right\}\right|.\label{eq:piece_functions}\end{align}
		
		Lemma~\ref{lem:regions} implies that Equation~\eqref{eq:piece_functions} is bounded by $(eN)^{\pdim(\cF^*)}$. In other words, for any region $\cP_j$ of the partition, as we vary the parameter vector $\vec{\rho} \in \cP_j$, $u_{\vec{\rho}}$ can label the problem instances $x_1, \dots, x_N$ in at
			most $(eN)^{\pdim(\cF^*)}$ ways with respect to the target thresholds $z_1, \dots, z_N$. Because there are $M < \left(ekN\right)^{\VC(\cG^*)}$ regions $\cP_j$ of the partition, we can conclude  that $\cU$ can label the problem instances
		$x_1, \dots, x_N$ in at most $(ekN)^{\VC(\cG^*)}(eN)^{\pdim(\cF^*)}$ distinct ways
		relative to the targets $z_1, \dots, z_N$. In other words, Equation~\eqref{eq:first_transition_to_dual} is bounded by $(ekN)^{\VC(\cG^*)}(eN)^{\pdim(\cF^*)}$. On the other hand, if $\cU$ shatters
		the problem instances $x_1, \dots, x_N$, then the number of distinct labelings must be
		$2^N$. Therefore, the pseduo-dimension of $\cU$ is at most the largest value of
		$N$ such that $2^N \leq (ekN)^{\VC(\cG^*)}(eN)^{\pdim(\cF^*)}$, which implies that \[N =O\left(\left(\pdim(\cF^*) + \VC(\cG^*)\right)\ln \left(\pdim(\cF^*) + \VC(\cG^*)\right) + \VC(\cG^*)\ln k\right),\] as claimed.
\end{proof}

We prove several lower bounds which show that Theorem~\ref{thm:main} is tight up to logarithmic factors.

\begin{restatable}{theorem}{lb}\label{thm:main_lb}
	The following lower bounds hold:
	\begin{enumerate}
\item 	There is a parameterized sequence alignment algorithm with $\pdim(\cU) = \Omega(\log n)$ for some $n \geq 1$. Its dual class $\cU^*$  is $(\cF, \cG, n)$-piecewise decomposable for classes $\cF$ and $\cG$ with $\pdim\left(\cF^*\right) =\VC\left(\cG^*\right) = 1$.
\item There is a parameterized voting mechanism with $\pdim(\cU) = \Omega(n)$ for some $n \geq 1$. Its dual class $\cU^*$  is $\left(\cF, \cG, 2\right)$-piecewise decomposable for  classes $\cF$ and $\cG$ with $\pdim\left(\cF^*\right) =1$ and $\VC\left(\cG^*\right) = n$.
\end{enumerate}
\end{restatable}

\begin{proof}
	In Theorem~\ref{thm:seq_lb}, we prove the result for sequence alignment, in which case $n$ is the maximum length of the sequences, $\cF$ is the set of constant functions, and $\cG$ is the set of threshold functions. In Theorem~\ref{thm:NAM_lb}, we prove the result for voting mechanisms, in which case $n$ is the number of agents that participate in the mechanism, $\cF$ is the set of constant functions, and $\cG$ is the set of homogeneous linear separators in $\R^n$.
	\end{proof}

\subsection*{Applications of our main theorem to representative function classes} \label{sec:examplePiecewise}

We now instantiate Theorem~\ref{thm:main} in a general setting inspired by data-driven algorithm design.

\paragraph{One-dimensional functions with a bounded number of oscillations.}
Let $\cU = \left\{u_{\rho} \mid \rho \in \R\right\}$ be a set of utility functions defined over a single-dimensional parameter space. We often find that the dual functions are piecewise constant, linear, or polynomial. More generally, the dual functions are piecewise structured with piece functions that oscillate a fixed number of times.
In other words, the dual class $\cU^*$ is $(\cF, \cG, k)$-piecewise decomposable where the boundary functions $\cG$ are thresholds and the piece functions $\cF$ oscillate a bounded number of times, as formalized below.
\begin{definition}\label{def:oscillate}
	A function $h :
	\R \to \reals$ has at most $B$ oscillations if for every $z \in \reals$,
	the function $\rho \mapsto \ind{h\left(\rho\right) \geq z}$ is piecewise constant with at
	most $B$ discontinuities.
\end{definition}

\begin{figure}
	\centering
	\begin{subfigure}{.24\textwidth}
		\includegraphics[scale =1]{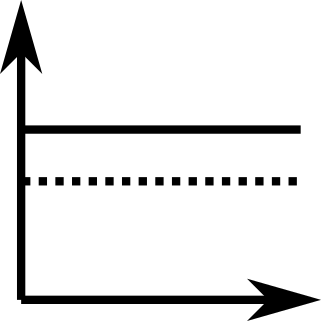}\centering
		\caption{Constant function (zero oscillations).}\label{fig:constant}
	\end{subfigure}\qquad
	\begin{subfigure}{.24\textwidth}
		\includegraphics[scale =1]{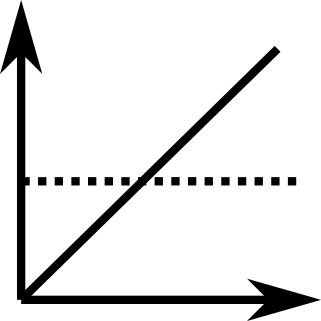}\centering
		\caption{Linear function (one oscillation).}\label{fig:linear}
	\end{subfigure}\qquad
	\begin{subfigure}{.38\textwidth}
		\includegraphics[scale =1]{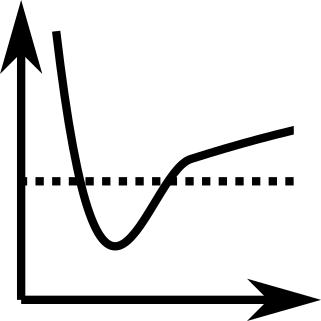}\centering
		\caption{Inverse-quadratic function of the form $h(\rho) = \frac{a}{\rho^2} + b\rho + c$ (two oscillations).}\label{fig:quadratic}
	\end{subfigure}
	\caption{Each solid line is a function with bounded oscillations and each dotted line is an arbitrary threshold. Many parameterized algorithms have piecewise-structured duals with piece functions from these families.}
	\label{fig:examples}
\end{figure}
Figure~\ref{fig:examples} illustrates three common types of functions with bounded oscillations.
In the following lemma, we prove that if $\cH$ is a
class of functions that map $\R$ to $\R$, each of which has at most $B$ oscillations, then $\pdim(\cH^*) = O(\ln B)$.

\begin{lemma}\label{lem:oscillate}
	Let $\cH$ be a
	class of functions mapping $\R$ to $\R$, each of which has at most $B$ oscillations. Then $\pdim(\cH^*) = O(\ln B)$.
\end{lemma}

\begin{proof}  	Suppose that $\pdim\left(\cH^*\right) = N$. Then there exist functions $h_1, \dots, h_N \in \cH$ and witnesses $z_1, \dots, z_N \in \reals$ such that for every subset $T \subseteq [N]$,
	there exists a parameter setting $\rho \in \R$ such that $h_{\rho}^*\left(h_i\right) \geq z_i$ if and only if $i \in T$. We can simplify notation as follows: since $h\left(\rho\right) = h_{\rho}^*\left(h\right)$ for every function $h \in \cH$, we have that for every subset $T \subseteq [N]$,
	there exists a parameter setting $\rho \in \R$ such that $h_i\left(\rho\right) \geq z_i$ if and only if $i \in T$. Let $\cP^*$ be the set of $2^N$ parameter settings corresponding to each subset $T \subseteq [N]$.
	By definition, these parameter settings induce $2^N$ distinct binary vectors as follows:
	\[\left|\left\{\begin{pmatrix}
		\ind{h_1\left(\rho\right) \geq z_1}\\
		\vdots\\
		\ind{h_N\left(\rho\right) \geq z_N}
	\end{pmatrix} : \rho \in \cP^*\right\}\right| = 2^N.
	\]
	
	On the other hand, since each function $h_i$ has at most $B$ oscillations, we can
	partition $\R$ into $M\leq BN+1$ intervals $I_1, \dots, I_M$ such that for
	every interval $I_j$ and every $i \in [N]$, the function $\rho \mapsto \ind{h_i\left(\rho\right) \geq z_i}$ is constant across the interval $I_j$. Therefore, at most one parameter setting $\rho \in \cP^*$ can fall within a single interval $I_j$. Otherwise, if $\rho, \rho' \in I_j \cap \cP^*$, then \[\begin{pmatrix}
		\ind{h_1\left(\rho\right) \geq z_1}\\
		\vdots\\
		\ind{h_N\left(\rho\right) \geq z_N}
	\end{pmatrix} = \begin{pmatrix}
		\ind{h_1\left(\rho'\right) \geq z_1}\\
		\vdots\\
		\ind{h_N\left(\rho'\right) \geq z_N}
	\end{pmatrix},
	\] which is a contradiction.
	As a result, $2^N \leq BN + 1$. The lemma then follows from
	Lemma~\ref{lem:log_ineq}.
\end{proof}

Lemma~\ref{lem:oscillate} implies the following pseudo-dimension bound when the dual function class $\cU^*$ is $(\cF, \cG, k)$-piecewise decomposable, where the boundary functions $\cG$ are thresholds and the piece functions $\cF$ oscillate a bounded number of times.

\begin{restatable}{lemma}{oscillate}\label{lem:oscillate_decomp}
	Let $\cU = \left\{u_{\rho} \mid \rho \in \R\right\}$ be a set of utility functions and suppose the dual class $\cU^*$ is $(\cF, \cG, k)$-decomposable, where the boundary functions $\cG = \left\{g_a \mid a \in \R\right\}$ are thresholds $g_a : u_\rho \mapsto \ind{a \leq \rho}$. Suppose for each $f \in \cF$, the function $\rho \mapsto f\left(u_\rho\right)$ has at most $B$ oscillations. Then $\pdim(\cU) = O((\ln B) \ln(k\ln B))$.
\end{restatable}

\begin{proof}
	First, we claim that $\VC\left(\cG^*\right) = 1$. For a contradiction, suppose $\cG^*$ can shatter two functions $g_a, g_b \in \cG^*$, where $a < b$. There must be a parameter setting $\rho \in \R$ where $g_{u_\rho}^*\left(g_a\right) = g_a\left(u_{\rho}\right) = \ind{a \leq \rho}= 0$ and $g_{u_\rho}^*\left(g_b\right) = g_b\left(u_{\rho}\right) = \ind{b \leq \rho}= 1$. Therefore, $b \leq \rho < a$, which is a contradiction, so $\VC\left(\cG^*\right) = 1$.
	
	Next, we claim that $\pdim\left(\cF^*\right) = O(\ln B)$. For each function $f \in \cF$, let $h_f: \R \to \R$ be defined as $h_f(\rho) = f\left(u_\rho\right)$. By assumption, each function $h_f$ has at most $B$ oscillations. Let $\cH = \left\{h_f \mid f \in \cF\right\}$ and let $N = \pdim\left(\cH^*\right)$. By Lemma~\ref{lem:oscillate}, we know that $N = O(\ln B)$. We claim that $\pdim(\cH^*) \geq \pdim(\cF^*)$. For a contradiction, suppose the class $\cF^*$ can shatter $N+1$ functions $f_1, \dots, f_{N+1}$ using witnesses $z_1, \dots, z_{N+1} \in \R$. By definition, this means that \[\left|\left\{\begin{pmatrix}
		\ind{f_{u_\rho}^*\left(f_1\right) \geq z_1}\\
		\vdots\\
		\ind{f_{u_\rho}^*\left(f_{N+1}\right) \geq z_{N+1}}
	\end{pmatrix} : \rho \in \cP\right\}\right| = 2^{N+1}.\] For any function $f \in \cF$ and any parameter setting $\rho \in \R$, $f_{u_\rho}^*(f) = f\left(u_\rho\right) = h_f(\rho) = h^*_\rho(h_f)$. Therefore, \[\left|\left\{\begin{pmatrix}
		\ind{h_\rho^*\left(h_{f_1}\right) \geq z_1}\\
		\vdots\\
		\ind{h_\rho^*\left(h_{f_{N+1}}\right) \geq z_{N+1}}
	\end{pmatrix} : \rho \in \cP\right\}\right| = \left|\left\{\begin{pmatrix}
		\ind{f_{u_\rho}^*\left(f_1\right) \geq z_1}\\
		\vdots\\
		\ind{f_{u_\rho}^*\left(f_{N+1}\right) \geq z_{N+1}}
	\end{pmatrix} : \rho \in \cP\right\}\right| = 2^{N+1},\] which contradicts the fact that $\pdim(\cH^*) = N$. Therefore, $\pdim(\cF^*) \leq N = O(\ln B)$. The corollary then follows from Theorem~\ref{thm:main}.
\end{proof}

\paragraph{Multi-dimensional piecewise-linear functions.}
In many data-driven algorithm design problems, we find that the boundary functions correspond to halfspace thresholds and the piece functions correspond to constant or linear functions. We handle this case in the following lemma.

\begin{restatable}{lemma}{linear}\label{lem:simple_multidim}
	Let $\cU = \left\{u_{\vec{\rho}} \mid \vec{\rho} \in \cP \subseteq\R^d\right\}$ be a class of utility functions defined over a $d$-dimensional parameter space. Suppose the dual class $\cU^*$ is $(\cF, \cG, k)$-piecewise decomposable, where the boundary functions $\cG = \left\{f_{\vec{a}, \theta} : \cU \to \{0,1\} \mid \vec{a} \in \R^d, \theta \in \R\right\}$ are halfspace indicator functions $g_{\vec{a}, \theta} : u_{\vec{\rho}} \mapsto \ind{\vec{a} \cdot \vec{\rho} \leq \theta}$ and the piece functions $\cF = \left\{f_{\vec{a}, \theta} : \cU \to \R \mid \vec{a} \in \R^d, \theta \in \R\right\}$ are linear functions
	$f_{\vec{a}, \theta} :u_{\vec{\rho}} \mapsto \vec{a} \cdot \vec{\rho} + \theta$. Then $\pdim(\cU) = O(d \ln (dk))$.
\end{restatable}

The proof of this lemma follows from classic VC- and pseudo-dimension bounds for linear functions and can be found in Appendix~\ref{app:general}.

%% file: bio_intro.tex
We study algorithms that are used in practice for three biological problems:
 sequence alignment, RNA folding, and predicting topologically associated domains in DNA.
In these applications, there are two unifying similarities.
First, algorithmic performance is measured in terms of the distance between the algorithm's output and a ground-truth solution.
In most cases, this  solution
is discovered using laboratory experimentation,
so it is only available for the instances in the training set.
Second, these algorithms use dynamic programming to maximize parameterized objective functions.
This objective function represents a surrogate optimization criterion for the dynamic programming algorithm, whereas utility measures how well the algorithm's output resembles the ground truth.
There may be multiple solutions that maximize this objective function, which we call \emph{co-optimal}. Although co-optimal solutions have the same objective function value, they may have different utilities.
To handle tie-breaking, we assume that in any region of the parameter space where the set of co-optimal solutions is fixed, the algorithm's output is also fixed, which is typically true in practice.

%% file: sequence.tex
\subsubsection{Global pairwise sequence alignment}

In pairwise sequence alignment, the goal is to line up strings in order to identify regions of similarity.
In biology, for example, these similar regions indicate functional, structural, or evolutionary relationships between the sequences.
Formally, let $\Sigma$ be an alphabet and let $S_1,S_2 \in \Sigma^n$ be two sequences.
A \emph{sequence alignment} is a pair of sequences $\tau_1, \tau_2 \in (\Sigma \cup \{-\})^*$ such that $\left|\tau_1\right| = \left|\tau_2\right|$,
$\texttt{del}\left(\tau_1\right) = S_1$, and $\texttt{del}\left(\tau_2\right) = S_2,$
where \texttt{del} is a function that deletes every $-$, or \emph{gap character}.
There are many features of an alignment that one might wish to optimize,
such as the number of \emph{matches} ($\tau_1[i] = \tau_2[i]$),
\emph{mismatches} ($\tau_1[i] \not= \tau_2[i]$),
\emph{indels} ($\tau_1[i] = -$ or $\tau_2[i] = -$), and
\emph{gaps} (maximal sequences of consecutive gap characters in $\tau \in \left\{\tau_1, \tau_2\right\}$). 
We denote these features using functions $\ell_1, \dots, \ell_d$ that map pairs of sequences $(S_1, S_2)$ and alignments $L$ to $\R$.

A common dynamic programming algorithm $A_{\vec{\rho}}$~\citep{Gotoh87:Improved,Waterman76:Some} returns the alignment $L$ that maximizes the objective function
\begin{equation}
	\rho[1] \cdot \ell_1\left(S_1, S_2, L\right) + \cdots + \rho[d] \cdot \ell_d\left(S_1, S_2, L\right),\label{eq:alignment}
\end{equation}
where $\vec{\rho} \in \R^d$ is a parameter vector.  We denote the output alignment as $A_{\vec{\rho}}\left(S_1, S_2\right)$.
As \citet*{Gusfield94:Parametric} wrote, ``there is considerable disagreement among molecular biologists about the correct choice'' of a parameter setting $\vec{\rho}$.

We assume that there is a utility function that characterizes an alignment's quality, denoted $u(S_1, S_2, L) \in \R$. For example, $u(S_1, S_2, L)$ might measure the distance between $L$ and a ``ground truth'' alignment of $S_1$ and $S_2$~\citep{Sauder00:Large}. We then define $u_{\vec{\rho}}\left(S_1, S_2\right) = u\left(S_1, S_2, A_{\vec{\rho}}\left(S_1, S_2\right)\right)$ to be the utility of the alignment returned by the algorithm $A_{\vec{\rho}}$.

In the following lemma, we prove that the set of utility functions $u_{\vec{\rho}}$ has piecewise-structured dual functions. 

\begin{restatable}{lemma}{sequence}\label{lem:sequence}
	Let $\cU$ be the set of functions $\cU = \left\{u_{\vec{\rho}} : (S_1, S_2) \mapsto u\left(S_1, S_2, A_{\vec{\rho}}\left(S_1, S_2\right)\right) \mid \vec{\rho} \in \R^d\right\}$ that map sequence pairs $S_1, S_2 \in \Sigma^n$ to $\R$.
	The dual class $\cU^*$ is $\left(\cF, \cG, 4^n n^{4n+2}\right)$-piecewise decomposable, where $\cF = \{f_c : \cU \to \R \mid c \in \R\}$ consists of constant functions $f_c : u_{\vec{\rho}} \mapsto c$ and $\cG = \left\{g_{\vec{a}} : \cU \to \{0,1\} \mid \vec{a} \in \R^d\right\}$ consists of halfspace indicator functions $g_{\vec{a}} : u_{\vec{\rho}} \mapsto \ind{\vec{a} \cdot \vec{\rho} < 0}$.
\end{restatable}

\begin{proof}		
	Fix a sequence pair $S_1$ and $S_2$. Let $\cL$ be the set of alignments the algorithm returns as we range over all parameter vectors $\vec{\rho} \in \R^d$. In other words, $\cL = \left\{A_{\vec{\rho}}(S_1, S_2) \mid \vec{\rho} \in \R^d\right\}$.
	In Lemma~\ref{lem:count}, we prove that $\left|\cL\right| \leq 2^n n^{2n+1}$. For any alignment $\alignment \in \cL$,  the algorithm $A_{\vec{\rho}}$ will return $L$ if and only if \begin{align}\rho[1] \cdot \ell_1\left(S_1, S_2, L\right) + \cdots + \rho[d] \cdot \ell_d\left(S_1, S_2, L\right) > \rho[1] \cdot \ell_1\left(S_1, S_2, L'\right) + \cdots + \rho[d] \cdot \ell_d\left(S_1, S_2, L'\right)\label{eq:generalized_halfspace}
	\end{align}
	for all $\alignment' \in \cL \setminus \{\alignment\}$. Therefore, there is a set $\cH$ of at most ${2^n n^{2n+1} \choose 2} \leq 4^n n^{4n+2}$ hyperplanes such that across all parameter vectors $\vec{\rho}$ in a single connected component of $\R^d \setminus \cH$, the  output of the algorithm parameterized by $\vec{\rho}$, $A_{\vec{\rho}}(S_1, S_2)$, is fixed. (As is standard, $\R^d \setminus \cH$ indicates set removal.) This means that for any connected component $R$ of $\R^d \setminus \cH$, there exists a real value $c_R$ such that $u_{\vec{\rho}}(S_1, S_2) = c_R$ for all $\vec{\rho} \in R$. By definition of the dual, this means that $u_{S_1, S_2}^*\left(u_{\vec{\rho}}\right) = u_{\vec{\rho}}\left(S_1, S_2\right)= c_R$ as well.
	
	We now use this structure to show that the dual class $\cU^*$ is $\left(\cF, \cG, 4^n n^{4n+2}\right)$-piecewise decomposable, as per Definition~\ref{def:decomposable_sign}.
	Recall that $\cG = \left\{g_{\vec{a}} : \cU \to \{0,1\} \mid \vec{a} \in \R^d\right\}$ consists of halfspace indicator functions $g_{\vec{a}} : u_{\vec{\rho}} \mapsto \ind{\vec{a} \cdot \vec{\rho} < 0}$ and $\cF = \{f_c : \cU \to \R \mid c \in \R\}$ consists of constant functions $f_c : u_{\vec{\rho}} \mapsto c$.
	For each pair of alignments $\alignment, \alignment' \in \cL$, let $g^{(\alignment, \alignment')} \in \cG$ correspond to the halfspace represented in Equation~\eqref{eq:generalized_halfspace}. Order these $k := {|\cL| \choose 2}$ functions arbitrarily as $g^{(1)}, \dots, g^{(k)}$.
	Every connected component $R$ of $\R^d \setminus \cH$ corresponds to a sign pattern of the $k$ hyperplanes. For a given region $R$, let $\vec{b}_R \in \{0,1\}^k$ be the corresponding sign pattern. Define the function $f^{\left(\vec{b}_R\right)} \in \cF$ as $f^{\left(\vec{b}_R\right)}= f_{c_R}$, so $f^{\left(\vec{b}_R\right)}\left(u_{\vec{\rho}}\right) =  c_R$ for all $\vec{\rho} \in \R^d$. Meanwhile, for every vector $\vec{b}$ not corresponding to a sign pattern of the $k$ hyperplanes, let $f^{(\vec{b})} = f_0$, so $f^{(\vec{b})}\left(u_{\vec{\rho}}\right) = 0$ for all $\vec{\rho} \in \R^d$.
	In this way, for every $\vec{\rho} \in \R^d$, \[u_{S_1, S_2}^*\left(u_{\vec{\rho}}\right) = \sum_{\vec{b} \in \{0,1\}^{k}} \ind{g^{(i)}\left(u_{\vec{\rho}}\right) = b[i], \forall i \in [k]} f^{(\vec{b})}(u_{\vec{\rho}}),\] as desired.
\end{proof}

Lemmas~\ref{lem:simple_multidim} and \ref{lem:sequence} imply that $\pdim(\cU) = O(nd \ln n + d \ln d).$
	In Appendix~\ref{app:sequence}, we also prove tighter guarantees for a structured subclass of algorithms~\citep{Gotoh87:Improved,Waterman76:Some}. In that case, $d = 4$ and $\ell_1\left(S_1, S_2, L\right)$ is the number of matches in the alignment, $\ell_2\left(S_1, S_2, L\right)$ is the number of mismatches, $\ell_3\left(S_1, S_2, L\right)$ is the number of indels, and $\ell_4\left(S_1, S_2, L\right)$ is the number of gaps. Building on prior research~\citep{Gusfield94:Parametric, FernandezBaca04:Parametric, Pachter04:Parametric}, we show (Lemma~\ref{lem:sequence_tighter}) that the dual class $\cU^*$ is $\left(\cF, \cG, O\left(n^3\right)\right)$-piecewise decomposable with $\cF$ and $\cG$ defined as in Lemma~\ref{lem:sequence}. This implies a pseudo-dimension bound of $O(\ln n)$, which is significantly tighter than that of Lemma~\ref{lem:sequence}. We also prove that this pseudo-dimension bound is tight with a lower bound of $\Omega(\ln n)$ (Theorem~\ref{thm:seq_lb}).
	Moreover, in Appendix~\ref{sec:multisequence}, we provide guarantees for algorithms that align more than two sequences.
	
\subsubsection{Tighter guarantees for a structured algorithm subclass: the affine-gap model} \label{app:affine}
	A line of prior work~\citep{Gusfield94:Parametric, FernandezBaca04:Parametric, Pachter04:Tropical, Pachter04:Parametric} 
	analyzed a specific instantiation of the objective function \eqref{eq:alignment} where $d=3$. In this case, we can obtain a pseudo-dimension bound of $O(\ln n)$, which is exponentially better than the bound implied by Lemma~\ref{lem:sequence}. Given a pair of sequences $S_1, S_2 \in \Sigma^n$, the dynamic programming algorithm $A_{\vec{\rho}}$ returns the alignment $L$ maximizes the objective function \[\mt(S_1, S_2, \alignment) - \rho[1] \cdot \ms(S_1, S_2, \alignment) - \rho[2] \cdot \id(S_1, S_2, \alignment) - \rho[3] \cdot \gp(S_1, S_2, \alignment),\] where $\mt(S_1, S_2, L)$ equals
	the number of matches, $\ms(S_1, S_2, L)$ is the number of mismatches,
	$\id(S_1, S_2, L)$ is the number of indels,
 $\gp(S_1, S_2, L)$ is the number of gaps, and $\vec{\rho} = \left(\rho[1], \rho[2], \rho[3]\right) \in \R^3$ is a parameter vector. We denote the output alignment as $A_{\vec{\rho}}\left(S_1, S_2\right)$. This is known as the \emph{affine-gap scoring model}. 
	We exploit specific structure exhibited by this algorithm family to obtain the exponential pseudo-dimension improvement. This useful structure guarantees that for any pair of sequences $S_1$ and $S_2$, there are only $O\left(n^{3/2}\right)$ different alignments the algorithm family $\left\{A_{\vec{\rho}} \mid \vec{\rho} \in \R^3\right\}$ might produce as we range over parameter vectors~\citep{Gusfield94:Parametric, FernandezBaca04:Parametric, Pachter04:Parametric}. This bound is exponentially smaller than our generic bound of $4^nn^{4n+2}$ from Lemma~\ref{lem:count}.

\begin{lemma}\label{lem:sequence_tighter}
	Let $\cU$ be the set of functions \[\cU = \left\{u_{\vec{\rho} } : (S_1, S_2) \mapsto u\left(S_1, S_2, A_{\vec{\rho}}\left(S_1, S_2\right)\right) \mid  \vec{\rho}\in \R_{\geq 0}\right\}\] that map sequence pairs $S_1, S_2 \in \Sigma^n$ to $\R$.
	The dual class $\cU^*$ is $\left(\cF, \cG, O\left(n^3\right)\right)$-piecewise decomposable, where $\cF = \{f_c : \cU \to \R \mid c \in \R\}$ consists of constant functions $f_c : u_{\vec{\rho}} \mapsto c$ and where $\cG = \{g_{\vec{a}} : \cU \to \{0,1\} \mid \vec{a} \in \R\}$ consists of halfspace indicator functions $g_{\vec{a}} : u_{\vec{\rho}} \mapsto \ind{a[1]\rho[1] + a[2]\rho[2] + a[3]\rho[3] < a[4]}$.
\end{lemma}
	
	\begin{proof}
Fix a sequence pair $S_1$ and $S_2$. Let $\cL$ be the set of alignments the algorithm returns as we range over all parameter vectors $\vec{\rho} \in \R^3$. In other words, $\cL = \{A_{\vec{\rho}}(S_1, S_2) \mid \vec{\rho} \in \R^3\}$.
From prior research~\citep{Gusfield94:Parametric, FernandezBaca04:Parametric, Pachter04:Parametric}, we know that $\left|\cL\right| = O\left(n^{3/2}\right)$. For any alignment $\alignment \in \cL$,  the algorithm $A_{\vec{\rho}}$ will return $L$ if and only if \begin{align*}&\mt(S_1, S_2, \alignment) - \rho[1] \cdot \ms(S_1, S_2, \alignment) - \rho[2] \cdot \id(S_1, S_2, \alignment) - \rho[3] \cdot \gp(S_1, S_2, \alignment)\\
	>\text{ } &\mt(S_1, S_2, \alignment') - \rho[1] \cdot \ms(S_1, S_2, \alignment') - \rho[2] \cdot \id(S_1, S_2, \alignment') - \rho[3] \cdot \gp(S_1, S_2, \alignment')\end{align*}
for all $\alignment' \in \cL \setminus \{\alignment\}$. Therefore, there is a set $\cH$ of at most $O\left(n^3\right)$ hyperplanes such that across all parameter vectors $\vec{\rho}$ in a single connected component of $\R^3 \setminus \cH$, the  output of the algorithm parameterized by $\vec{\rho}$, $A_{\vec{\rho}}(S_1, S_2)$, is fixed. The proof now follows by the exact same logic as that of Lemma~\ref{lem:sequence}.
	\end{proof}
	
Lemmas~\ref{lem:simple_multidim} and \ref{lem:sequence_tighter} imply that $\pdim(\cU) = O(\ln n).$
	We also prove that this pseudo-dimension bound is tight up to constant factors. In this lower bound proof, our utility function $u$ is the \emph{Q score} between a given alignment $L$ of two sequences $(S_1, S_2)$ and the ground-truth alignment $L^*$ (the Q score is also known as the \emph{SPS score} in the case of multiple sequence alignment~\citep{Edgar10:Quality}). The $Q$ score between $L$ and the ground-truth alignment $L^*$ is the fraction of aligned letter pairs in $L^*$ that are correctly reproduced in $L$. For example, the following alignment $L$ has a Q score of $\frac{2}{3}$ because it correctly aligns the two pairs of \texttt{C}s, but not the pair of \texttt{G}s:
	\[L = \begin{bmatrix}
		\texttt{G} & \texttt{A} & \texttt{T} & \texttt{C} & \texttt{C}\\
		\texttt{A} & \texttt{G} & \texttt{-} &\texttt{C} & \texttt{C}\\
	\end{bmatrix} \qquad L^* = \begin{bmatrix}
		\texttt{-} & \texttt{G} & \texttt{A} &\texttt{T} & \texttt{C} & \texttt{C}\\
		\texttt{A} & \texttt{G} &\texttt{-} &\texttt{-} & \texttt{C} & \texttt{C}\\
	\end{bmatrix}.\] We use the notation $u\left(S_1, S_2, L\right) \in [0,1]$ to denote the Q score between $L$ and the ground-truth alignment of $S_1$ and $S_2$. The full proof of the following theorem is in Appendix~\ref{app:sequence}.
	
	\begin{restatable}{theorem}{seqLB}\label{thm:seq_lb}
	There exists a set $\left\{A_{\vec{\rho}} \mid \vec{\rho} \in \R_{\geq 0}^3\right\}$ of co-optimal-constant algorithms and an alphabet $\Sigma$ such that the set of functions $\cU = \left\{u_{\vec{\rho}} : (S_1, S_2) \mapsto u\left(S_1, S_2, A_{\vec{\rho}}\left(S_1, S_2\right)\right) \mid \vec{\rho} \in \R^3_{\geq 0}\right\},$ which map sequence pairs $S_1, S_2 \in \cup_{i = 1}^n\Sigma^i$ of length at most $n$ to $[0,1]$, has a pseudo-dimension of $\Omega(\log n)$.
	\end{restatable}

\begin{proof}[Proof sketch]
	In this proof sketch, we illustrate the way in which two sequences pairs can be shattered, and then describe how the proof can be generalized to $\Theta(\log n)$ sequence pairs.
	
	\paragraph{Setup.} Our setup consists of the following three elements: the alphabet, the two sequence pairs $\left(S_1^{(1)}, S_2^{(1)}\right)$ and $\left(S_1^{(2)}, S_2^{(2)}\right)$, and ground-truth alignments of these pairs. We detail these elements below:
	\begin{enumerate}
	\item  Our alphabet  consists of twelve characters: $\left\{\texttt{a}_i, \texttt{b}_i, \texttt{c}_i, \texttt{d}_i\right\}_{i =1}^3$.
	\item The two sequence pairs are comprised of three subsequence pairs: $\left(t_1^{(1)}, t_2^{(1)}\right)$, $\left(t_1^{(2)}, t_2^{(2)}\right)$, and $\left(t_1^{(3)}, t_2^{(3)}\right)$, where \begin{equation}\begin{matrix}
		t_1^{(1)} = \texttt{a}_1\texttt{b}_1\texttt{d}_1\\
		t_2^{(1)} = \texttt{b}_1\texttt{c}_1\texttt{d}_1
	\end{matrix}, \qquad \begin{matrix}
	t_1^{(2)} = \texttt{a}_2\texttt{a}_2\texttt{b}_2\texttt{d}_2\\
	t_2^{(2)} = \texttt{b}_2\texttt{c}_2\texttt{c}_2\texttt{d}_2
\end{matrix}, \qquad \text{and} \qquad \begin{matrix}
t_1^{(3)} = \texttt{a}_3\texttt{a}_3\texttt{a}_3\texttt{b}_3\texttt{d}_3\\
t_2^{(3)} = \texttt{b}_3\texttt{c}_3\texttt{c}_3\texttt{c}_3\texttt{d}_3
\end{matrix}.\label{eq:subsequence}\end{equation} We define the two sequence pairs as \[\begin{matrix}
S_1^{(1)} = t_1^{(1)}t_1^{(2)}t_1^{(3)} = \texttt{a}_1\texttt{b}_1\texttt{d}_1\texttt{a}_2\texttt{a}_2\texttt{b}_2\texttt{d}_2\texttt{a}_3\texttt{a}_3\texttt{a}_3\texttt{b}_3\texttt{d}_3\\
S_2^{(1)} = t_2^{(1)}t_2^{(2)}t_2^{(3)} =\texttt{b}_1\texttt{c}_1\texttt{d}_1\texttt{b}_2\texttt{c}_2\texttt{c}_2\texttt{d}_2\texttt{b}_3\texttt{c}_3\texttt{c}_3\texttt{c}_3\texttt{d}_3
\end{matrix} \qquad \text{and} \qquad \begin{matrix}S_1^{(2)} = t_1^{(2)} = \texttt{a}_2\texttt{a}_2\texttt{b}_2\texttt{d}_2\\
S_2^{(2)} = t_2^{(2)} =\texttt{b}_2\texttt{c}_2\texttt{c}_2\texttt{d}_2
\end{matrix}.
\]
\item Finally, we define
ground-truth alignments of the sequence pairs $\left(S_1^{(1)}, S_2^{(1)}\right)$ and $\left(S_1^{(2)}, S_2^{(2)}\right)$. We define the ground-truth alignment of $\left(S_1^{(1)}, S_2^{(1)}\right)$ to be
\begin{equation}\begin{array}{cccccccccccccccccccc}\texttt{a}_1&\texttt{b}_1&\texttt{-} & \texttt{d}_1&\texttt{a}_2&\texttt{a}_2&\texttt{b}_2&\texttt{-}&\texttt{-} & \texttt{d}_2&\texttt{a}_3&\texttt{a}_3&\texttt{a}_3&\texttt{b}_3&\texttt{-}&\texttt{-}&\texttt{-} & \texttt{d}_3 \\
\texttt{b}_1&\texttt{-}&\texttt{c}_1& \texttt{d}_1&	\texttt{-}&\texttt{-}&\texttt{b}_2&\texttt{c}_2&\texttt{c}_2& \texttt{d}_2&\texttt{b}_3&\texttt{-}&\texttt{-}&\texttt{-}&\texttt{c}_3&\texttt{c}_3&\texttt{c}_3 & \texttt{d}_3\end{array}.\label{eq:ground_truth}\end{equation} The most important properties of this alignment are that the $\texttt{d}_j$ characters are always matching and the $\texttt{b}_j$ characters alternate between matching and not matching. Similarly, we define the ground-truth alignment of the pair $\left(S_1^{(2)}, S_2^{(2)}\right)$ to be \[\begin{array}{cccccccccccccccccccc}\texttt{a}_2&\texttt{a}_2&\texttt{b}_2&\texttt{-}&\texttt{-} & \texttt{d}_2 \\
\texttt{-}&\texttt{-}&\texttt{b}_2&\texttt{c}_2&\texttt{c}_2& \texttt{d}_2\end{array}.\]
\end{enumerate}

\paragraph{Shattering.} We now show that these two sequence pairs can be shattered. A key step is proving that the functions $u_{(0, \rho[2], 0)}\left(S_1^{(1)}, S_2^{(1)}\right)$ and $u_{(0, \rho[2], 0)}\left(S_1^{(2)}, S_2^{(2)}\right)$ have the following form: \begin{equation}u_{(0, \rho[2], 0)}\left(S_1^{(1)}, S_2^{(1)}\right) = \begin{cases}
	\frac{4}{6}&\text{if } \rho[2] \leq \frac{1}{6}\\
	\frac{5}{6} &\text{if } \frac{1}{6} < \rho[2] \leq \frac{1}{4}\\
	\frac{4}{6} &\text{if } \frac{1}{4} < \rho[2] \leq \frac{1}{2}\\
	\frac{5}{6} &\text{if } \rho[2] > \frac{1}{2}\\
\end{cases} \text{ and }  u_{(0, \rho[2], 0)}\left(S_1^{(2)}, S_2^{(2)}\right) = \begin{cases}
	1&\text{if } \rho[2] \leq \frac{1}{4}\\
	\frac{1}{2} &\text{if } \rho[2] > \frac{1}{4}
\end{cases}.\label{eq:lower_bound_oscillate}\end{equation}
\begin{figure}
	\centering
	\includegraphics[scale =1]{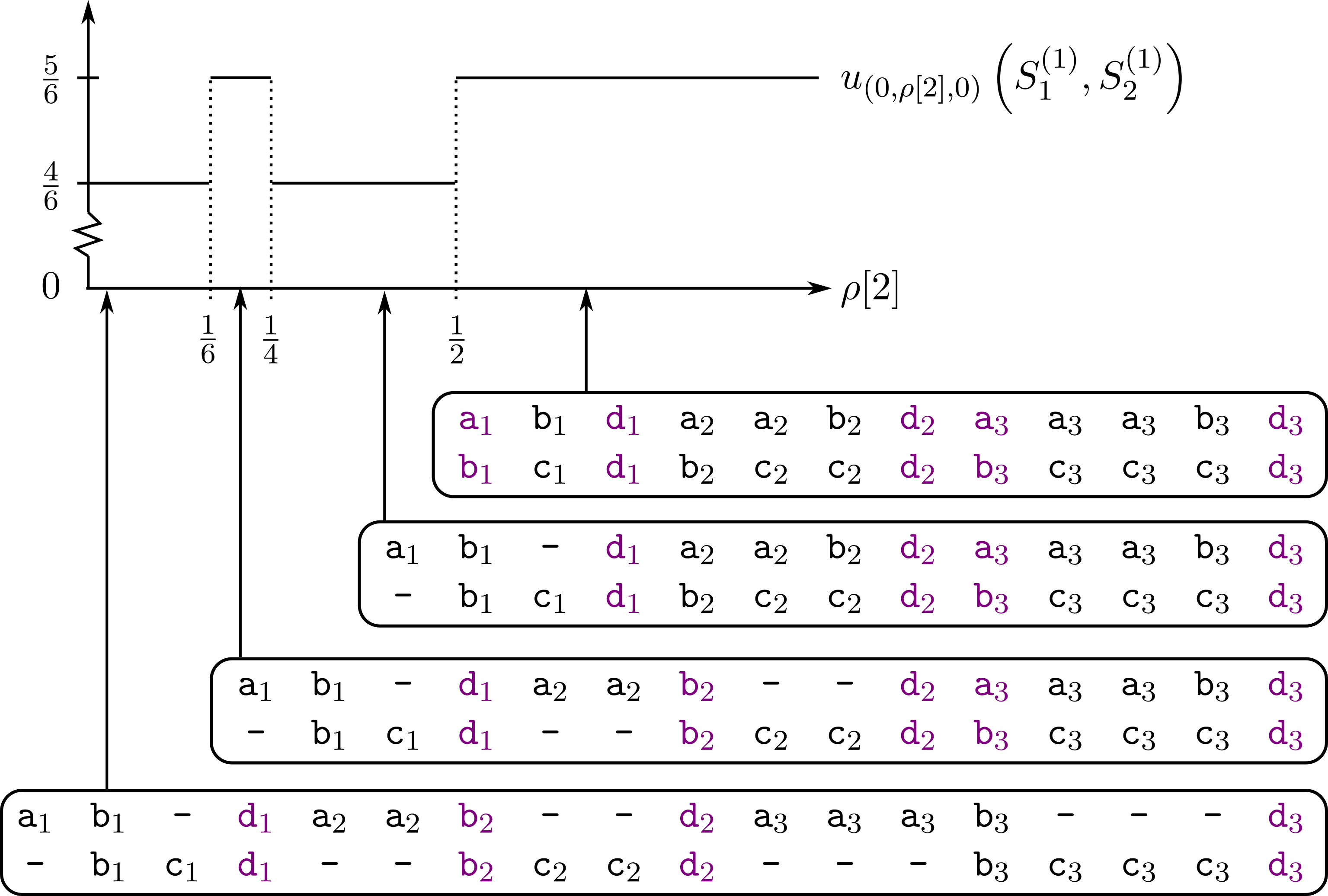}
	\caption{The form of $u_{(0, \rho[2], 0)}\left(S_1^{(1)}, S_2^{(1)}\right)$ as a function of the indel parameter $\rho[2]$. When $\rho[2] \leq \frac{1}{6}$, the algorithm returns the bottom alignment. When $\frac{1}{6} < \rho[2] \leq \frac{1}{4}$, the algorithm returns the alignment that is second to the bottom. When $\frac{1}{4} < \rho[2] \leq \frac{1}{2}$, the algorithm returns the alignment that is second to the top. Finally, when $\rho[2] > \frac{1}{2}$, the algorithm returns the top alignment. The purple characters denote which characters are correctly aligned according to the ground-truth alignment (Equation~\eqref{eq:ground_truth}).}\label{fig:lower_bound}
\end{figure}
The form of $u_{(0, \rho[2], 0)}\left(S_1^{(1)}, S_2^{(1)}\right)$ is illustrated by Figure~\ref{fig:lower_bound}. It is then straightforward to verify that the two sequence pairs are shattered by the parameter settings $(0,0,0)$, $\left(0, \frac{1}{5}, 0\right)$, $\left(0, \frac{1}{3}, 0\right)$, and $\left(0, 1, 0\right)$ with the witnesses $z_1 = z_2 = \frac{3}{4}$. In other words, the mismatch and gap parameters are set to $0$ and the indel parameter $\rho[2]$ takes the values $\left\{0, \frac{1}{5}, \frac{1}{3}, 1\right\}$.

\paragraph{Proof sketch of Equation~\eqref{eq:lower_bound_oscillate}.} The full proof that Equation~\eqref{eq:lower_bound_oscillate} holds follows the following high-level reasoning:
\begin{enumerate}
	\item First, we prove that under the algorithm's output alignment, the $\texttt{d}_j$ characters will always be matching. Intuitively, this is because the algorithm's objective function will always be maximized when each subsequence $t_1^{(j)}$ is aligned with $t_2^{(j)}$.
	\item Second, we prove that the characters $\texttt{b}_j$ will be matched if and only if $\rho[2] \leq \frac{1}{2j}$. Intuitively, this is because in order to match these characters, we must pay with $2j$ indels. Since the objective function is $\mt\left(S_1^{(1)}, S_2^{(1)}, L\right) - \rho[2] \cdot \id\left(S_1^{(1)}, S_2^{(1)}, L\right)$, the $1$ match will be worth the $2j$ indels if and only if $1 \geq 2j\rho[2]$.
\end{enumerate}
These two properties in conjunction mean that when $\rho[2] > \frac{1}{2}$, none of the $\texttt{b}_j$  characters are matched, so the characters that are correctly aligned (as per the ground-truth alignment (Equation~\eqref{eq:ground_truth})) in the algorithm's output are $\left(\texttt{a}_1, \texttt{b}_1\right)$, $\left(\texttt{d}_1, \texttt{d}_1\right)$, $\left(\texttt{d}_2, \texttt{d}_2\right)$, $\left(\texttt{a}_3, \texttt{b}_3\right)$, and $\left(\texttt{d}_3, \texttt{d}_3\right)$, as illustrated by purple in the top alignment of Figure~\ref{fig:lower_bound}. Since there are a total of $6$ aligned letters in the ground-truth alignment, we have that the Q score is $\frac{5}{6}$, or in other words, $u_{(0, \rho[2], 0)}\left(S_1^{(1)}, S_2^{(1)}\right) = \frac{5}{6}.$

When $\rho[2]$ shifts to the next-smallest interval $\left(\frac{1}{4}, \frac{1}{2}\right]$, the indel penalty $\rho[2]$ is sufficiently small that the $\texttt{b}_1$ characters will align. Thus we lose the correct alignment $\left(\texttt{a}_1, \texttt{b}_1\right)$, and the Q score drops to $\frac{4}{6}$. Similarly, if we decrease $\rho[2]$ to the next-smallest interval $\left(\frac{1}{6}, \frac{1}{4}\right]$, the $\texttt{b}_2$ characters will align, which is correct under the ground-truth alignment (Equation~\eqref{eq:ground_truth}). Thus the Q score increases back to $\frac{5}{6}$. Finally, by the same logic, when $\rho[2] \leq \frac{1}{6}$, we lose the correct alignment $\left(\texttt{a}_3, \texttt{b}_3\right)$ in favor of the alignment of the $\texttt{b}_3$ characters, so the Q score falls to $\frac{4}{6}$. In this way, we prove the form of $u_{(0, \rho[2], 0)}\left(S_1^{(1)}, S_2^{(1)}\right)$ from Equation~\eqref{eq:lower_bound_oscillate}. A parallel argument proves the form of $u_{(0, \rho[2], 0)}\left(S_1^{(2)}, S_2^{(2)}\right)$.

\paragraph{Generalization to shattering $\Theta(\log n)$ sequence pairs.} This proof intuition naturally generalizes to $\Theta(\log n)$ sequence pairs of length $O(n)$ by expanding the number of subsequences $t_i^{(j)}$ \emph{a la} Equation~\eqref{eq:subsequence}. In essence, if we define $S_1^{(1)} = t_1^{(1)}t_1^{(2)}\cdots t_1^{(k)}$ and $S_2^{(1)} = t_2^{(1)}t_2^{(2)}\cdots t_2^{(k)}$ for a carefully-chosen $k = \Theta\left(\sqrt{n}\right)$, then we can force $u_{\left(0, \rho[2], 0\right)}\left(S_1^{(1)}, S_2^{(1)}\right)$ to oscillate $O(n)$ times. Similarly, if we define $S_1^{(2)} = t_1^{(2)}t_1^{(4)}\cdots t_1^{(k-1)}$ and $S_2^{(2)} = t_2^{(2)}t_2^{(4)}\cdots t_2^{(k-1)}$, then we can force $u_{\left(0, \rho[2], 0\right)}\left(S_1^{(1)}, S_2^{(1)}\right)$ to oscillate half as many times, and so on. This construction allows us to shatter $\Theta(\log n)$ sequences.
\end{proof}

%% file: multisequence.tex
The multiple sequence alignment problem is a generalization of the pairwise
alignment problem introduced in Section~\ref{sec:sequence}. Let $\Sigma$ be an
abstract alphabet and let $S_1, \dots, S_\numseqs \in \Sigma^n$ be a collection of
sequences in $\Sigma$ of length $n$. A \emph{multiple sequence alignment} is a
collection of sequences $\tau_1, \dots, \tau_\numseqs \in (\Sigma \cup
\{\gapchar\})^*$ such that the following hold:
\begin{enumerate}
  \item The aligned sequences are the same length: $|\tau_1| = |\tau_2| = \cdots = \left|\tau_{\kappa}\right|$.
  \item Removing the gap characters from $\tau_i$ gives $S_i$: for all $i \in [\kappa]$, $\texttt{del}(\tau_i) = S_i$.
  \item For every position in the alignment, at least one of the
  aligned sequences has a non-gap character. In other words, for every position $i \in [|\tau_1|]$, there exists a sequence $\tau_j$ such that $\tau_j[i] \neq \gapchar$.
\end{enumerate}

The extension from pairwise to multiple sequence alignment is computationally challenging: 
all common formulations of the problem are NP-complete~\citep{Wang94:On,Kececioglu04:Aligning}. 
As a result, heuristics have been developed to find good but possibly sub-optimal alignments.
The most common heuristic approach
is called \emph{progressive multiple sequence alignment}.
It leverages
efficient pairwise alignment algorithms to heuristically align multiple
sequences~\citep{Feng87:Progressive}.

The input to a progressive multiple sequence alignment algorithm is a collection of sequences $S_1, \dots, S_\numseqs$ together with a binary \emph{guide tree} $G$ with $\kappa$ leaves\footnote{The problem of constructing the guide tree is also an algorithmic task, often tackled via hierarchical clustering, but we are agnostic to that pre-processing step.}. The tree indicates how the original alignment should be decomposed into a hierarchy of subproblems, each of which can be heuristically
	solved using pairwise alignment. The leaves of the guide tree correspond to the
	input sequences $S_1, \dots, S_\numseqs$.

While there are many formalizations of the progressive alignment method, for the sake of analysis we will focus on ``partial consensus'' described by \citet{Higgins88:Clustal}. Here, we provide a high-level description of the algorithm and in Appendix~\ref{app:multisequence}, we include more detailed pseudo-code.
At a high level, the algorithm recursively constructs an
alignment in two stages:
first, it creates a \emph{consensus sequence} for each node in the guide tree using a pairwise alignment algorithm, and 
then it propagates the node-level alignments to the leaves by inserting additional gap characters.

In a bit more detail, for each node $v$ in the tree, we construct an alignment $L'_v$ 
 of the consensus sequences of its children
as well as a consensus sequence $\sigma'_v \in \Sigma^*$. 
Since each leaf corresponds to a single input sequence, it has a
trivial alignment and the consensus sequence is just the
input sequence itself. 
For an internal node $v$ with children $c_1$ and $c_2$, we use a
pairwise alignment algorithm to construct an alignment of the consensus
strings $\sigma'_{v_1}$ and $\sigma'_{v_2}$. 
Finally, we define the consensus sequence of the node
$v$ to be the string $\sigma_v \in \Sigma^*$ such that $\sigma_v[i]$ is the
most-frequent non-gap character in the $i^\text{th}$ position in the alignment
$L'_v$.
By defining the consensus sequence in this way, we can represent all of the sub-alignments of the leaves of the subtree rooted at $v$ as a single sequence which can be aligned using existing methods.
We obtain a full multiple sequence alignment by iteratively replacing each consensus sequence by the pairwise alignment it represents, adding gap columns to the sub-alignments when necessary. 
Once we add a gap to a sequence, we never remove it: ``once a gap, always a gap.''

The family $\left\{A_{\vec{\rho}} \mid \vec{\rho} \in \reals^d\right\}$ of
parameterized pairwise alignment algorithms introduced in
Section~\ref{sec:sequence} induces a parameterized family of progressive
multiple sequence alignment algorithms $\left\{ M_{\vec{\rho}} \mid
\vec{\rho} \in \reals^d\right\}$. In particular, the algorithm
$M_{\vec{\rho}}$ takes as input a collection of input sequences $S_1, \dots,
S_\numseqs \in \Sigma^n$ and a guide tree $G$, and it outputs a multiple-sequence
alignment $L$ by applying the pairwise alignment algorithm $A_{\vec{\rho}}$ at
each node of the guide tree. We assume that there is a utility function that characterizes an alignment's quality, denoted $u\left(S_1,\dots, S_{\kappa}, L\right) \in [0,1]$. We then define $u_{\vec{\rho}}\left(S_1, \dots, S_{\kappa}, G\right) = u\left(S_1, \dots, S_{\kappa}, M_{\vec{\rho}}\left(S_1, \dots, S_{\kappa}, G\right)\right)$ to be the utility of the alignment returned by the algorithm $M_{\vec{\rho}}$. The proof of the following lemma is in Appendix~\ref{app:multisequence}. It follows by the same logic as Lemma~\ref{lem:sequence} for pairwise sequence alignment, inductively over the guide tree.

\begin{restatable}{lemma}{progressive}\label{lem:progressive}
Let $G$ be a guide tree of depth $\eta$ and let $\cU$ be the set
  of functions
  \[
  \cU =
  \left\{
    u_{\vec{\rho}} :
    (S_1, \dots, S_\numseqs, G) \mapsto
    u\bigl(S_1, \dots, S_\numseqs, M_{\vec{\rho}}(S_1, \dots, S_\numseqs, G)\bigr)
    \mid \vec{\rho} \in \reals^d
  \right\}.
  \]
  The dual class $\cU^*$ is \[\left(\cF, \cG,
  \left(4^{n\kappa} \left(n\kappa\right)^{4n\kappa + 2}\right)^{2d^{\eta}}4^{d^{\eta+1}}\right)\text{-piecewise decomposable,}\] where $\cG = \left\{g_{\vec{a}, \theta} : \cU \to \{0,1\} \mid \vec{a} \in \R^d, \theta \in \R\right\}$ consists of halfspace indicator functions $g_{\vec{a}, \theta} : u_{\vec{\rho}} \mapsto \ind{\vec{a} \cdot \vec{\rho} \leq \theta}$ and $\cF = \{f_c : \cU \to \R \mid c \in \R\}$ consists of constant functions $f_c : u_{\vec{\rho}} \mapsto c$.
\end{restatable}

This lemma together with Lemma~\ref{lem:simple_multidim} implies that $
\pdim(\cU) =  O\left(
d^{\eta + 1} n \kappa \ln(n \kappa)
+ d^{\eta + 2}
\right) 
$. Therefore, the pseudo-dimension grows only linearly in $n$ and quadratically in $\kappa$
in the affine-gap model ($d = 3$) when the guide tree is balanced ($\eta \leq \log \kappa$). 

%% file: RNA.tex
RNA molecules have many essential roles including protein coding and enzymatic functions~\cite{Holley65:Structure}.
RNA is assembled as a chain of \emph{bases} denoted \texttt{A}, \texttt{U}, \texttt{C}, and \texttt{G}.
It is often found as a single strand folded onto itself with non-adjacent bases physically bound together.
RNA folding algorithms infer the way strands would naturally fold, shedding light on their functions.
Given a sequence $S \in \left\{\texttt{A}, \texttt{U}, \texttt{C}, \texttt{G} \right\}^n$,
we represent a folding by a set of pairs $\phi \subset [n] \times [n]$.
If $(i,j) \in \phi$, then the $i^{th}$ and $j^{th}$ bases of $S$ bind together.
Typically, the bases \texttt{A} and \texttt{U} bind together, as do \texttt{C} and \texttt{G}.
Other matchings are likely less stable.
We assume that the foldings do not contain any \emph{pseudoknots}, which are pairs $(i,j), (i',j')$ that cross with $i < i' < j < j'$.

A well-studied algorithm returns a folding that maximizes a parameterized objective function~\cite{Nussinov80:Fast}.
At a high level, this objective function trades off between global properties of the folding (the number of binding pairs $|\phi|$) and local properties (the likelihood that bases would appear close together in the folding). Specifically,  the algorithm $A_{\rho}$ uses dynamic programming to  return the folding $A_{\rho}(S)$ that maximizes
\begin{equation}
\rho \left| \phi \right| + \left( 1 - \rho \right) \sum_{(i,j) \in \phi} M_{S[i], S[j], S[i-1],S[j+1]} \ind{(i-1,j+1)\in \phi},\label{eq:RNA_obj}
\end{equation}
where $\rho \in [0,1]$ is a parameter and $M_{S[i], S[j], S[i-1],S[j+1]}\in  \R$ is a score for having neighboring pairs of the letters $(S[i],S[j])$ and $(S[i-1],S[j+1])$.
These scores help identify stable sub-structures.

We assume there is a utility function that characterizes a folding's quality, denoted $u(S, \phi)$.
For example, $u(S, \phi)$ might measure the fraction of pairs shared between $\phi$ and a ``ground-truth'' folding, obtained via expensive computation or laboratory experiments.

\begin{restatable}{lemma}{RNA}\label{lem:RNA_piecewise}
	Let $\cU$ be the set of functions $\cU = \left\{u_{\rho} : S \mapsto u\left(S, A_{\rho}\left(S\right)\right) \mid \rho \in \R\right\}$.
	The dual class $\cU^*$ is $\left(\cF, \cG, n^{2}\right)$-piecewise decomposable, where $\cG = \{g_{a} : \cU \to \{0,1\} \mid a \in \R\}$ consists of threshold functions $g_{a} : u_{\rho} \mapsto \ind{\rho < a}$ and $\cF = \{f_c : \cU \to \R \mid c \in \R\}$ consists of constant functions $f_c : u_{\rho} \mapsto c$.
\end{restatable}

\begin{proof}
	Fix a sequence $S$. Let $\Phi$ be the set of alignments that the algorithm returns as we range over all parameters $\rho \in \R$. In other words, 
	$\Phi = \{A_{\rho}(S) \mid \rho \in [0,1]\}$.
	We know that every folding has length at most $n/2$. For any $k \in \{0,\dots, n/2\}$, let $\phi_k$ be the folding of length $k$ that maximizes the right-hand-side of Equation~\eqref{eq:RNA_obj}: \[\phi_k = \argmax_{\phi : |\phi| = k}\sum_{(i,j) \in \phi} M_{S[i], S[j], S[i-1],S[j+1]} \ind{(i-1,j+1)\in \phi}.\] The folding the algorithm returns will always be one of $\left\{\phi_0, \dots, \phi_{n/2}\right\}$, so $\left|\Phi\right| \leq \frac{n}{2} + 1.$
	
	Fix an arbitrary folding $\phi \in \Phi$. We know that $\phi$ will be the folding returned by the algorithm $A_{\rho}(S)$ if and only if \begin{align}\begin{split}
			&\rho\left|\phi\right|+(1-\rho)\sum_{(i,j) \in \phi} M_{S[i], S[j], S[i-1],S[j+1]} \ind{(i-1,j+1)\in \phi}\\
			\geq\text{ } & \rho\left|\phi'\right|+(1-\rho)\sum_{(i,j) \in \phi'} M_{S[i], S[j], S[i-1],S[j+1]} \ind{(i-1,j+1)\in \phi'}\label{eq:halfspace_folding}
	\end{split}\end{align}
	for all $\phi' \in \Phi \setminus \{\phi\}$. Since these functions are linear in $\rho$, this means there is a set of $T \leq {|\Phi| \choose 2} \leq n^{2}$ intervals $[\rho_1, \rho_2), [\rho_2, \rho_3), \dots, [\rho_{T}, \rho_{T+1}]$ with $\rho_1 := 0 < \rho_2 < \cdots < \rho_{T} < 1 := \rho_{T+1}$ such that for any one interval $I$, across all $\rho \in I$, $A_{\rho}(S)$ is fixed. This means that for any one interval $[\rho_i, \rho_{i + 1})$, there exists a real value $c_i$ such that $u_{\rho}(S) = c_i$ for all $\rho \in [\rho_i, \rho_{i + 1})$. By definition of the dual, this means that $u_S^*(u_{\rho}) = u_{\rho}(S) = c_i$ as well.
	
We now use this structure to show that the dual class $\cU^*$ is $\left(\cF, \cG, n^{2}\right)$-piecewise decomposable, as per Definition~\ref{def:decomposable_sign}.
	Recall that $\cG = \{g_{a} : \cU \to \{0,1\} \mid a \in \R\}$ consists of threshold functions $g_{a} : u_{\rho} \mapsto \ind{\rho < a}$ and $\cF = \{f_c : \cU \to \R \mid c \in \R\}$ consists of constant functions $f_c : u_{\rho} \mapsto c$.
	We claim that there exists a function $f^{(\vec{b})} \in \cF$ for every vector $\vec{b} \in \{0,1\}^{T}$ such that for every $\rho \in [0,1]$, \begin{equation}u_S^*(u_{\rho}) = \sum_{\vec{b} \in \{0,1\}^{T}} \ind{g_{\rho_i}(u_{\rho}) = b[i], \forall i \in [T]} f^{(\vec{b})}(u_{\rho}).\label{eq:folding_piecewise}
	\end{equation}
	To see why, suppose $\rho \in [\rho_i, \rho_{i + 1})$ for some $i \in [T]$. Then $g_{\rho_j}(u_{\rho}) = \ind{\rho \leq \rho_j} = 1$ for all $j \geq i + 1$ and $g_{\rho_j}(u_{\rho}) = \ind{\rho \leq \rho_j} = 0$ for all $j \leq i$. Let $\vec{b}_i \in \{0,1\}^{T}$ be the vector that has only 0's in its first $i$ coordinates and all $1$'s in its remaining $T -i$ coordinates. For all $i \in [T]$, we define $f^{\left(\vec{b}_i\right)} = f_{c_i}$, so $f^{\left(\vec{b}_i\right)}\left(u_{\rho}\right) = c_i$ for all $\rho \in [0,1]$. For any other $\vec{b}$, we set $f^{(\vec{b})} = f_0$,  so $f^{(\vec{b})}\left(u_{\rho}\right) = 0$ for all $\rho \in [0,1]$. Therefore, Equation~\eqref{eq:folding_piecewise} holds.
\end{proof}

	Since constant functions have zero oscillations, Lemmas~\ref{lem:oscillate_decomp} and \ref{lem:RNA_piecewise} imply that $\pdim(\cU)=O\left(\ln n\right).$ 

%% file: TAD.tex
Inside a cell, the linear DNA of the genome wraps into three-dimensional structures that influence genome function.
Some regions of the genome are
closer than others and thereby interact more.
\emph{Topologically associating domains (TADs)} are contiguous segments of the genome that fold into compact regions.
More formally, given the genome length $n$, a TAD set is a set $T = \left\{(i_1, j_1), \dots, (i_t, j_t)\right\} \subset [n] \times [n]$ such that $i_1 < j_1 < i_2 < j_2 < \cdots < i_t < j_t$. If $(i,j) \in T$, the bases within the corresponding substring physically interact more frequently with each other than with other bases.
Disrupting TAD boundaries can affect the expression of nearby genes, which can trigger diseases such as congenital malformations and cancer~\cite{Lupianez16:Breaking}.

The contact frequency of any two genome locations, denoted by a matrix $M \in \R^{n \times n}$, can be measured via experiments~\cite{Lieberman-Aiden09:Comprehensive}.
A dynamic programming algorithm $A_{\rho}$ introduced by~\citet{Filippova14:Identification} returns the TAD set $A_{\rho}(M)$ that maximizes
\begin{equation}
	\sum_{(i,j) \in T}s_\rho(i,j) - \mu_\rho(j-i),
	\label{eq:TAD_obj}\end{equation} where $\rho \geq 0$ is a parameter,
\[s_\rho(i,j) = \frac{1}{(j-i)^\rho}\sum_{i\le p < q \le j}M_{pq}\] is the scaled density of the subgraph induced by the interactions between genomic loci $i$ and $j$, and
\[\mu_\rho(d) = \frac{1}{n-d}\sum_{t=0}^{n-d-1} s_\rho(t, t+d)\] is the mean value of $s_{\rho}$ over all sub-matrices of length $d$ along the diagonal of $M$.
We note that unlike the sequence alignment and RNA folding algorithms, the parameter $\rho$ appears in the exponent of the objective function.

We assume there is a utility function that characterizes the quality of a TAD set $T$, denoted $u(M, T) \in \R$. 
For example, $u(M, T)$ might measure the fraction 
of TADs in $T$ that are in the correct location with respect to a ground-truth TAD set.

\begin{restatable}{lemma}{TAD}\label{lem:TAD_decomp}
	Let $\cU$ be the set of functions $\cU = \left\{u_{\rho} : M \mapsto u\left(M, A_{\rho}\left(M\right)\right) \mid \rho \in \R\right\}$.
	The dual class $\cU^*$ is $\left(\cF, \cG, 2n^2 4^{n^2}\right)$-piecewise decomposable, where $\cG = \{g_{a} : \cU \to \{0,1\} \mid a \in \R\}$ consists of threshold functions $g_{a} : u_{\rho} \mapsto \ind{\rho < a}$ and $\cF = \{f_c : \cU \to \R \mid c \in \R\}$ consists of constant functions $f_c : u_{\rho} \mapsto c$.
\end{restatable}

\begin{proof}
	Fix a matrix $M$. We begin by rewriting Equation~\eqref{eq:TAD_obj} as follows:
	\begin{align*}
		A_\rho(M) &= \underset{T \subset [n] \times [n]}{\argmax}\sum_{(i,j) \in T}\left(\frac{1}{(j-i)^\rho}\left(\sum_{i\le u < v \le j}M_{uv}\right) -\frac{1}{n-j+i}\sum_{t=0}^{n-j+i} \frac{1}{(j-i)^\rho}\sum_{t\le p < q \le t+j-i}M_{pq}\right)\\
		&= \argmax\sum_{(i,j) \in T}\frac{1}{(j-i)^\rho}\left(\left(\sum_{i\le u < v \le j}M_{uv}\right) -\frac{1}{n-j+i}\sum_{t=0}^{n-j+i} \sum_{t\le p < q \le t+j-i}M_{pq}\right)\\
		&= \argmax\sum_{(i,j) \in T}\frac{c_{ij}}{(j-i)^\rho},
	\end{align*} where \[c_{ij} = \left(\sum_{i\le u < v \le j}M_{uv}\right) -\frac{1}{n-j+i}\sum_{t=0}^{n-j+i} \sum_{t\le p < q \le t+j-i}M_{pq}\]  is a constant that does not depend on $\rho$.
	
Let $\cT$ be the set of TAD sets that the algorithm returns as we range over all parameters $\rho \geq 0$. In other words, $\cT = \{A_{\rho}(M) \mid \rho \in \R_{\geq 0}\}$. Since each TAD set is a subset of $[n] \times [n]$, $\left|\cT\right| \leq 2^{n^2}$.
For any TAD set $T \in \cT$,  the algorithm $A_{\vec{\rho}}$ will return $T$ if and only if \[\sum_{(i,j) \in T}\frac{c_{ij}}{(j-i)^\rho} > \sum_{(i',j') \in T'}\frac{c_{i'j'}}{(j'-i')^\rho}\]
	for all $T' \in \cT \setminus \{T\}$.
	This means that as we range $\rho$ over the positive reals, the TAD set returned by algorithm $A_{\rho}(M)$ will only change when \begin{equation}\sum_{(i,j) \in T}\frac{c_{ij}}{(j-i)^\rho} - \sum_{(i',j') \in T'}\frac{c_{i'j'}}{(j'-i')^\rho} = 0\label{eq:TAD_zeros}\end{equation} for some $T, T' \in \cT$. As a result of Rolle's Theorem (Corollary~\ref{cor:roots}), we know that Equation~\eqref{eq:TAD_zeros} has at most $|T| + |T'| \leq 2n^2$ solutions. This means there are $t \leq 2n^2{|\cT| \choose 2} \leq 2n^2 4^{n^2}$ intervals $\left[\rho_1, \rho_2\right), \left[\rho_2, \rho_3\right), \dots, \left[\rho_{t}, \rho_{t+1}\right)$ with $ \rho_1 := 0 < \rho_2 < \cdots < \rho_{t} < \infty := \rho_{t+1}$ that partition  $\R_{\geq 0}$ such that across all $\rho$ within any one interval $\left[\rho_i, \rho_{i+1}\right)$, the TAD set returned by algorithm $A_{\rho}(M)$ is fixed. Therefore, there exists a real value $c_i$ such that $u_{\rho}(M) = c_i$ for all $\rho \in \left[\rho_i, \rho_{i+1}\right)$. By definition of the dual, this means that $u_{M}^*(u_{\rho}) = u_{\rho}(M) = c_i$ as well.
	
We now use this structure to show that the dual class $\cU^*$ is $\left(\cF, \cG, 2n^2 4^{n^2}\right)$-piecewise decomposable, as per Definition~\ref{def:decomposable_sign}.
	Recall that $\cG = \{g_{a} : \cU \to \{0,1\} \mid a \in \R\}$ consists of threshold functions $g_{a} : u_{\rho} \mapsto \ind{\rho < a}$ and $\cF = \{f_c : \cU \to \R \mid c \in \R\}$ consists of constant functions $f_c : u_{\rho} \mapsto c$.
	We claim that there exists a function $f^{(\vec{b})} \in \cF$ for every vector $\vec{b} \in \{0,1\}^{t}$ such that for every $\rho \geq 0$, \begin{equation}u_M^*(u_{\rho}) = \sum_{\vec{b} \in \{0,1\}^{t}} \ind{g_{\rho_i}(u_{\rho}) = b[i], \forall i \in [t]} f^{(\vec{b})}(u_{\rho}).\label{eq:TAD_piecewise}
	\end{equation}
	To see why, suppose $\rho \in [\rho_i, \rho_{i + 1})$ for some $i \in [t]$. Then $g_{\rho_j}(u_{\rho}) = \ind{\rho \leq \rho_j} = 1$ for all $j \geq i + 1$ and $g_{\rho_j}(u_{\rho}) = \ind{\rho \leq \rho_j} = 0$ for all $j \leq i$. Let $\vec{b}_i \in \{0,1\}^{t}$ be the vector that has only 0's in its first $i$ coordinates and all $1$'s in its remaining $t-i$ coordinates. For all $i \in [t]$, we define $f^{\left(\vec{b}_i\right)} = f_{c_i}$, so $f^{\left(\vec{b}_i\right)}\left(u_{\rho}\right) = c_i$ for all $\rho \in [0,1]$. For any other $\vec{b}$, we set $f^{(\vec{b})} = f_0$, so $f^{(\vec{b})}\left(u_{\rho}\right) = 0$ for all $\rho \in [0,1]$. Therefore, Equation~\eqref{eq:TAD_piecewise} holds.
\end{proof}

	Since constant functions have zero oscillations, Lemmas~\ref{lem:oscillate_decomp} and \ref{lem:TAD_decomp} imply that $\pdim(\cU)=O\left(n^2\right).$

%% file: econ.tex
A large body of research in economics studies how to design protocols---or \emph{mechanisms}---that help groups of agents come to collective decisions. For example, when children inherit an estate, how should they divide the property? When a jointly-owned company is dissolved, which partner should buy the others out? There is no one protocol that best answers these questions; the optimal mechanism depends on the setting at hand. 

We study a family of mechanisms called \emph{neutral affine maximizers (NAMs})~\cite{Roberts79:Characterization,Mishra12:Roberts,Nath19:Efficiency}. A NAM takes as input a set of agents' reported values for each possible outcome and returns one of those outcomes. A NAM can thus be thought of as an algorithm that the agents use to arrive at a single outcome. NAMs are \emph{incentive compatible}, which means that each agent is incentivized to report his values truthfully. In order to satisfy incentive compatibility, each agent may have to make a payment. NAMs are also \emph{budget-balanced} which means that the aggregated payments are redistributed among the agents.

Formally, we study a setting where there is a set of $m$ alternatives and a set of $n$ agents. Each agent $i$ has a value $v_i(j)  \in \R$ for each alternative $j \in [m]$. We denote all of his values as $\vec{v}_i \in \R^m$ and all $n$ agents' values as $\vec{v} = \left(\vec{v}_1, \dots, \vec{v}_n\right) \in \R^{nm}$.
 In this case, the unknown distribution $\dist$ is over vectors $\vec{v} \in \R^{nm}$.

	A NAM is defined by $n$ parameters (one per agent) $\vec{\rho} = \left(\rho[1], \dots, \rho[n]\right) \in \R^n_{\geq 0}$ such that at least one agent is assigned a weight of zero.
There is a \emph{social choice function}
$\psi_{\vec{\rho}} : \R^{nm} \to [m]$ which uses the values $\vec{v} \in \R^{nm}$ to choose an alternative $\psi_{\vec{\rho}}(\vec{v}) \in [m]$.
In particular, $\psi_{\vec{\rho}}(\vec{v}) = \argmax_{j \in [m]} \sum_{i = 1}^n \rho[i]v_i(j)$ maximizes the agents' weighted values. Each agent $i$ with zero weight $\rho[i] = 0$ is called a \emph{sink agent} because his values do not influence the outcome. 
For every agent who is not a sink agent $\left(\rho[i] \not = 0\right)$, their payment is defined as in the weighted version of the classic Vickrey-Clarke-Groves mechanism~\citep{Vickrey61:Counterspeculation,Clarke71:Multipart,Groves73:Incentives}. To achieve budget balance, these payments are given to the sink agent(s). More formally,
let $j^* = \psi_{\vec{\rho}}(\vec{v})$ and for each agent $i$, let $j_{-i} =  \argmax_{j \in [m]} \sum_{i' \not= i} \rho[i']v_{i'}(j).$ The payment function is defined as \[p_i(\vec{v}) = \begin{cases} \frac{1}{\rho[i]} \left(\sum_{i' \not= i} \rho[i']v_{i'}\left(j^*\right) - \sum_{i' \not= i} \rho[i']v_{i'}\left(j_{-i}\right)\right) &\text{if } \rho[i] \not= 0\\
	-\sum_{i' \not= i} p_{i'}(\vec{v}) &\text{if } i = \min\left\{i' : \rho[i'] = 0\right\}\\
	0 &\text{otherwise.}\end{cases}\]

We aim to optimize the expected social welfare $\E_{\vec{v}\sim \dist}\left[\sum_{i = 1}^n v_i\left(\psi_{\vec{\rho}}(\vec{v}) \right)\right]$ of the NAM's outcome $\psi_{\vec{\rho}}(\vec{v})$, so we define the utility function
$u_{\vec{\rho}}(\vec{v}) = \sum_{i = 1}^n v_i\left(\psi_{\vec{\rho}}(\vec{v}) \right)$.

\begin{restatable}{lemma}{NAM}\label{lem:AM_decomposable}
	Let $\cU$ be the set of functions $\cU = \left\{u_{\vec{\rho}} \mid \vec{\rho} \in \R_{\geq 0}^n, \left\{i \mid \rho[i] = 0\right\} \not= \emptyset\right\}$. The dual class $\cU^*$ is $\left(\cF, \cG, m^2\right)$-piecewise decomposable, where $\cG = \{g_{\vec{a}} : \cU \to \{0,1\} \mid \vec{a} \in \R^n\}$ consists of halfspace indicators $g_{\vec{a}} : u_{\vec{\rho}} \mapsto \ind{\vec{\rho} \cdot \vec{a} \leq 0}$ and $\cF = \{f_c : \cU \to \R \mid c \in \R\}$ consists of constant functions $f_c : u_{\vec{\rho}} \mapsto c$.
\end{restatable}

\begin{proof}
	Fix a valuation vector $\vec{v} \in \R^{nm}$. We know that for any two alternatives $j, j' \in [m]$, the alternative $j$ would be selected over $j'$ so long as \begin{equation}\sum_{i = 1}^n \rho[i] v_i(j) > \sum_{i = 1}^n \rho[i] v_i\left(j'\right).\label{eq:AM_halfspace}\end{equation} Therefore, there is a set $\cH$ of ${m \choose 2}$ hyperplanes such that across all parameter vectors $\vec{\rho}$ in a single connected component of $\R^n \setminus \cH$, the outcome of the NAM defined by $\vec{\rho}$ is fixed. When the outcome of the NAM is fixed, the social welfare is fixed as well. This means that for a single connected component $R$ of $\R^n \setminus \cH$, there exists a real value $c_R$ such that $u_{\vec{\rho}}(\vec{v}) = c_R$ for all $\vec{\rho} \in R$. By definition of the dual, this means that $u^*_{\vec{v}}\left(u_{\vec{\rho}}\right)  = u_{\vec{\rho}}(\vec{v}) = c_R$ as well.

We now use this structure to show that the dual class $\cU^*$ is $\left(\cF, \cG, m^2\right)$-piecewise decomposable, as per Definition~\ref{def:decomposable_sign}.
	Recall that $\cG = \left\{g_{\vec{a}} : \cU \to \{0,1\} \mid \vec{a} \in \R^n\right\}$ consists of halfspace indicator functions $g_{\vec{a}} : u_{\vec{\rho}} \mapsto \ind{\vec{a} \cdot \vec{\rho} < 0}$ and $\cF = \{f_c : \cU \to \R \mid c \in \R\}$ consists of constant functions $f_c : u_{\vec{\rho}} \mapsto c$.
	For each pair  of alternatives $j,j' \in \cL$, let $g^{(j,j')} \in \cG$ correspond to the halfspace represented in Equation~\eqref{eq:AM_halfspace}. Order these $k := {m \choose 2}$ functions arbitrarily as $g^{(1)}, \dots, g^{(k)}$.
	Every connected component $R$ of $\R^n \setminus \cH$ corresponds to a sign pattern of the $k$ hyperplanes. For a given region $R$, let $\vec{b}_R \in \{0,1\}^k$ be the corresponding sign pattern. Define the function $f^{\left(\vec{b}_R\right)} \in \cF$ as $f^{\left(\vec{b}_R\right)}= f_{c_R}$, so $f^{\left(\vec{b}_R\right)}\left(u_{\vec{\rho}}\right) =  c_R$ for all $\vec{\rho} \in \R^n$. Meanwhile, for every vector $\vec{b}$ not corresponding to a sign pattern of the $k$ hyperplanes, let $f^{(\vec{b})} = f_0$, so $f^{(\vec{b})}\left(u_{\vec{\rho}}\right) = 0$ for all $\vec{\rho} \in \R^n$.
	In this way, for every $\vec{\rho} \in \R^n$, \[u_{\vec{v}}^*\left(u_{\vec{\rho}}\right) = \sum_{\vec{b} \in \{0,1\}^{k}} \ind{g^{(i)}\left(u_{\vec{\rho}}\right) = b[i], \forall i \in [k]} f^{(\vec{b})}(u_{\vec{\rho}}),\] as desired.
\end{proof}

Theorem~\ref{thm:main} and Lemma~\ref{lem:AM_decomposable} imply that the pseudo-dimension of $\cU$ is $O(n \ln m).$
Next, we prove that the pseudo-dimension of $\cU$ is at least $\frac{n}{2}$, which means that our pseudo-dimension upper bound is tight up to log factors.

\begin{theorem}\label{thm:NAM_lb}
	Let $\cU$ be the set of functions
	$\cU = \left\{u_{\vec{\rho}} \mid \vec{\rho} \in \R_{\geq 0}^n,
	\left\{\rho[i] \mid i = 0\right\} \not= \emptyset\right\}$.
	Then $\pdim(\cU) \geq \frac{n}{2}$.
\end{theorem}

\begin{proof}
	Let the number of alternatives $m = 2$ and without loss of generality, suppose that $n$ is even. To prove this theorem, we will identify a set of $N = \frac{n}{2}$ valuation vectors $\vec{v}^{(1)}, \dots, \vec{v}^{(N)}$ that are shattered by the set $\cU$ of social welfare functions.
	
	Let $\epsilon$ be an arbitrary number in $\left(0,\frac{1}{2}\right)$. For each $\ell \in [N]$, define agent $i$'s values for the first and second alternatives under the $\ell^{th}$ valuation vector $\vec{v}^{(\ell)}$---namely, $v_i^{(\ell)}(1)$ and $v_i^{(\ell)}(2)$---as follows:
	\[v_i^{(\ell)}(1) = \begin{cases} 1 &\text{if } \ell = i\\
		0 &\text{otherwise}\end{cases} \text{ and } v_i^{(\ell)}(2) = \begin{cases} \epsilon &\text{if } \ell = \frac{n}{2} + i\\
		0 &\text{otherwise.}\end{cases}\]
For example, if there are $n = 6$ agents, then across the $N = \frac{n}{2} = 3$ valuation vectors $\vec{v}^{(1)}, \vec{v}^{(2)}, \vec{v}^{(3)}$, the agents' values for the first alternative are defined as
\[\begin{bmatrix} v_1^{(1)}(1) & \cdots & v_6^{(1)}(1)\\
v_1^{(2)}(1)  & \cdots & v_6^{(2)}(1) \\
v_1^{(3)}(1) & \cdots & v_6^{(3)}(1)\end{bmatrix} = \begin{bmatrix}
1 & 0 & 0 & 0 & 0 & 0 \\
0 & 1 & 0 & 0 & 0 & 0 \\
0& 0 &1 & 0 & 0 & 0 
\end{bmatrix}\]	and their values for the second alternative are defined as \[\begin{bmatrix} v_1^{(1)}(2) & \cdots & v_6^{(1)}(2)\\
v_1^{(2)}(2)  & \cdots & v_6^{(2)}(2) \\
v_1^{(3)}(2) & \cdots & v_6^{(3)}(2)\end{bmatrix} = \begin{bmatrix}
0 & 0 & 0 & \epsilon & 0 & 0 \\
0 & 0 & 0 & 0 & \epsilon & 0 \\
0& 0 & 0 & 0 & 0 & \epsilon
\end{bmatrix}.\]	
	
	Let $\vec{b} \in \{0,1\}^N$ be an arbitrary bit vector. We will construct a NAM parameter vector $\vec{\rho}$ such that for any $\ell \in [N]$, if $b_{\ell} = 0$, then the outcome of the NAM given bids $\vec{v}^{(\ell)}$ will be the second alternative, so $u_{\vec{\rho}}\left(\vec{v}^{(\ell)}\right) = \epsilon$ because there is always exactly one agent who has a value of $\epsilon$ for the second alternative, and every other agent has a value of $0$. Meanwhile, if $b_{\ell} = 0$, then the outcome of the NAM given bids $\vec{v}^{(\ell)}$ will be the first alternative, so $u_{\vec{\rho}}\left(\vec{v}^{(\ell)}\right) = 1$ because there is always exactly one agent who has a value of $1$ for the first alternative, and every other agent has a value of $0$.
	To do so, when $b_{\ell} = 0$, $\vec{\rho}$ must ignore the values of agent $\ell$ in favor of the values of agent $\frac{n}{2} + \ell$. After all, under $\vec{v}^{(\ell)}$, agent $\ell$ has a value of $1$ for the first alternative and agent $\frac{n}{2} + \ell$ has a value of $\epsilon$ for the second alternative, and all other values are $0$. By a similar argument, when $b_{\ell} = 1$, $\vec{\rho}$ must ignore the values of agent $\frac{n}{2} + \ell$ in favor of the values of agent $\ell$.
	Specifically, we define $\vec{\rho} \in \{0,1\}^n$ as follows: for all $\ell \in [N] = \left[\frac{n}{2}\right]$, if $b_\ell = 0$, then $\rho[\ell] = 0$ and $\rho\left[\frac{n}{2} + \ell\right] = 1$ and if $b_\ell = 1$, then $\rho[\ell] = 1$ and $\rho\left[\frac{n}{2} + \ell\right] = 0$. All other entries of $\vec{\rho}$ are set to $0$.
	
	We claim that if $b_\ell = 0$, then $u_{\vec{\rho}}\left(\vec{v}^{(\ell)}\right) = \epsilon$. To see why, we know that $\sum_{i = 1}^n \rho[i] v_i^{(\ell)}(1) = \rho[\ell]v_{\ell}^{(\ell)}(1) = \rho[\ell] = 0$. Meanwhile, $\sum_{i = 1}^n \rho[i] v_i^{(\ell)}(2) = \rho\left[\frac{n}{2} + \ell\right]v_{\frac{n}{2} + \ell}^{(\ell)}(1) = \epsilon$. Therefore, the outcome of the NAM is alternative 2. The social welfare of this alternative is $\epsilon$, so $u_{\vec{\rho}}\left(\vec{v}^{(\ell)}\right) = \epsilon$.
	
	Next, we claim that if $b_\ell = 1$, then $u_{\vec{\rho}}\left(\vec{v}^{(\ell)}\right) = 1$. To see why, we know that $\sum_{i = 1}^n \rho[i] v_i^{(\ell)}(1) = \rho[\ell]v_{\ell}^{(\ell)}(1) = \rho[\ell] = 1$. Meanwhile, $\sum_{i = 1}^n \rho[i] v_i^{(\ell)}(2) = \rho\left[\frac{n}{2} + \ell\right]v_{\frac{n}{2} + \ell}^{(\ell)}(1) = 0$. Therefore, the outcome of the NAM is alternative 1. The social welfare of this alternative is $1$, so $u_{\vec{\rho}}\left(\vec{v}^{(\ell)}\right) = 1$.
	
	We conclude that the valuation vectors $\vec{v}^{(1)}, \dots, \vec{v}^{(N)}$ that are shattered by the set $\cU$ of social welfare functions with witnesses $z^{(1)} = \cdots = z^{(N)} = \frac{1}{2}$.
\end{proof}

Theorem~\ref{thm:NAM_lb} implies that the pseudo-dimension upper bound from Lemma~\ref{lem:AM_decomposable} is tight up to logarithmic factors.

%% file: prior_research.tex
Theorem~\ref{thm:main} also recovers existing guarantees for data-driven algorithm design. In all of these cases, Theorem~\ref{thm:main} implies generalization guarantees that match the existing bounds,
but in many cases, our approach provides a more succinct proof. 
\begin{enumerate}
\item In Section~\ref{sec:clustering}, we analyze several parameterized clustering algorithms~\citep{Balcan17:Learning}, which have piecewise-constant dual functions.
These algorithms first run a linkage routine which builds a hierarchical tree of clusters. The parameters interpolate between the popular single, average, and complete linkage. The linkage routine is followed by a dynamic  programming procedure that returns a clustering corresponding to a pruning of the hierarchical tree.
\item \citet{Balcan20:LearningToLink} study a family of linkage-based clustering algorithms where the parameters control the distance metric used for clustering in addition to the linkage routine. The algorithm family has two sets of parameters. The first set of parameters interpolate between linkage algorithms, while the second set interpolate between distance metrics. The dual functions are piecewise-constant with \emph{quadratic} boundary functions. We recover their generalization bounds in Section~\ref{sec:learningToLink}.
\item In Section~\ref{sec:IP}, we analyze several integer programming algorithms, which have piecewise-constant and piecewise-inverse-quadratic dual functions (as in Figure~\ref{fig:quadratic}). The first is branch-and-bound, which is used by commercial solvers such as CPLEX. Branch-and-bound always finds an optimal solution and its parameters control runtime and memory usage. We also study semidefinite programming approximation algorithms for integer quadratic programming. We analyze a parameterized algorithm introduced by~\citet{Feige06:RPR} which includes the Goemans-Williamson algorithm~\citep{Goemans95:Improved} as a special case. We recover previous generalization bounds in both settings~\citep{Balcan18:Learning,Balcan17:Learning}.
\item \citet{Gupta17:PAC} introduced parameterized greedy algorithms for the knapsack and maximum weight independent set problems, which we show have piecewise-constant dual functions. We recover their generalization bounds in Section~\ref{sec:greedy}.
\item We provide generalization bounds for parameterized selling mechanisms when the goal is to maximize revenue, which have piecewise-linear dual functions (as in Figure~\ref{fig:linear}). A long line of research has studied revenue maximization via machine learning~\citep{Likhodedov04:Boosting,Likhodedov05:Approximating,Sandholm15:Automated,Balcan05:Mechanism,Elkind07:Designing,Cole14:Sample,Devanur16:Sample,Gonczarowski17:Efficient,Guo19:Settling,Cai17:Learning,Gonczarowski21:Sample,Morgenstern16:Learning,Mohri14:Learning}.
	In Section~\ref{sec:revenue}, we recover \citeauthor*{Balcan18:General}'s generalization bounds~\citep{Balcan18:General} which apply to a variety of  pricing, auction, and lottery mechanisms. They proved new bounds for mechanism classes not  previously studied  in  the  sample-based  mechanism  design  literature  and  matched  or  improved  over the  best  known  guarantees  for  many  classes.
\end{enumerate}

\subsection{Clustering algorithms}\label{sec:clustering}
A clustering instance is made up of a set points $V$ from a data domain $\cX$ and a distance metric $d : \cX \times \cX \to \mathbb{R}_{\geq 0}$. The goal is to split up the points into groups, or ``clusters,'' so that within
each group, distances are minimized and between each group, distances are maximized.
Typically, a clustering's quality is quantified by some objective function. Classic choices include the $k$-means, $k$-median, or $k$-center objective functions. Unfortunately, finding the clustering that minimizes any one of these objectives is NP-hard.
Clustering algorithms have uses in data science, computational biology~\citep{Navlakha09:Finding}, and many other fields.

\citet{Balcan17:Learning,Balcan20:LearningToLink} analyze \emph{agglomerative clustering algorithms.}
This type of algorithm requires a merge function $c(A,B; d) \to \R_{\geq 0}$, defining the distances between point sets $A,B \subseteq V$.
The algorithm constructs a \emph{cluster tree}. This tree starts with $n$ leaf nodes, each containing a point from $V$. Over a series of rounds, the algorithm merges the sets with minimum distance according to $c$. The tree is complete when
there is one node remaining, which consists of the set $V$.
The children of each internal node consist of the two sets merged to create the node. There are several common merge function $c$:
$\min_{a\in A,b\in B} d(a,b)$ (single-linkage),  $\frac{1}{|A|\cdot |B|}\sum_{a\in A,b\in B}d(a,b)$ (average-linkage), and $\max_{a\in A,b\in B} d(a,b)$ (complete-linkage).
Following the linkage procedure, there is a dynamic programming step. This steps finds the tree pruning that minimizes an objective function, such as the $k$-means, -median, or -center objectives.

To evaluate the quality of a clustering, we assume access to a utility function $u: \cT \to [-1,1]$ where $\cT$ is the set of all cluster trees over the data domain $\cX$. For example, $u\left(T\right)$ might measure the distance between the ground truth clustering and the optimal $k$-means pruning of the cluster tree $T \in \cT$.

In Section~\ref{sec:merge}, we present results for learning merge functions and in Section~\ref{sec:learningToLink}, we present results for learning distance functions in addition to merge functions. The latter set of results apply to a special subclass of merge functions called \emph{two-point-based} (as we describe in Section~\ref{sec:learningToLink}), and thus do not subsume the results in Section~\ref{sec:merge}, but do apply to the more general problem of learning a distance function in addition to a merge function.

\subsubsection{Learning merge functions}\label{sec:merge}

\citet*{Balcan17:Learning} study several families of merge functions:
\begin{align*}
	\cC_1 & =\left\lbrace\left. c_{1, \rho} : (A,B;d) \mapsto \left(
	\min_{u \in A, v \in B}(d(u,v))^{\rho} + \max_{u \in A, v \in B}(d(u,v))^
	\rho \right)^{1/\rho}\, \right| \, \rho\in\mathbb{R}\cup\{\infty, -\infty\}\right\rbrace,                                                       \\
	\cC_2 & =\left\lbrace \left. c_{2, \rho} : (A,B;d) \mapsto \rho\min_{u\in A,v\in B}d(u,v)+(1-\rho)\max_{u\in A,v\in B}d(u,v)\, \right| \,
	\rho\in[0,1]\right\rbrace,                                                                                                                      \\
	\cC_3 & =\left\lbrace c_{3, \rho} : (A,B;d) \mapsto \left(
	\left. \frac{1}{|A||B|}\sum_{u \in A, v \in B} \left(d(u, v)\right)^{\rho}\right)^{1/\rho} \, \right| \, \rho \in \mathbb{R} \cup \{\infty, -\infty \}\right\rbrace.
\end{align*}

The classes
$\cC_1$ and $\cC_2$ interpolate between single- ($c_{1,-\infty}$ and $c_{2,1}$) and complete-linkage ($c_{1,\infty}$ and $c_{2,0}$). The class
$\cC_3$ includes as special cases average-, complete-, and single-linkage.

For each class $i \in \{1, 2, 3\}$ and each parameter $\rho$, let $A_{i, \rho}$ be the algorithm that takes as input a clustering instance $(V, d)$ and returns a cluster tree $A_{i, \rho}(V,d) \in \cT$.

\citet{Balcan17:Learning} prove the following useful structure about the classes $\cC_1$ and $\cC_2$:

\begin{lemma}[\citep{Balcan17:Learning}]
	Let $(V, d)$ be an arbitrary clustering instance over $n$ points. There is a partition of $\R$ into $k \leq n^8$ intervals $I_1, \dots, I_k$ such that for any interval $I_j$ and any two parameters $\rho, \rho' \in I_j$, the sequences of merges the agglomerative clustering algorithm makes using the merge functions $c_{1, \rho}$ and $c_{1, \rho'}$ are identical. The same holds for the set of merge functions $\cC_2$.
\end{lemma}

This structure immediately implies that the corresponding class of utility functions has a piecewise-structured dual class.

\begin{cor}\label{cor:simple_clustering}
Let $\cU$ be the set of functions \[\cU = \left\{u_{\rho} : (V, d) \mapsto u\left(A_{1, \rho}(V, d)\right) \mid \rho \in \R \cup \{- \infty, \infty\}\right\}\] mapping clustering instances $(V, d)$ to $[-1,1]$. The dual class $\cU^*$ is $(\cF, \cG, n^8)$-piecewise decomposable, where $\cG = \{g_{a} : \cU \to \{0,1\} \mid a \in \R\}$ consists of threshold functions $g_{a} : u_{\rho} \mapsto \ind{\rho < a}$ and $\cF = \{f_c : \cU \to \R \mid c \in \R\}$ consists of constant functions $f_c : u_{\rho} \mapsto c$. The same holds when $\cU$ is defined according to merge functions in $\cC_2$ as $\cU = \left\{u_{\rho} : (V, d) \mapsto u\left(A_{2, \rho}(V, d)\right) \mid \rho \in [0,1]\right\}.$
\end{cor}

Lemma~\ref{lem:oscillate_decomp} and Corollary~\ref{cor:simple_clustering} imply the following pseudo-dimension bound.

\begin{cor}\label{cor:clustering_pdim_simple}
Let $\cU$ be the set of functions
	\[\cU = \left\{u_{\rho} : (V, d) \mapsto u\left(A_{1, \rho}(V,
	d)\right) \mid \rho \in \R \cup \{- \infty, \infty\}\right\}\]
	mapping clustering instances $(V, d)$ to $[-1,1]$.
	Then $\pdim(\cU) = O(\ln n)$.
	The same holds when $\cU$ is defined according to merge functions in $\cC_2$ as $\cU = \left\{u_{\rho} : (V, d) \mapsto u\left(A_{2, \rho}(V, d)\right) \mid \rho \in [0,1]\right\}.$
\end{cor}

\citet{Balcan17:Learning} prove a similar guarantee for the more complicated class $\cC_3$.

\begin{lemma}[\citep{Balcan17:Learning}]
	Let $(V, d)$ be an arbitrary clustering instance over $n$ points. There is a partition of $\R$ into $k \leq n^23^{2n}$ intervals $I_1, \dots, I_k$ such that for any interval $I_j$ and any two parameters $\rho, \rho' \in I_j$, the sequences of merges the agglomerative clustering algorithm makes using the merge functions $c_{3, \rho}$ and $c_{3, \rho'}$ are identical.
\end{lemma}

Again, this structure immediately implies that the corresponding class of utility functions has a piecewise-structured dual class.

\begin{cor}\label{cor:clustering_complex}
Let $\cU$ be the set of functions \[\cU = \left\{u_{\rho} : (V, d) \mapsto u\left(A_{3, \rho}(V, d)\right) \mid \rho \in \R \cup \{- \infty, \infty\}\right\}\] mapping clustering instances $(V, d)$ to $[-1,1]$. The dual class $\cU^*$ is $\left(\cF, \cG, n^23^{2n}\right)$-piecewise decomposable, where $\cG = \{g_{a} : \cU \to \{0,1\} \mid a \in \R\}$ consists of threshold functions $g_{a} : u_{\rho} \mapsto \ind{\rho < a}$ and $\cF = \{f_c : \cU \to \R \mid c \in \R\}$ consists of constant functions $f_c : u_{\rho} \mapsto c$.
\end{cor}

Lemma~\ref{lem:oscillate_decomp} and Corollary~\ref{cor:clustering_complex} imply the following pseudo-dimension bound.

\begin{cor}\label{cor:clustering_pseudo_complex}
Let $\cU$ be the set of functions
	\[\cU = \left\{u_{\rho} : (V, d) \mapsto u\left(A_{3, \rho}(V,
	d)\right) \mid \rho \in \R \cup \{- \infty, \infty\}\right\}\]
	mapping clustering instances $(V, d)$ to $[-1,1]$.
	Then $\pdim(\cU) = O(n)$.
\end{cor}

Corollaries~\ref{cor:clustering_pdim_simple} and \ref{cor:clustering_pseudo_complex} match the pseudo-dimension guarantees that \citet{Balcan17:Learning} prove.

\subsubsection{Learning merge functions and distance functions}\label{sec:learningToLink}
	\citet*{Balcan20:LearningToLink} extend the clustering generalization bounds of \citet*{Balcan17:Learning} to the case of learning both a distance metric and a merge function. 
	They introduce a family of linkage-based clustering algorithms that simultaneously interpolate between a collection of base metrics $d_1, \dots, d_L$ and base merge functions $c_1, \dots, c_L$. 
	The algorithm family is parameterized by $\vec{\rho} = (\vec{\alpha}, \vec{\beta}) \in \Delta_{L'} \times \Delta_{L}$, where $\vec{\alpha}$ and $\vec{\beta}$ are mixing weights for the merge functions and metrics, respectively. 
	The algorithm with parameters $\vec{\rho} = (\vec{\alpha}, \vec{\beta})$ starts with each point in a cluster of its own and repeatedly merges the pair of clusters $A$ and $B$ minimizing $c_{\vec{\alpha}}(A, B; d_{\vec{\beta}})$, where 
	\[
	c_{\vec{\alpha}}(A,B; d) = \sum_{i=1}^{L'} \alpha_i \cdot c_i(A, B; d)
	\qquad\text{and}\qquad  
	d_{\vec{\beta}}(a, b) = \sum_{i=1}^L \beta_i \cdot d_i(a,b).
	\]
We use the notation $A_{\vec{\rho}}$ to denote the algorithm that takes as input a clustering instance $(V, d)$ and returns a cluster tree $A_{\vec{\rho}}(V,d) \in \cT$ using the merge function $c_{\vec{\alpha}}(A, B; d_{\vec{\beta}})$, where $\vec{\rho} = (\vec{\alpha}, \vec{\beta})$.
	
When analyzing this algorithm family, \citet{Balcan20:LearningToLink} prove that the following piecewise-structure holds when all of the merge functions are \emph{two-point-based}, which roughly requires that for any pair of clusters $A$ and $B$, there exist points $a \in A$ and $b \in B$ such that $c(A,B;d) = d(a,b)$. Single- and complete-linkage are two-point-based, but average-linkage is not.
	\begin{lemma}[\citep{Balcan20:LearningToLink}]
	For any clustering instance $V$, there exists a collection of $O\left(|V|^{4L'}\right)$ quadratic boundary functions that partition the $(L+L')$-dimensional parameter space into regions where the algorithm's output is constant on each region in the partition.
	\end{lemma}

This lemma immediately implies that the corresponding class of utility functions has a piecewise-structured dual class.

\begin{cor}\label{cor:linkage}
	Let $\cU$ be the set of functions $\cU = \left\{u_{\vec{\rho}} : (V, d) \mapsto u\left(A_{\vec{\rho}}(V, d)\right) \mid \vec{\rho} \in \Delta_{L'} \times \Delta_{L}\right\}$ mapping clustering instances $(V, d)$ to $[-1,1]$. The dual class $\cU^*$ is $\left(\cF, \cG, O\left(|V|^{4L'}\right)\right)$-piecewise decomposable, where $\cF$ is the set of constant functions and $\cG$ is the set of quadratic functions defined on $\Delta_{L'} \times \Delta_{L}$.
\end{cor}
	
	Using the fact that $\VC\left(\cG^*\right) = O((L+L')^2)$, 
	we obtain the following pseudo-dimension bound.
	\begin{cor}
		Let $\cU$ be the set of functions $\cU = \left\{u_{\vec{\rho}} : (V, d) \mapsto u\left(A_{\vec{\rho}}(V, d)\right) \mid \vec{\rho} \in \Delta_{L'} \times \Delta_{L}\right\}$ 
		mapping clustering instances $(V, d)$ to $[-1,1]$.
		Then \[\pdim(\cU) = O\left(\left(L+L'\right)^2 \log\left(L + L'\right) + \left(L+L'\right)^2 L' \log(n)\right).\]
	\end{cor}
	
This matches the generalization bound that~\citet{Balcan20:LearningToLink} prove.

\subsection{Integer programming}\label{sec:IP}
Several papers~\citep{Balcan17:Learning,Balcan18:Learning} study data-driven algorithm design for both integer linear and integer quadratic programming, as we describe below.

\paragraph{Integer linear programming.} In the context of integer linear programming, \citet{Balcan18:Learning} focus on branch-and-bound (B\&B)~\citep{Land60:Automatic}, an algorithm for solving mixed integer linear programs (MILPs). A MILP is defined by a matrix $A \in \R^{m \times n}$, a vector $\vec{b} \in \R^m$, a vector $\vec{c}\in \R^n$, and a set of indices $I \subseteq [n]$. The goal is to find a vector $\vec{x} \in \R^n$ such that $\vec{c}\cdot \vec{x}$ is maximized, $A\vec{x} \leq \vec{b}$, and for every index $i \in I$, $x_i$ is constrained to be binary: $x_i \in \{0,1\}$.

Branch-and-bound builds a search tree to solve an input MILP $Q$. At the root of the search tree is the original MILP $Q$. At each round, the algorithm chooses a leaf of the search tree, which represents an MILP $Q'$. It does so using a \emph{node selection policy}; common choices include depth- and best-first search. Then, it chooses an index $i \in I$ using a \emph{variable selection policy}. It next \emph{branches} on $x_i$: it sets the left child of $Q'$ to be that same integer program, but with the additional constraint that $x_i = 0$, and it sets the right child of $Q'$ to be that same integer program, but with the additional constraint that $x_i = 1$. The algorithm \emph{fathoms} a leaf, which means that it never will branch on that leaf, if it can guarantee that the optimal solution does not lie along that path. The algorithm terminates when it has fathomed every leaf. At that point, we can guarantee that the best solution to $Q$ found so far is optimal. See the paper by \citet{Balcan18:Learning} for more details.

\citet{Balcan18:Learning} study \emph{mixed integer linear programs} (MILPs)
where the goal is to maximize an objective function $\vec{c}^\top \vec{x}$ subject to the constraints that $A\vec{x} \leq \vec{b}$ and that some of the components of $\vec{x}$ are contained in $\{0,1\}$.
Given a MILP $Q$, we use the notation $\breve{\vec{x}}_Q = \left(\breve{x}_{Q}[1], \dots \breve{x}_{Q}[n]\right)$ to denote an optimal solution to the MILP's LP relaxation. We denote the optimal objective value to the MILP's LP relaxation as $\breve{c}_Q$, which means that $\breve{c}_Q = \vec{c}^{\top} \breve{\vec{x}}_Q$.

Branch-and-bound systematically partitions the feasible set in order to find an optimal solution, organizing the partition as a tree. At the root of this tree is the original integer program. Each child represents the simplified integer program obtained by partitioning the feasible set of the problem contained in the parent node. The algorithm prunes a branch if the corresponding subproblem is infeasible or its optimal solution cannot be better than the best one discovered so far.
Oftentimes, branch-and-bound partitions the feasible set by adding a constraint. For example, if the feasible set is characterized by the constraints $A\vec{x} \leq \vec{b}$ and $\vec{x} \in \{0,1\}^n$, the algorithm partition the feasible set into one subset where $A \vec{x} \leq \vec{b}$, $x_1 = 0$, and $x_2, \dots, x_n \in \{0,1\}$, and another where $A \vec{x} \leq \vec{b}$, $x_1 = 1$, and $x_2, \dots, x_n \in \{0,1\}$. In this case, we say that the algorithm \emph{branches} on $x_1$.

\citet{Balcan18:Learning} show how to learn variable selection policies. Specifically, they study \emph{score-based variable selection policies}, defined below.
\begin{definition}[Score-based variable selection policy~\citep{Balcan18:Learning}]
	Let $\score$ be a deterministic function that takes as input a partial search tree $\tree$, a leaf $Q$ of that tree, and an index $i$, and returns a real value $\score(\tree, Q, i) \in \R$. For a leaf $Q$ of a tree $\tree$, let $N_{\tree, Q}$ be the set of variables that have not yet been branched on along the path from the root of $\tree$ to $Q$. A score-based variable selection policy selects the variable $\argmax_{x_i \in N_{\tree, Q}} \{\score(\tree, Q, i)\}$ to branch on at the node $Q$.
\end{definition}
This type of variable selection policy is widely used~\citep{Linderoth99:Computational, Achterberg09:SCIP, Gilpin11:Information}. See the paper by \citet{Balcan18:Learning} for examples.

Given $d$ arbitrary scoring rules $\score_1, \dots, \score_d$, \citet{Balcan18:Learning} provide guidance for learning a linear combination $\rho[1]\score_1 + \cdots + \rho[d]\score_d$ that leads to small expected tree sizes. They assume that all aspects of the tree search algorithm except the variable selection policy, such as the node selection policy, are fixed. In their analysis, they prove the following lemma.

\begin{lemma}[\citep{Balcan18:Learning}]
	Let $\score_1, \dots, \score_d$ be $d$ arbitrary scoring rules and let $Q$ be an arbitrary MILP over $n$ binary variables. Suppose we limit B\&B to producing search trees of
	size $\tau$. There is a set $\mathcal{H}$ of at most $n^{2(\tau + 1)}$ hyperplanes such that for any connected component $R$ of $[0,1]^d \setminus \mathcal{H}$, the search tree B\&B builds using the scoring rule $\rho[1]\score_1 + \cdots + \rho[d]\score_d$ is invariant across all $(\rho[1], \dots, \rho[d]) \in R$.
\end{lemma}

This piecewise structure immediately implies the following guarantee.

\begin{cor}\label{cor:BB}
	Let $\score_1, \dots, \score_d$ be $d$ arbitrary scoring rules and let $Q$ be an arbitrary MILP over $n$ binary variables. Suppose we limit B\&B to producing search trees of
	size $\tau$. For each parameter vector $\vec{\rho} = (\rho[1], \dots, \rho[d]) \in [0,1]^d$, let $u_{\vec{\rho}}(Q)$ be the size of the tree, divided by $\tau$, that B\&B builds using the scoring rule $\rho[1]\score_1 + \cdots + \rho[d]\score_d$ given $Q$ as input. Let $\cU$ be the set of functions $\cU = \left\{u_{\rho} \mid \vec{\rho} \in [0,1]^d\right\}$ mapping MILPs to $[0,1]$. The dual class $\cU^*$ is $\left(\cF,\cG, n^{2(\tau + 1)}\right)$-piecewise decomposable, where $\cG = \{g_{\vec{a}, \theta} : \cU \to \{0,1\} \mid \vec{a} \in \R^d, \theta \in \R\}$ consists of halfspace indicator functions $g_{\vec{a}, \theta} : u_{\vec{\rho}} \mapsto \ind{\vec{\rho} \cdot \vec{a} \leq \theta}$ and $\cF = \{f_c : \cU \to \R \mid c \in \R\}$ consists of constant functions $f_c : u_{\vec{\rho}} \mapsto c$.
\end{cor}

Corollary~\ref{cor:BB} and Lemma~\ref{lem:simple_multidim} imply the following pseudo-dimension bound.

\begin{cor}\label{cor:BB_pseudo}
	Let $\score_1, \dots, \score_d$ be $d$ arbitrary scoring rules and let
	$Q$ be an arbitrary MILP over $n$ binary variables. Suppose we limit
	B\&B to producing search trees of size $\tau$. For each parameter
	vector $\vec{\rho} = (\rho[1], \dots, \rho[d]) \in [0,1]^d$, let
	$u_{\vec{\rho}}(Q)$ be the size of the tree, divided by $\tau$, that
	B\&B builds using the scoring rule
	$\rho[1]\score_1 + \cdots + \rho[d]\score_d$ given $Q$ as input. Let
	$\cU$ be the set of functions
	$\cU = \left\{u_{\rho} \mid \vec{\rho} \in [0,1]^d\right\}$ mapping
	MILPs to $[0,1]$.
	Then $\pdim(\cU) = O(d (\tau \ln(n) + \ln(d)))$.
\end{cor}

Corollary~\ref{cor:BB_pseudo} matches the pseudo-dimension guarantee that \citet{Balcan18:Learning} prove.

\paragraph{Integer quadratic programming.} A diverse array of NP-hard problems, including max-2SAT, max-cut, and correlation clustering, can be characterized as integer quadratic programs (IQPs). An IQP is represented by a matrix $A \in \R^{n \times n}$. The goal is to find a set $X = \{x_1, \dots, x_n\} \in \{-1,1\}^n$ maximizing $\sum_{i,j \in [n]} a_{ij}x_ix_j$.
The most-studied IQP approximation algorithms operate via an SDP relaxation:
\begin{equation}\label{eq:SDP}\text{maximize } \sum_{i,j \in [n]} a_{ij}\langle \vec{u}_i , \vec{u}_j\rangle \qquad \text{subject to }\vec{u}_i \in S^{n-1}.\end{equation}
The approximation algorithm must transform, or ``round,'' the unit vectors into a binary assignment of the variables $x_1, \dots, x_n$. In the seminal \emph{GW algorithm}~\citep{Goemans95:Improved}, the algorithm projects the unit vectors onto a random vector $\vec{Z}$, which it draws from the $n$-dimensional Gaussian distribution, which we denote using $\vec{Z}$. If $\left\langle\vec{u}_i, \vec{Z} \right\rangle > 0$, it sets $x_i = 1$. Otherwise, it sets $x_i = -1$.

The GW algorithm's approximation ratio can sometimes be improved if the algorithm probabilistically assigns the binary variables. In the final step, the algorithm can use any rounding function $r: \R \to [-1,1]$ to set $x_i = 1$ with probability $\frac{1}{2} + \frac{1}{2}\cdot r\left(\langle \vec{Z}, \vec{u}_i \rangle\right)$ and $x_i = -1$ with probability $\frac{1}{2} - \frac{1}{2}\cdot r\left(\langle \vec{Z}, \vec{u}_i \rangle\right)$. See Algorithm~\ref{alg:GW} for the pseudocode.
\begin{algorithm}[t]
	\caption{SDP rounding algorithm with rounding function $r$}\label{alg:GW}
	\begin{algorithmic}[1]
		\Require Matrix $A \in \R^{n \times n}$.
		\State Draw a random vector $\vec{Z}$ from $\mathcal{Z}$, the $n$-dimensional Gaussian distribution.\label{step:draw}
		\State Solve the SDP (\ref{eq:SDP}) for the optimal embedding $U = \left\{\vec{u}_1, \dots, \vec{u}_n\right\}$.
		\State Compute set of fractional assignments $r(\langle \vec{Z}, \vec{u}_1\rangle), \dots, r(\langle \vec{Z}, \vec{u}_n\rangle)$.\label{step:fractional}
		\State For all $i \in [n]$, set $x_i$ to 1 with probability $\frac{1}{2} + \frac{1}{2}\cdot r\left(\langle \vec{Z}, \vec{u}_i \rangle\right)$ and $-1$ with probability $\frac{1}{2} - \frac{1}{2}\cdot r\left(\langle \vec{Z}, \vec{u}_i \rangle\right)$.\label{step:round}
		\Ensure $x_1, \dots, x_n$.
	\end{algorithmic}
\end{algorithm}
Algorithm~\ref{alg:GW} is known as a \emph{Random Projection, Randomized Rounding} (RPR$^2$) algorithm, so named by the seminal work of \citet{Feige06:RPR}.

\citet{Balcan17:Learning} analyze $s$-linear rounding functions~\citep{Feige06:RPR} $\phi_s:\R \to [-1,1]$, parameterized by $s > 0$, defined as follows:
\[\phi_s(y) = \begin{cases} -1  & \text{if } y < -s           \\
	y/s & \text{if } -s \leq y \leq s \\
	1   & \text{if } y > s.\end{cases}\]

The goal is to learn a parameter $s$ such that in expectation, $\sum_{i, j \in [n]} a_{ij} x_i x_j$ is maximized. The expectation is over several sources of randomness: first, the distribution $\dist$ over matrices $A$; second, the vector $\vec{Z}$; and third, the assignment of $x_1, \dots, x_n$. This final assignment depends on the parameter $s$, the matrix $A$, and the vector $\vec{Z}$.  \citet{Balcan17:Learning} refer to this value as the \emph{true utility of the parameter $s$}. Note that the distribution over matrices, which defines the algorithm's input, is unknown and external to the algorithm, whereas the Gaussian distribution over vectors as well as the distribution defining the variable assignment are internal to the algorithm.

The distribution over matrices is unknown, so we cannot know any parameter's true utility. Therefore, to learn a good parameter $s$, we must use samples. \citet{Balcan17:Learning} suggest drawing samples from two sources of randomness:  the distributions over vectors and matrices. In other words, they suggest drawing a set of samples $\sample = \left\{\left(A^{(1)}, \vec{Z}^{(1)}\right), \dots, \left(A^{(m)}, \vec{Z}^{(m)}\right)\right\} \sim \left(\dist \times \mathcal{Z}\right)^m.$ Given these samples, \citet{Balcan17:Learning} define a parameter's \emph{empirical utility} to be the expected objective value of the solution Algorithm~\ref{alg:GW} returns given input $\qp$, using the vector $\vec{Z}$ and $\phi_{s}$ in Step~\ref{step:fractional}, on average over all $(A, \vec{Z}) \in \sample$. Generally speaking, \citet{Balcan17:Learning} suggest sampling the first two randomness sources in order to isolate the third randomness source. They argue that this third source of randomness has an expectation that is simple to analyze. Using pseudo-dimension, they prove that every parameter $s$, its empirical and true utilities converge.

A bit more formally, \citet{Balcan17:Learning} use the notation $p_{(i, \vec{Z}, A, s)}$ to denote the distribution that the binary value $x_i$ is drawn from when Algorithm~\ref{alg:GW} is given $\qp$ as input and uses the rounding function $r = \phi_s$ and the hyperplane $\vec{Z}$ in Step~\ref{step:fractional}. Using this notation, the parameter $s$ has a true utility of \[\E_{A, \vec{Z} \sim \dist \times \mathcal{Z}}\left[\E_{x_i \sim p_{(i, \vec{Z}, A, s)}}\left[\sum_{i,j} a_{ij} x_i x_j\right]\right].\footnote{We, like \citet{Balcan17:Learning}, use the abbreviated notation \[\E_{A, \vec{Z} \sim \dist \times \mathcal{Z}}\left[\E_{x_i \sim p_{(i, \vec{Z}, A, s)}}\left[\sum_{i,j} a_{ij} x_i x_j\right]\right] = \E_{A, \vec{Z} \sim \dist \times \mathcal{Z}}\left[\E_{x_1 \sim p_{(1, \vec{Z}, A, s)}, \dots, x_n \sim p_{(n, \vec{Z}, A, s)}}\left[\sum_{i,j} a_{ij} x_i x_j\right]\right].\]}\] We also use the notation $\slin_s(A, \vec{Z})$ to denote the expected objective value of the solution Algorithm~\ref{alg:GW} returns given input $\qp$, using the vector $\vec{Z}$ and $\phi_{s}$ in Step~\ref{step:fractional}. The expectation is over the final assignment of each variable $x_i$. Specifically, $\slin_s(A, \vec{Z}) = \E_{x_i \sim p_{(i, \vec{Z}, A, s)}}\left[\sum_{i,j} a_{ij} x_i x_j\right]$. By definition, a parameter's true utility equals $\E_{A, \vec{Z} \sim \dist \times \mathcal{Z}}\left[\slin_s(A, \vec{Z})\right]$.
Given a set $\left(A^{(1)}, \vec{Z}^{(1)}\right), \dots, \left(A^{(m)}, \vec{Z}^{(m)}\right) \sim \dist \times \mathcal{Z}$, a parameter's empirical utility is $\frac{1}{m} \sum_{i = 1}^m \slin_s\left(A^{(i)}, \vec{Z}^{(i)}\right)$.

Both we and \citet{Balcan17:Learning} bound the pseudo-dimension of the function class $\cU = \left\{\slin_s : s > 0\right\}$.
\citet{Balcan17:Learning} prove that the functions in $\cU$ are piecewise structured: roughly speaking, for a fixed matrix $A$ and vector $\vec{Z}$, each function in $\cU$ is a piecewise, inverse-quadratic function of the parameter $s$. To present this lemma, we use the following notation: given a tuple $\left(\qp, \vec{Z}\right)$, let $\slin_{\qp, \vec{Z}} : \R \to \R$ be defined such that $\slin_{\qp, \vec{Z}}(s) = \slin_s\left(\qp, \vec{Z}\right)$.

\begin{lemma}[\citep{Balcan17:Learning}]\label{lem:piecewise_quad}
	For any matrix $A$ and vector $\vec{Z}$, the function $\slin_{\qp, \vec{Z}}:\R_{>0} \to \R$ is made up of $n+1$ piecewise components of the form $\frac{a}{s^2} + \frac{b}{s} + c$ for some $a,b,c \in \R$. Moreover, if the border between two components falls at some $s \in \R_{>0}$, then it must be that $s = \left|\langle \vec{u}_i, \vec{Z}\rangle \right|$ for some $\vec{u}_i$ in the optimal
	SDP embedding of $\qp$.
\end{lemma}

This piecewise structure immediately implies the following corollary about the dual class $\cU^*$.

\begin{cor}\label{cor:slin}
	Let $\cU$ be the set of functions $\cU = \left\{\slin_s : s > 0\right\}$. The dual class $\cU^*$ is $\left(\cF, \cG, n\right)$-piecewise decomposable, where $\cG = \{g_{a} : \cU \to \{0,1\} \mid a \in \R\}$ consists of threshold functions $g_{a} : u_{s} \mapsto \ind{s \leq a}$ and $\cF = \{f_{a,b,c} : \cU \to \R \mid a,b,c \in \R\}$ consists of inverse-quadratic functions $f_{a,b,c} : u_{s} \mapsto \frac{a}{s^2} + \frac{b}{s} + c$.
\end{cor}

Lemma~\ref{lem:oscillate_decomp} and Corollary~\ref{cor:slin} imply the following pseudo-dimension bound.

\begin{cor}\label{cor:slin_pseudo}
	Let $\cU$ be the set of functions
	$\cU = \left\{\slin_s : s > 0\right\}$. The pseudo-dimension of
	$\cU$ is at most $O(\ln n)$.
\end{cor}

Corollary~\ref{cor:slin_pseudo} matches the pseudo-dimension bound that \citet{Balcan17:Learning} prove.

\subsection{Greedy algorithms}\label{sec:greedy}
\citet{Gupta17:PAC} provide pseudo-dimension bounds for greedy algorithm configuration, analyzing two canonical combinatorial problems: the maximum weight independent set problem and the knapsack problem. We recover their bounds in both cases.

\paragraph{Maximum weight independent set (MWIS)} In the MWIS problem, the input is a graph $G$ with a weight $w\left(v\right) \in \R_{\geq 0}$ per vertex $v$. The objective is to
find a maximum-weight (or \emph{independent}) set of non-adjacent vertices. On each iteration, the classic greedy algorithm adds the vertex $v$ that maximizes $w\left(v\right)/\left(1 + \deg\left(v\right)\right)$ to the set. It then removes $v$ and its neighbors from the graph.
Given a parameter $\rho \geq 0$, \citet{Gupta17:PAC} propose
the greedy heuristic $w\left(v\right)/\left(1 + \deg\left(v\right)\right)^{\rho}$. 
In this context, the utility function $u_{\rho}(G, w)$ equals the weight of the vertices in the set returned by the algorithm parameterized by $\rho$.
\citet{Gupta17:PAC} implicitly prove the following lemma about each function $u_{\rho}$ (made explicit in work by \citet{Balcan18:Dispersion}). To present this lemma, we use the following notation: let $u_{G, w} : \R \to \R$ be defined such that $u_{G,w}(\rho) = u_{\rho}\left(G,w\right)$.

\begin{lemma}[\citep{Gupta17:PAC}]
	For any weighted graph $\left(G,w\right)$, the function $u_{G,w} : \R \to \R$ is piecewise constant with at most $n^4$ discontinuities.
\end{lemma}

This structure immediately implies that the function class $\cU = \left\{u_{\rho} : \rho > 0\right\}$ has a piecewise-structured dual class.

\begin{cor}\label{cor:MWIS}
	Let $\cU$ be the set of functions $\cU = \left\{u_{\rho} : \rho > 0\right\}$.
	The dual class $\cU^*$ is $(\cF, \cG, n^4)$-piecewise decomposable, where $\cG = \{g_{a} : \cU \to \{0,1\} \mid a \in \R\}$ consists of threshold functions $g_{a} : u_{\rho} \mapsto \ind{\rho < a}$ and $\cF = \{f_c : \cU \to \R \mid c \in \R\}$ consists of constant functions $f_c : u_{\rho} \mapsto c$.
\end{cor}

Lemma~\ref{lem:oscillate_decomp} and Corollary~\ref{cor:MWIS} imply the following pseudo-dimension bound.

\begin{cor}
	Let $\cU$ be the set of functions
	$\cU = \left\{u_{\rho} : \rho > 0\right\}$. The pseudo-dimension of
	$\cU$ is $O(\ln n)$.
\end{cor}

This matches the pseudo-dimension bound by \citet{Gupta17:PAC}.

\paragraph{Knapsack.} Moving to the classic knapsack problem, the input is a knapsack capacity $C$ and a set of $n$ items
$i$ each with a value $\nu_i$ and a size $s_i$. The goal is to determine a set $I
\subseteq \left\{1, \dots, n\right\}$ with maximium total value $\sum_{i \in I} \nu_i$ such
that $\sum_{i \in I} s_i \leq C$.
\citet{Gupta17:PAC} suggest the family of algorithms parameterized by $\rho >0$ where each algorithm returns the better of the following two
solutions:
\begin{itemize}
	\item Greedily pack items in order of nonincreasing value $\nu_i$ subject to feasibility.
	\item Greedily pack items in order of $\nu_i/s_i^{\rho}$ subject to feasibility.
\end{itemize}
It is well-known that the algorithm with $\rho= 1$ achieves a 2-approximation.
We use the notation $u_{\rho}\left(\vec{\nu}, \vec{s}, C\right)$ to denote the
total value of the items returned by the algorithm parameterized by $\rho$ given input $(\vec{\nu}, \vec{s}, C)$.

\citet{Gupta17:PAC} implicitly prove the following fact about the functions $u_{\rho}$
(made explicit in work by \citet{Balcan18:Dispersion}). To present this lemma, we use the following notation: given a tuple $\left(\vec{\nu}, \vec{s}, C\right)$, let $u_{\vec{\nu}, \vec{s}, C} : \R \to \R$ be defined such that $u_{\vec{\nu}, \vec{s}, C}(\rho) = u_{\rho}\left(\vec{\nu}, \vec{s}, C\right)$.

\begin{lemma}[\citep{Gupta17:PAC}]
	For any tuple $\left(\vec{\nu}, \vec{s}, C\right)$, the function $u_{\vec{\nu}, \vec{s}, C} : \R \to \R$ is piecewise constant with at most $n^2$ discontinuities.
\end{lemma}

This structure immediately implies that the function class $\cU = \left\{u_{\rho} : \rho > 0\right\}$ has a piecewise-structured dual class.

\begin{cor}\label{cor:knapsack}
	Let $\cU$ be the set of functions $\cU = \left\{u_{\rho} : \rho > 0\right\}$.
	The dual class $\cU^*$ is $(\cF, \cG, n^2)$-piecewise decomposable, where $\cG = \{g_{a} : \cU \to \{0,1\} \mid a \in \R\}$ consists of threshold functions $g_{a} : u_{\rho} \mapsto \ind{\rho < a}$ and $\cF = \{f_c : \cU \to \R \mid c \in \R\}$ consists of constant functions $f_c : u_{\rho} \mapsto c$.
\end{cor}

Lemma~\ref{lem:oscillate_decomp} and Corollary~\ref{cor:knapsack} imply the following pseudo-dimension bound.

\begin{cor}
	Let $\cU$ be the set of functions
	$\cU = \left\{u_{\rho} : \rho > 0\right\}$. The pseudo-dimension of
	$\cU$ is $O(\ln n)$.
\end{cor}

This matches the pseudo-dimension bound by \citet{Gupta17:PAC}.

\subsection{Revenue maximization}\label{sec:revenue}
The design of revenue-maximizing multi-item mechanisms is a notoriously challenging problem. Remarkably, the revenue-maximizing mechanism is not known even when there are just two items for sale. In this setting, the mechanism designer's goal is to field a mechanism
with high expected revenue on the distribution over agents' values. A line of research has provided generalization guarantees for mechanism design in the context of revenue maximization~\citep{Elkind07:Designing,Cole14:Sample,Devanur16:Sample,Gonczarowski17:Efficient,Guo19:Settling,Cai17:Learning,Gonczarowski21:Sample,Morgenstern16:Learning,Mohri14:Learning,Balcan16:Sample}.
These papers focus on sales settings: there is a seller, not included among the agents, who will use a mechanism to allocate a set of goods among the agents. The agents submit bids describing their values for the goods for sale. The mechanism determines which agents receive which items and how much the agents pay. The seller's revenue is the sum of the agents' payments. The mechanism designer's goal is to select a mechanism that maximizes the revenue.
In contrast to the mechanisms we analyze in Section~\ref{sec:econ}, these papers study mechanisms that are not necessarily budget-balanced. Specifically, under every mechanism they study, the sum of the agents' payments---the revenue---is at least zero. As in Section~\ref{sec:econ}, all of the mechanisms they analyze are incentive compatible.

We study the problem of selling $m$ heterogeneous goods to $n$ buyers. They denote a bundle of goods as a subset $b \subseteq [m]$. Each buyer $j \in [n]$ has a valuation function $v_j : 2^{[m]} \to \R$ over bundles of goods. In this setting, the set $\cX$ of problem instances consists of $n$-tuples of buyer values $\vec{v} = \left(v_1, \dots, v_n\right)$.
As in Section~\ref{sec:econ}, selling mechanisms defined by an \emph{allocation function} and a set of \emph{payment functions}. Every auction in the classes they study is incentive compatible, so they assume that the bids equal the bidders' valuations. An allocation function $\psi: \cX \to \left(2^{[m]}\right)^n$ maps the values $\vec{v} \in \cX$ to a division of the goods $\left(b_1, \dots, b_n\right) \in \left(2^{[m]}\right)^n$, where $b_i \subseteq [m]$ is the set of goods buyer $i$ receives. For each agent $i \in [n]$, there is a payment function $p_i: \cX \to \R$ which maps values $\vec{v} \in \cX$ to a payment $p_i(\vec{v}) \in \R_{\geq 0}$ that agent $i$ must make.

We recover generalization guarantees proved by \citet{Balcan18:General} which apply to a variety of widely-studied parameterized mechanism classes, including posted-price mechanisms,
multi-part tariffs,
second-price auctions with reserves,
affine maximizer auctions,
virtual valuations combinatorial auctions
mixed-bundling auctions,
and randomized mechanisms. They provided new bounds for mechanism classes not  previously studied  in  the  sample-based  mechanism  design  literature  and  matched  or  improved  over the  best  known  guarantees  for  many  classes.
They proved these guarantees by uncovering structure shared by all of these mechanisms: for any set of buyers' values, revenue is a piecewise-linear function of the mechanism's parameters. This structure is captured by our definition of piecewise decomposability.

Each of these mechanism classes is parameterized by a $d$-dimensional vector $\vec{\rho} \in \cP \subseteq \R^d$ for some $d \geq 1$. For example, when $d = m$, $\vec{\rho}$ might be a vector of prices for each of the items. The revenue of a mechanism is the sum of the agents' payments. Given a mechanism parameterized by a vector $\vec{\rho} \in \R^d$, we denote the revenue as $u_{\vec{\rho}} : \cX \to \R$, where $u_{\vec{\rho}}(\vec{v}) = \sum_{i = 1}^n p_i(\vec{v})$.

\citet{Balcan18:General} provide psuedo-dimension bounds for any mechanism class that is \emph{delineable}. To define this notion, for any fixed valuation vector $\vec{v} \in \cX$, we use the notation $u_{\vec{v}}(\vec{\rho})$ to denote revenue as a function of the mechanism's parameters.

\begin{definition}[$\left(d,t\right)$-delineable~\citep{Balcan18:General}] A mechanism class is \emph{$\left(d,t\right)$-delineable} if:
	\begin{enumerate}
		\item The class consists of mechanisms parameterized by vectors $\vec{\rho}$ from a set $\pspace \subseteq \R^d$; and
		\item For any $\vec{v} \in \cX$, there is a set $\cH$ of $t$ hyperplanes such that for any connected component $\pspace'$ of $\pspace \setminus \cH$, the function $u_{\vec{v}}\left(\vec{\rho}\right)$ is linear over $\pspace'.$
	\end{enumerate}
\end{definition}

Delineability naturally translates to decomposability, as we formalize below.

\begin{lemma}\label{lem:delineable}
	Let $\cU$ be a set of revenue functions corresponding to a $(d,t)$-delineable mechanism class. The dual class $\cU^*$ is $\left(\cF, \cG, t\right)$-piecewise decomposable, where $\cG = \{g_{\vec{a}, \theta} : \cU \to \{0,1\} \mid \vec{a} \in \R^d, \theta \in \R\}$ consists of halfspace indicator functions $g_{\vec{a}, \theta} : u_{\vec{\rho}} \mapsto \ind{\vec{\rho} \cdot \vec{a} \leq \theta}$ and $\cF = \{f_{\vec{a}, \theta} : \cU \to \R \mid \vec{a} \in \R^d, \theta \in \R\}$ consists of linear functions $f_{\vec{a}, \theta} : u_{\vec{\rho}} \mapsto \vec{\rho} \cdot \vec{a} + \theta$.
\end{lemma}

Lemmas~\ref{lem:simple_multidim} and \ref{lem:delineable} imply the following bound.

\begin{cor}\label{cor:delineable}
	Let $\cU$ be a set of revenue functions corresponding to a
	$(d,t)$-delineable mechanism class. The pseudo-dimension of $\cU$ is
	at most $O(d \ln (td))$.
\end{cor}

Corollary~\ref{cor:delineable} matches the pseudo-dimension bound that \citet{Balcan18:General} prove.

%% file: experiments.tex
We complement our theoretical guarantees with experiments in several settings to help illustrate our bounds.
Our experiments are in the contexts of both new and previously-studied domains: tuning parameters of sequence alignment algorithms in computational biology, tuning parameters of sales mechanisms so as to maximize revenue, and tuning parameters of voting mechanisms so as to maximize social good. We summarize two observations from the experiments that help illustrate the theoretical message of this paper.

\begin{observation}\label{obs:estimate}
	\textnormal{First, using sequence alignment data (Section~\ref{sec:experiments_sequence}), we demonstrate that with a finite number of samples, an algorithm's average performance over the samples provides a tight estimate of its expected performance on unseen instances.}
\end{observation}

\begin{observation}\label{obs:two_families}
\textnormal{Second, our experiments empirically illustrate that two algorithm families for the same computational problem may require markedly different training set sizes to avoid overfitting, a fact that has also been explored in prior theoretical research~\cite{Balcan17:Learning,Balcan18:General}. This shows that it is crucial to bound an algorithm family's intrinsic complexity to provide accurate guarantees, and in this paper we provide tools for doing so. Our experiments here are in the contexts of mechanism design for revenue maximization (Section~\ref{sec:experiments_revenue}) and mechanism design for voting (Section~\ref{sec:experiments_welfare}), both using real-world data. In both of these settings, we analyze two natural, practical mechanism families where one of the families is intrinsically simple and the other is intrinsically complex. When we use Theorem~\ref{thm:main} to select enough samples to ensure that overfitting does not occur for the simple class, we indeed empirically observe overfitting when optimizing over the complex class. The complex class requires more samples to avoid overfitting. In revenue-maximizing mechanism design, it was known that there exists a separation between the intrinsic complexities of the two corresponding mechanism families~\cite{Balcan18:General}, though this was not known in voting mechanism design.
In both of these settings, despite surface-level similarities between the complex and simple mechanism families, there is a pronounced difference in their intrinsic complexities, as illustrated by these experiments.}
\end{observation}

%% file: experiments_bio.tex
Changing the alignment parameter can alter the accuracy of the produced alignments. 
Figure~\ref{fig:singleAlignmentDecomp} shows the regions of the gap-open/gap-extension penalty plane 
divided into regions such that each region corresponds to a different computed alignment. 
The regions in the figure are produced using the \texttt{XPARAL} software of~\citet{Gusfield96:Parametric}, 
with using the BLOSUM62 amino acid replacement matrix, the scores in each region were computed using Robert Edgar's \texttt{qscore} package\footnote{\url{http://drive5.com/qscore/}}.
The alignment sequences are a single pairwise alignment from the data set described below. 

To illustrate Observation~\ref{obs:estimate},
we use the \texttt{IPA} tool~\citep{Kim07:inverse} to learn optimal parameter choices for a given set of example pairwise sequence alignments. 
We used 861 protein multiple sequence alignment benchmarks that had been previously been used in~\citet{DeBlasio18:Parameter},
which split these benchmarks into 12 cross-validation folds that evenly distributed the ``difficulty'' of an alignment 
(the accuracy of the alignment produced using aligner defaults parameter choice). 
All pairwise alignments were extracted from each multiple sequence alignment.
We then took randomized increasing sized subsets of the pairwise sequence alignments from each training set
and found the optimal parameter choices for affine gap costs and alphabet-dependent substitution costs. 
These parameters were then given to the \texttt{Opal} aligner~\citep[v3.1b, ][]{Wheeler07:Multiple} to realign the pairwise alignments in the associated test sets.

\begin{figure}
\centering
\includegraphics[width=\textwidth]{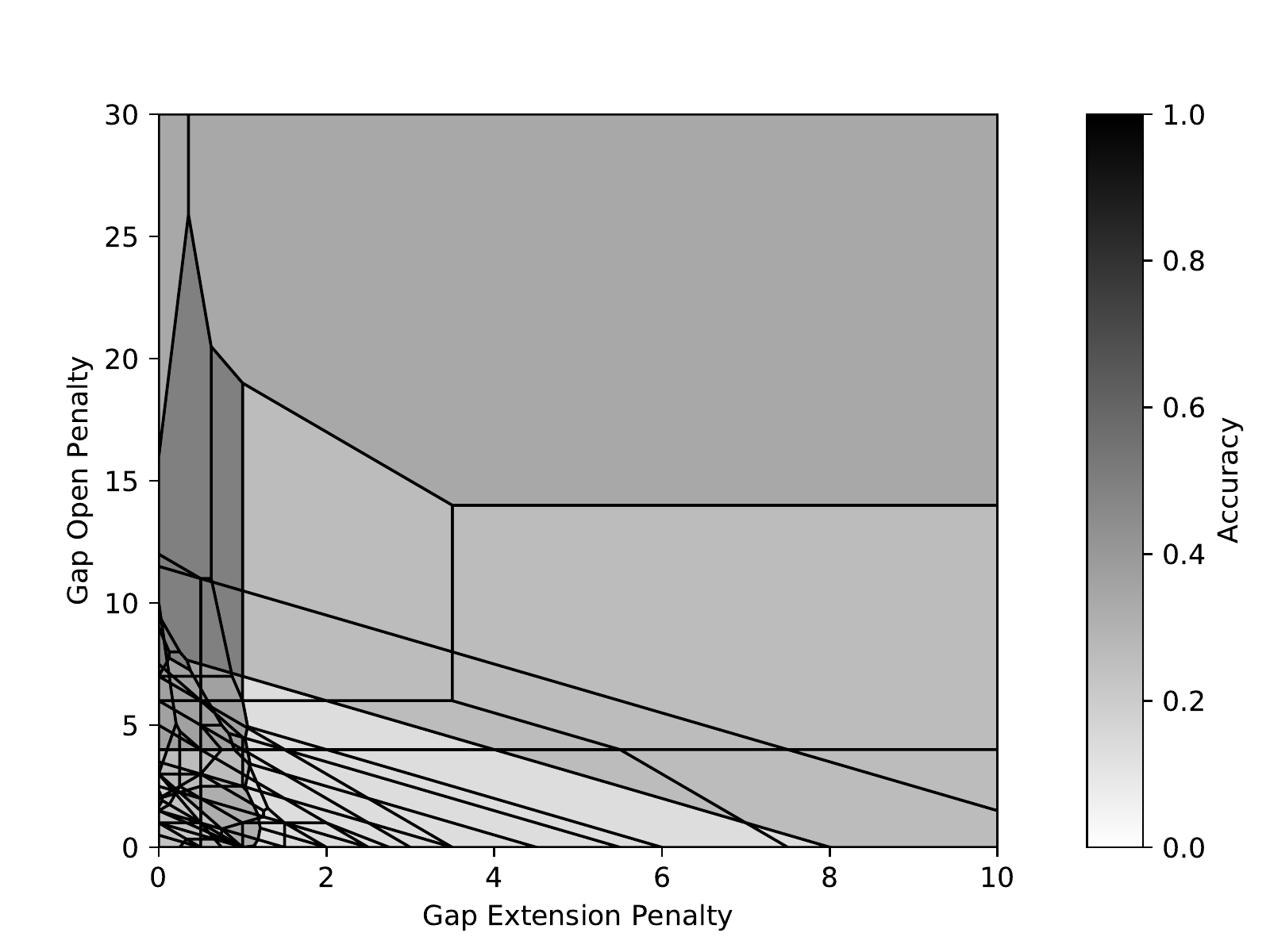}
\caption{Parameter space decomposition for a single example.}
\label{fig:singleAlignmentDecomp}
\end{figure}
\begin{figure}
\centering
\includegraphics{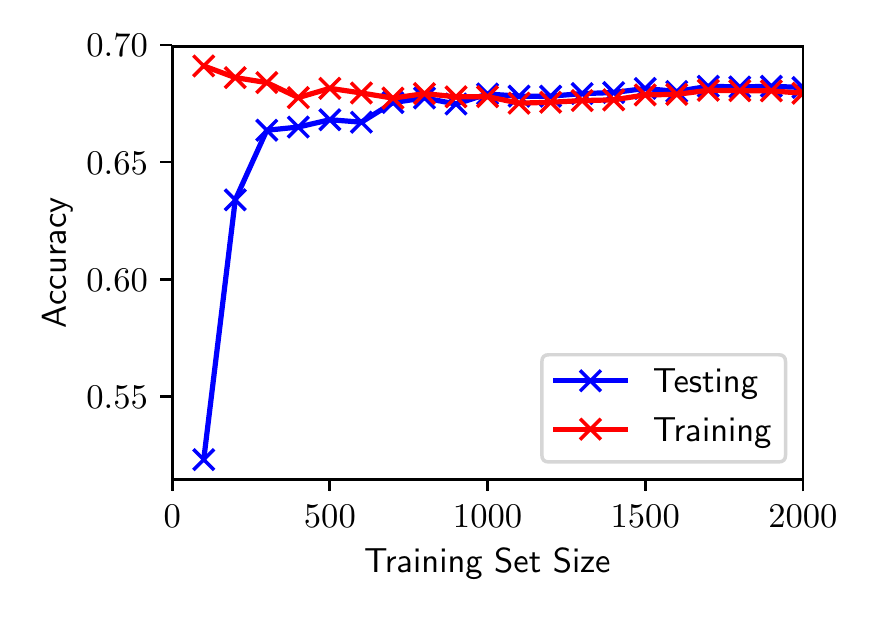}
\caption{Pairwise sequence alignment experiments showing the average accuracy on training and test examples using parameter choices optimized for various training set sizes. }
\label{fig:alignmentTrainTest}
\end{figure}

Figure~\ref{fig:alignmentTrainTest} shows the impact of increasing the number of training examples used to optimize parameter choices. 
As the number of training examples increases, 
the optimized parameter choice is 
less able to fit the training data exactly
and thus the training accuracy decreases.
For the same reason the parameter choices are more general and the test accuracy increases. 
The test and training accuracies are roughly equal when the training set size is close to $1000$ examples and remains equal for larger training sets.
The test accuracy is actually slightly higher and this is likely due to the training subset not representing the distribution of inputs as well as the full test set 
due to the randomization being on all of the alignments rather than across difficulty, as was done to create the cross-validation separations.

%% file: experiments_auctions.tex
In this section, we provide experiments analyzing the estimation error of several mechanism classes. We introduced the notion of estimation error in Section~\ref{sec:prelim}, but review it briefly here. For a class of mechanisms parameterized by vectors $\vec{\rho}$ (such as the class of neutral affine maximizers from Section~\ref{sec:econ}), let $u_{\vec{\rho}}(\vec{v}) \in [0,H]$ be the utility of the mechanism parameterized by $\vec{\rho}$ when the agents' values are defined by the vector $\vec{v}$. For example, this utility might equal its social welfare or revenue. Given a set of $N$ samples $\sample$, the mechanism class's estimation error equals the maximum difference, across all parameter vectors $\vec{\rho}$, between the average utility of the mechanism over the samples $\frac{1}{N} \sum_{\vec{v} \in \sample} u_{\vec{\rho}}(\vec{v}) $ and its expected utility $\E_{\vec{v} \sim \dist}\left[u_{\vec{\rho}}(\vec{v})\right]$. We know that with high probability, the estimation error is bounded by $\tilde{O}\left(H\sqrt{d/N}\right)$, where $d$ is the pseudo-dimension of the mechanism class. In other words, for all parameter vectors $\vec{\rho}$, \[\left|\frac{1}{N} \sum_{\vec{v} \in \sample} u_{\vec{\rho}}(\vec{v}) - \E_{\vec{v} \sim \dist}\left[u_{\vec{\rho}}(\vec{v})\right]\right| = \tilde{O}\left(H\sqrt{\frac{d}{N}}\right).\] Said another way, $\tilde{O}\left(\frac{H^2d}{\epsilon^2}\right)$ samples are sufficient to ensure that with high probability, for every parameter vector $\vec{\rho}$, the difference between the average and expected utility of the mechanism parameterized by $\vec{\rho}$ is at most $\epsilon$.

\subsubsection{Revenue maximization for selling mechanisms}\label{sec:experiments_revenue}
We provide experiments that illustrate our guarantees in the context of revenue maximization (which we introduced in Section~\ref{sec:revenue}). The type of mechanisms we analyze in our experiments are \emph{second-price auctions with reserves}, one of the most well-studied mechanism classes in economics~\citep{Vickrey61:Counterspeculation}. In this setting, there is one item for sale and a set of $n$ agents who are interested in buying that item. This mechanism family is defined by a parameter vector $\vec{\rho} \in \R^n_{\geq 0}$, where each entry $\rho[i]$ is the \emph{reserve price} for agent $i$. The agents report their values for the item to the auctioneer. If the highest bidder---say, agent $i^*$---reports a bid that is larger than her reserve $\rho\left[i^*\right]$, she wins the item and pays the maximum of the second-highest bid and her reserve $\rho\left[i^*\right]$. Otherwise, no agent wins the item. The second-price auction (SPA) is called \emph{anonymous} if every agent has the same reserve price: $\rho[1] = \rho[2] = \cdots = \rho[n]$. Otherwise, the SPA is called \emph{non-anonymous}. Like neutral affine maximizers, SPAs are incentive compatible, so we assume the agents' bids equal their true values. We refer to the class of non-anonymous SPAs as $\cA_N$ and the class of anonymous SPAs as $\cA_A$.

The class of non-anonymous SPAs $\cA_N$ is more complex than the class of anonymous SPAs $\cA_A$, and thus the sample complexity of the former is higher than the latter. However, non-anonymous SPAs can provide much higher revenue than anonymous SPAs.
We illustrate this tradeoff in our experiments, which helps exemplify Observation~\ref{obs:two_families}.

In a bit more detail, the pseudo-dimension of  the class of non-anonymous SPAs $\cA_N$ is $\tilde{O}(n)$, an upper bound proved in prior research~\citep{Balcan18:General,Morgenstern16:Learning} which can be recovered using the techniques in this paper, as we summarize in Section~\ref{sec:revenue}. What's more,~\citet{Balcan18:General} proved a pseudo-dimension lower bound of $\Omega(n)$ for this class.  Meanwhile, the pseudo-dimension of the class of anonymous SPAs $\cA_A$ is $O(1)$. This upper bound was proved in prior research by \citet{Morgenstern16:Learning} and \citet{Balcan18:General}, and it follows from the general theory presented in this paper, as we formalize below in Lemma~\ref{lem:SPA}.
In our experiments, we exhibit a distribution over agents' values under which the following properties hold:
\begin{enumerate}
	\item The true estimation error of the set of non-anonymous SPAs $\cA_N$ is larger than our upper bound on the worst-case estimation error of the set of anonymous SPAs $\cA_A$.
	\item The expected revenue of the optimal non-anonymous SPA is significantly larger than the expected revenue of the optimal anonymous SPA: the former is 0.38 and the latter is 0.57.
\end{enumerate}
These two points imply that even though it may be tempting to optimize over the class of non-anonymous SPAs $\cA_N$ in order to obtain higher expected revenue, the training set must be large enough to ensure the empirically-optimal mechanism has nearly optimal expected revenue.

The distribution we analyze is over the Jester Collaborative Filtering Dataset~\citep{Goldberg01:Eigentaste}, which consists of ratings from tens of thousands of users of one hundred jokes---in this example, the jokes could be proxies for comedians, and the agents could be venues who are bidding for the opportunity to host the comedians. Since the auctions we analyze in these experiments are for a single item, we run experiments using agents' values for a single joke, which we select to have a roughly equal number of agents with high values as with low values (we describe this selection in greater detail below). Figure~\ref{fig:auction_experiment} illustrates the outcome of our experiments.
\begin{figure}[t]
	\centering
	\includegraphics[scale=0.7]{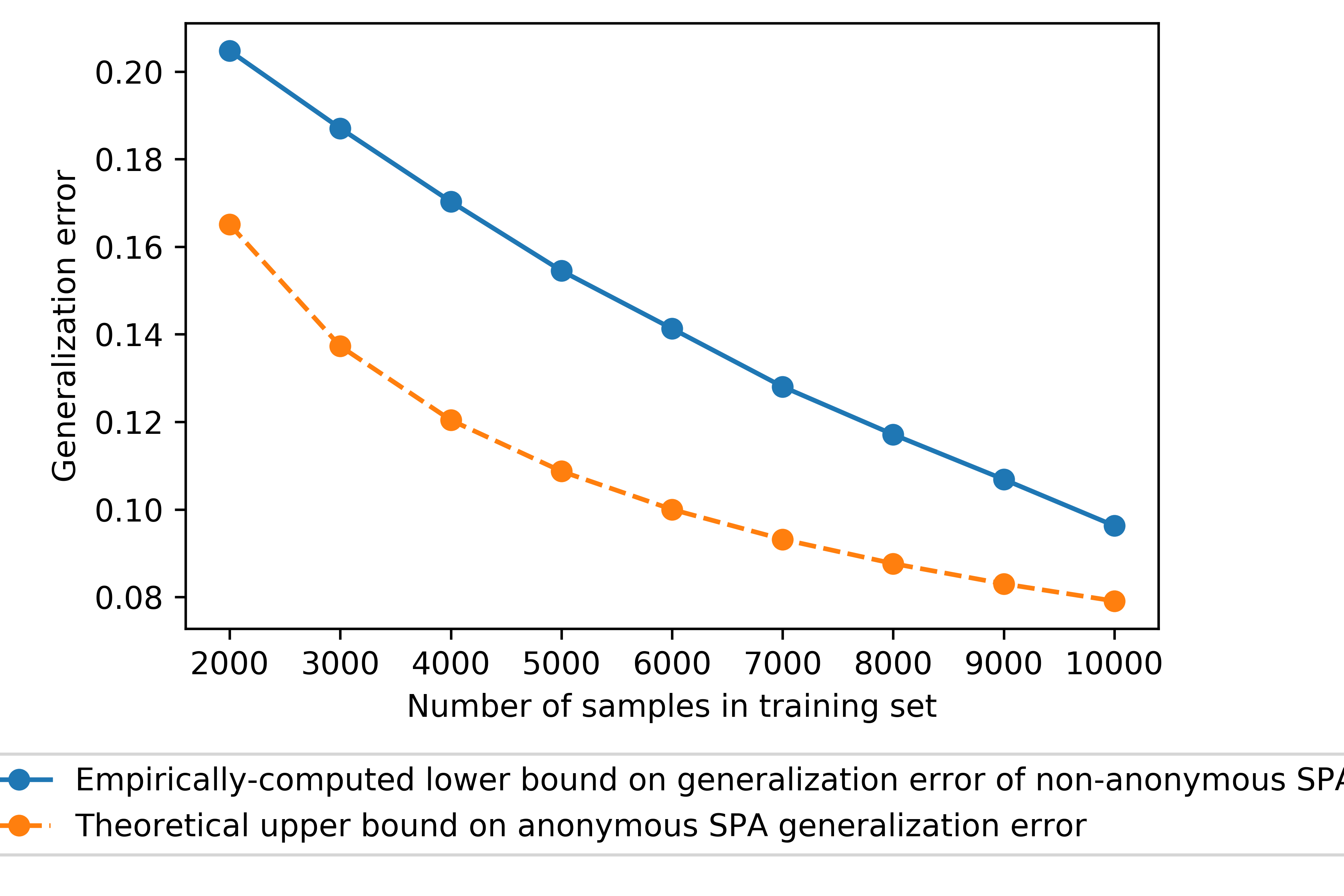}
	\caption{Revenue maximization experiments.
		We vary the size of the training set, $N$, along the $x$-axis.
		The orange dashed line is our theoretical upper bound on the estimation error of the class of anonymous SPAs $\cA_A$, $\sqrt{\frac{4}{N}\ln(eN)} + \sqrt{\frac{1}{2N}\ln 100}$, which follows from our pseudo-dimension bound (Lemma~\ref{lem:SPA}). The blue solid line lower bounds the true estimation error of the class of non-anonymous SPAs $\cA_N$ over the Jester dataset. For several choices of $N \in [2000,10000]$, we compute this lower bound by drawing a set of $N$ training instances, finding a mechanism in $\cA_N$ with high average revenue over the training set, and calculating the mechanism's estimation error (the difference between its average revenue and expected revenue). For scale, estimation error is a quantity in the range $[0,1]$.}
	\label{fig:auction_experiment}
\end{figure}
The orange dashed line is our upper bound on the estimation error of the class of anonymous SPAs $\cA_A$, which equals $\sqrt{\frac{4}{N}\ln(eN)} + \sqrt{\frac{1}{2N}\ln \frac{1}{\delta}}$, with $\delta = 0.01$. This upper bound has been presented in prior research~\citep{Morgenstern16:Learning,Balcan18:General}, and we recover it using the general theory presented in this paper, as we demonstrate below in Lemma~\ref{lem:SPA}. The blue solid line is the true estimation error\footnote{Sample complexity and estimation error bounds for SPAs have been studied in prior research~\citep{Morgenstern16:Learning, Devanur16:Sample, Balcan18:General}, and our guarantees match the bounds provided in that literature.} of the class of non-anonymous SPAs $\cA_N$ over the Jester dataset, which we calculate via the following experiment. For a selection of training set sizes $N \in [2000, 12000]$, we draw $N$ sample agent values, with the number of agents equal to 10,612 (as we explain below). We then find a vector of non-anonymous reserves with maximum average revenue over the samples but low expected revenue on the true distribution, as we describe in greater detail below. We plot this difference between average and expected revenue, averaged over 100 trials of the same procedure. As these experiments illustrate, there is a tradeoff between the sample complexity and revenue guarantees of these two classes, and it is crucial to calculate a class's intrinsic complexity to provide accurate guarantees. We now explain the details of these experiments.

\paragraph{Distribution over agents' values.} We use the Jester Collaborative Filtering Dataset~\citep{Goldberg01:Eigentaste}, which consists of ratings from 24,983 users of 100 jokes. The users' ratings are in the range $[-10, 10]$, so we normalize their values to lie in the interval $[0,1]$. We select a joke which has at least 5000 (normalized) ratings in the interval $\left[\frac{1}{4}, \frac{1}{2}\right]$ and at least 5000 (normalized) ratings in the interval $\left[\frac{3}{4}, 1\right]$ (specifically, Joke \#22). There is a total of 5334 ratings of the first type and 5278 ratings of the second type. Let $W = \left\{w_1, \dots, w_{10,612}\right\}$ be all 10,612 values. For every $i \in [10612]$, we define the following valuation vector $\vec{v}^{(i)}$: agent $i$'s value $v_i^{(i)}$ equals $w_i$ and for all other agents $j \not= i$, $v_j^{(i)} = 0$. We define the distribution $\dist$ over agents' values to be the uniform distribution over $\vec{v}^{(1)}, \dots, \vec{v}^{(10,612)}$.

\paragraph{Empirically-optimal non-anonymous reserve vectors with poor generalization.} Given a set of samples $\sample \sim \dist^N$, let $p\left(\vec{v}^{(1)}\right), \dots, p\left(\vec{v}^{(10,612)}\right)$ define the empirical distribution over $\sample$ (that is, $p\left(\vec{v}^{(i)}\right)$ equals the number of times $\vec{v}^{(i)}$ appears in $\sample$ divided by $N$). Then for any non-anonymous reserve vector $\vec{\rho} \in \R^n$, the average revenue over the samples is \begin{equation}\frac{1}{N} \sum_{i = 1}^{10,612} p\left(\vec{v}^{(i)}\right) \rho[i] \textbf{1}_{\left\{w_i \geq \rho[i]\right\}}.\label{eq:empirical_revenue}\end{equation} This is because for every vector $\vec{v}^{(i)}$, the only agent with a non-zero value is agent $i$, whose value is $w_i$. Therefore, the highest bidder is agent $i$, who wins the item if $w_i \geq \rho[i]$, and pays reserve $\rho[i]$, which is the maximum of $\rho[i]$ and the second-highest bid. As is clear from Equation~\eqref{eq:empirical_revenue}, if $\vec{v}^{(i)} \in \sample$, an empirically-optimal reserve vector $\hat{\vec{\rho}}$ (which maximizes average revenue over the samples) will set $\hat{\rho}_i = w_i$, and if $\vec{v}^{(i)} \not\in \sample$, $\hat{\rho}_i$ can be set arbitrarily, because it has no effect on the average revenue over the samples. In our experiments, for all $\vec{v}^{(i)} \not\in \sample$, we set $\hat{\rho}_i = \frac{3}{4}$. The intuition is that those agents $i$ with $\vec{v}^{(i)} \not\in \sample$ and $w_i \in \left[\frac{1}{4}, \frac{1}{2}\right]$ will not win the item, and therefore will drag down expected revenue.

In Figure~\ref{fig:auction_experiment}, the orange dashed line is the difference between the average value of $\hat{\vec{\rho}}$ over $\sample$ and its expected value, which we calculate via the following experiment. For a selection of training set sizes $N \in [2000, 12000]$, we draw the set $\sample \sim \dist^N$. We then construct the reserve vector $\hat{\vec{\rho}}$ as described above. We plot the difference between the average value of $\hat{\vec{\rho}}$ over $\sample$ and its expected value, averaged over 100 trials of the same procedure.

\paragraph{Estimation error of anonymous SPAs.}
We now use our general theory to bound the estimation error of the class of anonymous SPAs $\cA_A$ so that we can plot the blue solid line in Figure~\ref{fig:auction_experiment}.
This pseudo-dimension upper bound has been presented in prior research~\citep{Morgenstern16:Learning,Balcan18:General}, but here we show that it can be recovered using this paper's techniques.
 We use Lemma~\ref{lem:oscillate} to obtain a slightly tighter pseudo-dimension bound (up to constant factors) than that of Corollary~\ref{cor:delineable}.

Given a reserve price $\rho \geq 0$ and valuation vector $\vec{v} \in \R^n$, let $u_{\rho}(\vec{v}) \in [0,1]$ be the revenue of the anonymous second-price auction with a reserve price of $\rho$ when the bids equal $\vec{v}$. Using Lemma~\ref{lem:oscillate}, we prove that the pseudo-dimension of this class of revenue functions is at most 2.

\begin{lemma}\label{lem:SPA}
The pseudo-dimension of the class $\cA_A = \left\{u_{\rho} : \rho \geq 0\right\}$ is at most 2.
\end{lemma}

\begin{proof}
	Given a vector $\vec{v}$ of bids, let $v_{(1)}$ be the highest bid in $\vec{v}$ and let $v_{(2)}$ be the second-highest bid. Under the anonymous SPA, the highest bidder wins the item if $v_{(1)} \geq \rho$ and pays $\max\left\{v_{(2)}, \rho\right\}$. Therefore $u_{\rho}(\vec{v}) = \max\left\{v_{(2)}, \rho\right\}\textbf{1}_{\left\{v_{(1)} \geq \rho\right\}}.$ By definition of the dual function, \[u_{\vec{v}}^*\left(u_{\rho}\right) = \max\left\{v_{(2)}, \rho\right\}\textbf{1}_{\left\{v_{(1)} \geq \rho\right\}}.\] The dual function is thus an increasing function of $\rho$ in the interval $\left[0, v_{(1)}\right]$ and is equal to zero in the interval $\left(v_{(1)}, \infty\right)$. Therefore, the function has at most two oscillations (as in Definition~\ref{def:oscillate}). By Lemma~\ref{lem:oscillate}, the pseudo-dimension of $\cA_A$ is at most the largest integer $D$ such that $2^D \leq 2D + 1$, which equals 2. Therefore, the theorem statement holds.
	\end{proof}

By Equation~\eqref{eq:pollard}, with probability $1-\delta$ over the draw of $\sample \sim \dist^N$, for any reserve $\rho \geq 0$, \begin{equation}\left|\frac{1}{N}\sum_{\vec{v} \in \sample}
u_{\rho}(\vec{v}) - \E_{\vec{v} \sim \dist}\left[u_{\rho}(\vec{v})\right]\right| \leq \sqrt{\frac{4}{N}\ln(eN)} + \sqrt{\frac{1}{2N}\ln \frac{1}{\delta}}.\label{eq:SPA_generalization}\end{equation} In Figure~\ref{fig:auction_experiment}, the blue solid line equals the right-hand-side of Equation~\eqref{eq:SPA_generalization} with $\delta = 0.01$ as a function of $N$.

\paragraph{Discussion.} In summary, this section exemplifies a distribution over agents' values where:
\begin{enumerate}
	\item The true estimation error of the set of non-anonymous SPAs (the orange dashed line in Figure~\ref{fig:auction_experiment}) is larger than our upper bound on the worst-case estimation error of the set of anonymous SPAs (the blue solid line in Figure~\ref{fig:auction_experiment}).
	\item The expected revenue of the optimal non-anonymous SPA is significantly larger than the expected revenue of the optimal anonymous SPA: the former is 0.38 and the latter is 0.57.
\end{enumerate}
Therefore, there is a tradeoff between the sample complexity and revenue guarantees of these two classes.

\subsubsection{Social welfare maximization for voting mechanisms}\label{sec:experiments_welfare}

In this section, we provide similar experiments as those in the previous section, but in the context of neutral affine maximizers (NAMs), which we defined in Section~\ref{sec:econ}. We present a simple subset of neutral affine maximizers (NAMs) with a small estimation error upper bound. We then experimentally demonstrate that the true estimation error of the class of NAMs is larger than this simple subset's estimation error. Therefore, it is crucial to calculate a class's intrinsic complexity in order to provide accurate guarantees. These experiments further illustrate Observation~\ref{obs:two_families}.

Our simple set of mechanisms is defined as follows. One agent is selected to be the sink agent (a sink agent $i$ has the weight $\rho[i] = 0$), and every other agent's weight is set to 1. In other words, this class is defined by the set of all vectors $\vec{\rho} \in \{0,1\}^n$ where exactly one component of $\vec{\rho}$ is equal to zero. We use the notation $\cA_0$ to denote this simple class, $\cA_{NAM}$ to denote the set of all NAMs and $u_{\vec{\rho}}(\vec{v})$ to denote the social welfare of the NAM parameterized by $\vec{\rho}$ given the valuation vector $\vec{v}$.

Since there are $n$ NAMs in $\cA_0$, a Hoeffding and union bound tells us that for any distribution $\dist$, with probability $0.99$ over the draw of $N$ valuation vectors $\sample \sim \dist^N$, for all $n$ parameter vectors $\vec{\rho}$, \begin{equation}\left|\E_{\vec{v} \sim \dist}\left[u_{\vec{\rho}}(\vec{v})\right] - \frac{1}{N} \sum_{\vec{v} \in \sample} u_{\vec{\rho}}(\vec{v})\right|\leq \sqrt{\frac{\ln (200 n)}{2N}}.\label{eq:union}\end{equation}
\begin{figure}[t]
	\centering
	\includegraphics[width=0.7\textwidth]{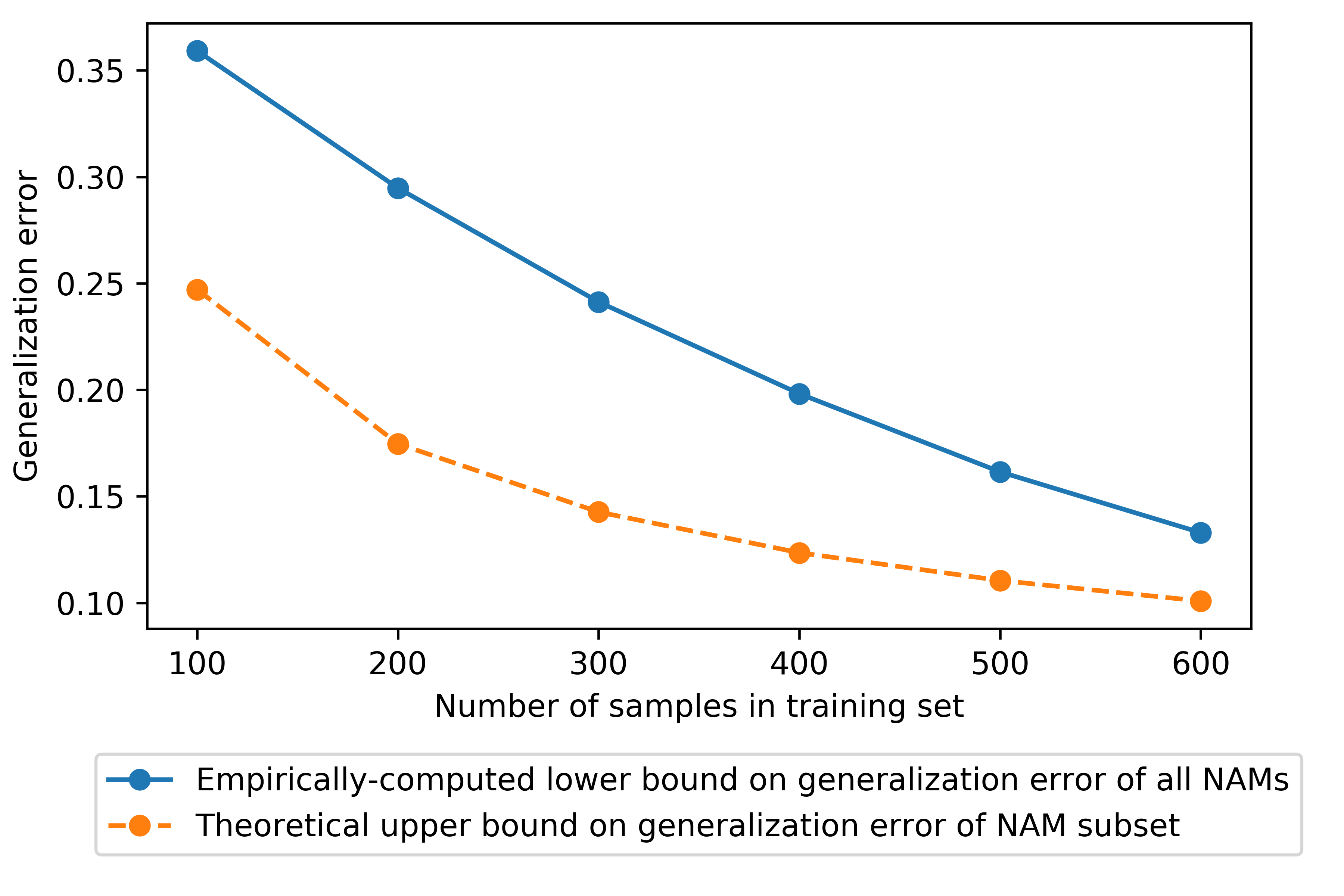}
	\caption{Neutral affine maximizer experiments. 
We vary the size of the training set, $N$, along the $x$-axis.
The orange dashed line is our upper bound on the estimation error of the simple subset of NAMs $\cA_0$, $\sqrt{\frac{\ln (200 n)}{2N}}$ (Equation~\eqref{eq:union}). The blue solid line lower bounds the true estimation error of the entire class of NAMs $\cA_{NAM}$ over the Jester dataset. For several choices of $N \in [100,600]$, we compute this lower bound by drawing a set of $N$ training instances, finding a mechanism in $\cA_{NAM}$ with high average social welfare over the training set, and calculating the mechanism's estimation error (the difference between its average social welfare and expected social welfare). For scale, estimation error is a quantity in the range $[0,1]$.}
	\label{fig:NAM}
\end{figure}
This is the orange dashed line in Figure~\ref{fig:NAM}.
Meanwhile, as we prove in Section~\ref{sec:econ}, the pseudo-dimension of the class of all NAMs $\cA_{NAM}$ is $\tilde{\Theta}(n)$, so it is a more complex set of mechanisms than $\cA_0$.
 To experimentally compute the lower bound on the true estimation error of the class of all NAMs $\cA_{NAM}$, we identify a distribution such that the class has high estimation error, as we describe below.

\paragraph{Distribution.}
As in the previous section, we use the Jester Collaborative Filtering Dataset~\citep{Goldberg01:Eigentaste} to define our distribution.
This dataset consists of ratings from 24,983 users of 100 jokes. The users' ratings are in the range $[-10, 10]$, so we normalize their ratings to lie in the interval $[-0.5,0.5]$. 
 We begin by selecting two jokes (jokes \#7 and \#15) such that---at a high level---a large number of agents either like the first joke a lot and do not like the second joke, or do not like the first joke and like the second joke a medium amount. We explain the intuition behind this choice below.
Specifically, we split the agents into two groups: in the first group, the agents rated joke 1 at least 0.35 and rated joke 2 at most 0, and in the second group, the agents rated joke 1 at most 0 and rated joke 2 between 0 and 0.15.
We call the set of ratings corresponding to the first group $A_1 \subseteq \R^2$ and those corresponding to the second group $A_2 \subseteq \R^2$. The set $A_1$ has size 870 and $A_2$ has size 1677.

We use $A_1$ and $A_2$ to define a distribution $\dist$ over the valuations of $n = 1000$ agents for two jokes. The support of $\dist$ consists of 500 valuation vectors $\vec{v}^{(1)}, \dots, \vec{v}^{(500)} \in \R^{2 \times 1000}$.
For $i\in [500]$, $\vec{v}^{(i)}$ is defined as follows. The values of agent $i$ for the two jokes, $\left(v_i^{(i)} (1), v_i^{(i)} (2)\right)$, are chosen uniformly at random from $A_1$ and the values of agent $500 + i$ are chosen uniformly at random from $A_2$. Every other agent $i$ has a value of $v_i^{(i)}(1) = v_i^{(i)}(2) = 0$.

\paragraph{Parameter vector with poor estimation error.} Given a set of samples $\sample \subseteq \left\{\vec{v}^{(1)}, \dots, \vec{v}^{(500)}\right\}$, we define a parameter vector $\vec{\rho} \in \{0,1\}^{1000}$ with high estimation error as follows: for all $i \in [500]$, 
\begin{equation}\rho[i] = \begin{cases} 1 &\text{if } \vec{v}^{(i)} \in \sample\\
0 &\text{otherwise}\end{cases} \qquad \text{ and } \rho[500+i] = \begin{cases} 0 &\text{if } \vec{v}^{(i)} \not\in \sample\\
1 &\text{otherwise.}\end{cases}\label{eq:NAM_bad_param}\end{equation} Intuitively, this parameter vector\footnote{Although this parameter vector has high average social welfare over the samples, it may set multiple agents to be sink agents, which may be wasteful. We leave the problem of finding a parameter vector with high estimation error and only a single sink agent to future research.} has high estimation error for the following reason. Suppose $\vec{v}^{(i)} \in \sample$. The only agents with nonzero values in $\vec{v}^{(i)}$ are agent $i$ and agent $500 + i$. Since $\vec{v}^{(i)} \in \sample$, $\rho[i] = 1$ and $\rho[500+i] = 0$. Therefore, agent $i$'s favorite joke is selected. Since agent $i$'s values are from the set $A_1$, they have a value of at least 0.35 for joke 1 and a value of at most 0 for joke 2. Therefore, joke 1 will be the selected joke. Meanwhile, by the same reasoning for every $\vec{v}^{(i)} \not\in \sample$, if we run the NAM defined by $\vec{\rho}$, joke 2 will be the selected joke. In expectation over $\dist$, joke 1 has a significantly higher social welfare than joke 2. Therefore, the NAM defined by $\vec{\rho}$ will have a high average social welfare over the samples in $\sample$ but a low expected social welfare, which means it will have high estimation error. We illustrate this intuition in our experiments. 

\paragraph{Experimental procedure.} We repeat the following experiment 100 times. For various choices of $N \in [600]$, we draw a set of samples $\sample \sim \dist^N$, compute the parameter vector $\vec{\rho}$ defined by Equation~\eqref{eq:NAM_bad_param}, and compute the difference between the average social welfare of $\vec{\rho}$ over $\sample$ and its expected social welfare. We plot the difference averaged over all 100 runs.

\paragraph{Discussion.}
These experiments demonstrate that although the simple and complex NAM families $\cA_{NAM}$ and $\cA_0$ are artificially similar (they are both defined by the $m$ agent weights), the complex family $\cA_{NAM}$ requires far more samples to avoid overfitting than the simple family $\cA_0$. This illustrates the importance of using our pseudo-dimension bounds to provide accurate guarantees.

%% file: conclusion.tex
We provided a general sample complexity theorem for learning high-performing algorithm configurations.
Our bound applies whenever a parameterized algorithm's performance is a piecewise-structured function of its parameters: for any fixed problem instance, boundary functions partition the parameters into regions where performance is a well-structured function. We proved this guarantee by exploiting intricate connections between primal function classes (measuring the algorithm's performance as a function of its input) and dual function classes (measuring the algorithm's performance on a fixed input as a function of its parameters). We demonstrated that many parameterized algorithms exhibit this structure and thus our main theorem implies sample complexity guarantees for a broad array of algorithms and application domains. 

	A great direction for future research is to build on these ideas for the sake of learning a \emph{portfolio} of configurations, rather than a single high-performing configuration. At runtime, machine learning is used to determine which configuration in the portfolio to employ for the given input. \citet{Gupta17:PAC} and \citet{Balcan21:Generalization} have provided initial results in this direction, but a general theory of portfolio-based algorithm configuration is yet to be developed.

%% file: appendix_helpful.tex
\begin{lemma}[\citet{Shalev14:Understanding}]\label{lem:log_ineq}
Let $a \geq 1$ and $b > 0$. If $y < a\ln y + b$, then $y < 4a \ln (2a) + 2b$.
\end{lemma}

The following is a corollary of Rolle's theorem that we use in the proof of Lemma~\ref{lem:TAD_decomp}.

\begin{lemma}[\citet{Tossavainen06:On}]\label{lem:roots}
	Let $h$ be a polynomial-exponential sum of the form $h(x) = \sum_{i = 1}^t a_i b_i^x$, where $b_i > 0$ and $a_i \in \R$. The number of roots of $h$ is upper bounded by $t$.
\end{lemma}

\begin{cor}\label{cor:roots}
	Let $h$ be a polynomial-exponential sum of the form \[h(x) = \sum_{i = 1}^t \frac{a_i}{b_i^{x}},\] where $b_i > 0$ and $a_i \in \R$. The number of roots of $h$ is upper bounded by $t$.
\end{cor}

\begin{proof}
	Note that $\sum_{i = 1}^t \frac{a_i}{b_i^{x}} = 0$ if and only if \[\left(\prod_{j=1}^n b_i^x\right)\sum_{i = 1}^t \frac{a_i}{b_i^{x}} = \sum_{i = 1}^n a_i \left(\prod_{j \not= i} b_i\right)^x = 0.\] Therefore, the corollary follows from Lemma~\ref{lem:roots}.
\end{proof}

%% file: appendix_general.tex
\linear*

\begin{proof}
	First, we prove that the VC-dimension of the dual class $\cG^*$ is at most $d+1$. The dual class $\cG^*$ consists of functions $g_{u_{\vec{\rho}}}^*$ for all $\vec{\rho} \in \cP$ where $g_{u_{\vec{\rho}}}^*\left(g_{\vec{a}, \theta}\right) = \ind{\vec{a} \cdot \vec{\rho} \leq \theta}$. Let $\hat{\cG} = \left\{\hat{g}_{\vec{\rho}} : \R^{d+1} \to \{0,1\}\right\}$ be the class of halfspace thresholds $\hat{g}_{\vec{\rho}} : (\vec{a}, \theta) \mapsto \ind{\vec{a} \cdot \vec{\rho} \leq \theta}$. It is well-known that $\VC\left(\hat{\cG}\right) \leq d+1$, which we prove means that $\VC\left(\cG^*\right) \leq d+1$. For a contradiction, suppose $\cG^*$ can shatter $d+2$ functions $g_{\vec{a}_1, \theta_1}, \dots, g_{\vec{a}_{d+2}, \theta_{d+2}} \in \cG$. Then for every subset $T \subseteq [d+2]$, there exists a parameter vector $\vec{\rho}_T$ such that $\vec{a}_i \cdot \vec{\rho}_T \leq \theta_i$ if and only if $i \in T$. This means that $\hat{\cG}$ can shatter the tuples $\left(\vec{a}_1, \theta_1\right), \dots, \left(\vec{a}_{d+2}, \theta_{d+2}\right)$ as well, which contradicts the fact that $\VC\left(\hat{\cG}\right) \leq d+1$. Therefore, $\VC\left(\cG^*\right) \leq d+1$.
	
	By a similar argument, we prove that the pseudo-dimension of the dual class $\cF^*$ is at most $d+1$.  The dual class $\cF^*$ consists of functions $f_{u_{\vec{\rho}}}^*$ for all $\vec{\rho} \in \cP$ where $f_{u_{\vec{\rho}}}^*\left(f_{\vec{a}, \theta}\right) = \vec{a} \cdot \vec{\rho} + \theta$. Let $\hat{\cF} = \left\{\hat{f}_{\vec{\rho}} : \R^{d+1} \to \R\right\}$ be the class of linear functions $\hat{f}_{\vec{\rho}} : (\vec{a}, \theta) \mapsto \vec{a} \cdot \vec{\rho} + \theta$. It is well-known that $\pdim\left(\hat{\cF}\right) \leq d+1$, which we prove means that $\pdim\left(\cF^*\right) \leq d+1$. For a contradiction, suppose $\cF^*$ can shatter $d+2$ functions $f_{\vec{a}_1, \theta_1}, \dots, f_{\vec{a}_{d+2}, \theta_{d+2}} \in \cF$. Then there exist witnesses $z_1, \dots, z_{d+2}$ such that for every subset $T \subseteq [d+2]$, there exists a parameter vector $\vec{\rho}_T$ such that $\vec{a}_i \cdot \vec{\rho}_T + \theta_i \leq z_i$ if and only if $i \in T$. This means that $\hat{\cF}$ can shatter the tuples $\left(\vec{a}_1, \theta_1\right), \dots, \left(\vec{a}_{d+2}, \theta_{d+2}\right)$ as well, which contradicts the fact that $\pdim\left(\hat{\cF}\right) \leq d+1$. Therefore, $\pdim\left(\cF^*\right) \leq d+1$.
	
	The lemma statement now follows from Theorem~\ref{thm:main}.
\end{proof}

%% file: appendix_sequence.tex
\begin{lemma}\label{lem:count}
Fix a pair of sequences $S_1, S_2 \in \Sigma^n$. There are at most $2^n n^{2n+1}$ alignments of $S_1$ and $S_2$.
\end{lemma}

\begin{proof}
For any alignment $\left(\tau_1, \tau_2\right)$, we know that $\left|\tau_1\right| = \left|\tau_2\right|$ and for all $i \in [|\tau_1|]$, if $\tau_1[i] = -$, then $\tau_2[i] \not= -$ and vice versa. This means that $\tau_1$ and $\tau_2$ have the same number of gaps. To prove the upper bound, we count the number of alignments $\left(\tau_1, \tau_2\right)$ where $\tau_1$ and $\tau_2$ each have exactly $i$ gaps. There are ${ n+i \choose i}$ choices for the sequence $\tau_1$. Given a sequence $\tau_1$, we can only pair a gap in $\tau_2$ with a non-gap in $\tau_1$. Since there are $i$ gaps in $\tau_2$ and $n$ non-gaps in $\tau_1$, there are ${n \choose i}$ choices for the sequence $\tau_2$ once $\tau_1$ is fixed. This means that there are ${n+i \choose i} {n \choose i} \leq 2^n n^{2n}$ alignments $\left(\tau_1, \tau_2\right)$ where $\tau_1$ and $\tau_2$ each have exactly $i$ gaps. Summing over $i \in [n]$, the total number of alignments is at most $2^n n^{2n+1}$.
\end{proof}

\seqLB*
	\begin{proof}
	To prove this theorem, we identify:
	\begin{enumerate}
		\item An alphabet $\Sigma$,
		\item A set of $N = \Theta(\log n)$ sequence pairs $\left(S_1^{(1)}, S_2^{(1)}\right), \dots, \left(S_1^{(N)}, S_2^{(N)}\right) \in \cup_{i = 1}^n\Sigma^i \times \Sigma^i$,
		\item A ground-truth alignment $L_*^{(i)}$ for each sequence pair $\left(S_1^{(i)}, S_2^{(i)}\right)$, and
		\item A set of $N$ witnesses $z_1, \dots, z_N \in \R$ such that for any subset $T \subseteq[N]$, there exists an indel penalty parameter $\rho[T]$ such that if $i \in [T]$, then $u_{0, \rho[T], 0}\left(S_1^{(i)}, S_2^{(i)}\right) < z_i$ and if $i \not\in [T]$, then $u_{0, \rho[T], 0}\left(S_1^{(i)}, S_2^{(i)}\right) \geq z_i$.
	\end{enumerate}
	
 We now describe each of these four elements in turn.
	
	\paragraph{\textbf{The alphabet $\Sigma$.}} Let $k = 2^{\left\lfloor \log  \sqrt{n/2}\right\rfloor} - 1 = \Theta(\sqrt{n})$. The alphabet $\Sigma$ consists of $4k$ characters\footnote{To simplify the proof, we use this alphabet of size $4k$, but we believe it is possible to adapt this proof to handle the case where there are only 4 characters in the alphabet.} we denote as $\left\{\texttt{a}_i, \texttt{b}_i, \texttt{c}_i, \texttt{d}_i\right\}_{i =1}^k$.
	
	\paragraph{\textbf{The set of $N = \Theta(\log n)$ sequence pairs.}}
	These $N$ sequence pairs are defined by a set of $k$ subsequence pairs $\left(t_1^{(1)}, t_2^{(1)}\right), \dots, \left(t_1^{(k)}, t_2^{(k)}\right) \in \Sigma^* \times \Sigma^*$. Each pair $\left(t_1^{(i)}, t_2^{(i)}\right)$ is defined as follows:
	\begin{itemize}
		\item The subsequence $t_1^{(i)}$ begins with $i$ $\texttt{a}_i$s followed by $\texttt{b}_i\texttt{d}_i$. For example, $t_1^{(3)} = \texttt{a}_3\texttt{a}_3\texttt{a}_3\texttt{b}_3\texttt{d}_3$.
		\item The subsequence $t_2^{(i)}$ begins with 1 $\texttt{b}_i$, followed by $i$ $\texttt{c}_i$s, followed by 1 $\texttt{d}_i$. For example,  $t_2^{(3)} = \texttt{b}_3\texttt{c}_3\texttt{c}_3\texttt{c}_3\texttt{d}_3$.
	\end{itemize}
	Therefore, $t_1^{(i)}$ and $t_2^{(i)}$ are both of length $i + 2$.
	
	We use these subsequence pairs to define a set of $N = \log(k + 1) = \Theta(\log n)$ sequence pairs. The first sequence pair, $\left(S_1^{(1)}, S_2^{(1)}\right)$ is defined as follows: $S_1^{(1)}$ is the concatenation of all subsequences $t_1^{(1)}, \dots, t_1^{(k)}$ and $S_2^{(1)}$ is the concatenation of all subsequences $t_2^{(1)}, \dots, t_2^{(k)}$:
	\[S_1^{(1)} = t_1^{(1)} t_1^{(2)} t_1^{(3)}  \cdots t_1^{(k)} \quad \text{ and } \quad S_2^{(1)} = t_2^{(1)} t_2^{(2)} t_2^{(3)}  \cdots t_2^{(k)}.\label{eq:first_seq}\] Next, $S_1^{(2)}$ and $S_2^{(2)}$ are the concatenation of every $2^{nd}$ subsequence: \[S_1^{(2)} = t_1^{(2)} t_1^{(4)} t_1^{(6)}  \cdots t_1^{(k-1)} \quad \text{ and } \quad S_2^{(2)} = t_2^{(2)} t_2^{(4)} t_2^{(6)}  \cdots t_2^{(k-1)}.\] Similarly, $S_1^{(3)}$ and $S_2^{(3)}$ are the concatenation of every $4^{th}$ subsequence: \[S_1^{(3)} = t_1^{(4)} t_1^{(8)} t_1^{(12)}  \cdots t_1^{(k-3)} \quad \text{ and } \quad S_2^{(3)} = t_2^{(4)} t_2^{(8)} t_2^{(12)}  \cdots t_2^{(k-3)}.\] Generally speaking, $S_1^{(i)}$ and $S_2^{(i)}$ are the concatenation of every $\left(2^{i-1}\right)^{th}$ subsequence: \[S_1^{(i)} = t_1^{(2^{i-1})} t_1^{(2 \cdot 2^{i-1})} t_1^{(3\cdot 2^{i-1})}  \cdots t_1^{(k+1-2^{i-1})} \quad \text{ and } \quad S_2^{(i)} = t_2^{(2^{i-1})} t_2^{(2 \cdot 2^{i-1})} t_2^{(3\cdot 2^{i-1})}  \cdots t_2^{(k+1-2^{i-1})}.\] To explain the index of the last subsequence of every pair, since $k+1$ is a power of two, we know that $k-1$ is divisible by 2, $k-3$ is divisible by 4, and more generally, $k + 1 - 2^{i-1}$ is divisible by $2^{i-1}$.
	
We claim that there are a total of $N = \log(k + 1)$ such sequence pairs. To see why, note that each sequence in the first pair $S_1^{(1)}$ and $S_2^{(1)}$ consists of $k$ subsequences. Each sequence in the second pair $S_1^{(2)}$ and $S_2^{(2)}$ consists of $\frac{k-1}{2}$ subsequences. More generally,  each sequence in the $i^{th}$ pair $S_1^{\left(k+1-2^{i-1}\right)}$ and $S_2^{\left(k + 1-2^{i-1} \right)}$ consists of $\frac{k + 1-2^{i-1}}{2^{i-1}}$ subsequences. The final pair will will consist of only one subsequence, so $\frac{k + 1-2^{i-1}}{2^{i-1}} = 1$, or in other words $i = \log (k+1)$.
		
		We also claim that each sequence has length at most $n$. This is because the longest sequence pair is the first, $\left(S_1^{(1)}, S_2^{(1)}\right)$. By definition of the subsequences $t_j^{(i)}$, these two sequences are of length $\sum_{i = 1}^k (i + 2) = \frac{1}{2}k (k+5) \leq n.$ Therefore, all $N$ sequence pairs are of length at most $n$.
	
	\begin{example}\label{ex:seq_lb}
		Suppose that $n = 128$. Then $k=2^{\left\lfloor \log  \sqrt{n/2}\right\rfloor} - 1= 7$ and $N = \log (k+1)= 3$. The three sequence pairs have the following form\footnote{The maximum length of these six strings is 42, which is smaller than 128, as required.}:
		{\small\begin{align*}
				S_1^{(1)} &= \texttt{a}_1\texttt{b}_1\texttt{d}_1\texttt{a}_2\texttt{a}_2\texttt{b}_2\texttt{d}_2\texttt{a}_3\texttt{a}_3 \texttt{a}_3\texttt{b}_3\texttt{d}_3\texttt{a}_4\texttt{a}_4\texttt{a}_4\texttt{a}_4\texttt{b}_4\texttt{d}_4\texttt{a}_5\texttt{a}_5\texttt{a}_5\texttt{a}_5\texttt{a}_5\texttt{b}_5\texttt{d}_5\texttt{a}_6\texttt{a}_6\texttt{a}_6\texttt{a}_6\texttt{a}_6\texttt{a}_6\texttt{b}_6\texttt{d}_6\texttt{a}_7\texttt{a}_7\texttt{a}_7\texttt{a}_7\texttt{a}_7\texttt{a}_7\texttt{a}_7\texttt{b}_7\texttt{d}_7\\
				S_2^{(1)} &= \texttt{b}_1\texttt{c}_1\texttt{d}_1 \texttt{b}_2\texttt{c}_2\texttt{c}_2\texttt{d}_2\texttt{b}_3\texttt{c}_3\texttt{c}_3 \texttt{c}_3\texttt{d}_3\texttt{b}_4\texttt{c}_4\texttt{c}_4\texttt{c}_4\texttt{c}_4\texttt{d}_4\texttt{b}_5\texttt{c}_5\texttt{c}_5\texttt{c}_5\texttt{c}_5\texttt{c}_5\texttt{d}_5\texttt{b}_6\texttt{c}_6\texttt{c}_6\texttt{c}_6\texttt{c}_6\texttt{c}_6\texttt{c}_6\texttt{d}_6\texttt{b}_7\texttt{c}_7\texttt{c}_7\texttt{c}_7\texttt{c}_7\texttt{c}_7\texttt{c}_7\texttt{c}_7\texttt{d}_7\\
				S_1^{(2)} &= \texttt{a}_2\texttt{a}_2\texttt{b}_2\texttt{d}_2\texttt{a}_4\texttt{a}_4\texttt{a}_4\texttt{a}_4\texttt{b}_4\texttt{d}_4\texttt{a}_6\texttt{a}_6\texttt{a}_6\texttt{a}_6\texttt{a}_6\texttt{a}_6\texttt{b}_6\texttt{d}_6\\
				S_2^{(2)} &= \texttt{b}_2\texttt{c}_2\texttt{c}_2\texttt{d}_2\texttt{b}_4\texttt{c}_4\texttt{c}_4\texttt{c}_4\texttt{c}_4\texttt{d}_4\texttt{b}_6\texttt{c}_6\texttt{c}_6\texttt{c}_6\texttt{c}_6\texttt{c}_6\texttt{c}_6\texttt{d}_6\\
				S_1^{(3)} &= \texttt{a}_4\texttt{a}_4\texttt{a}_4\texttt{a}_4\texttt{b}_4\texttt{d}_4\\
				S_2^{(3)} &= \texttt{b}_4\texttt{c}_4\texttt{c}_4\texttt{c}_4\texttt{c}_4\texttt{d}_4
		\end{align*}}
	\end{example}
	
	\paragraph{\textbf{A ground-truth alignment for every sequence pair.}} To define a ground-truth alignment for all $N$ sequence pairs, we first define two alignments per subsequence pair $\left(t_1^{(i)}, t_2^{(i)}\right)$. The resulting ground-truth alignments will be a concatenation of these alignments. The first alignment, which we denote as $\left(h_1^{(i)}, h_2^{(i)}\right)$, is defined as follows: $h_1^{(i)}$ begins with $i$ $\texttt{a}_i$s, followed by 1 $\texttt{b}_i$, followed by $i$ gap characters, followed by 1 $\texttt{d}_i$; $h_2^{(i)}$ begins with $i$ gap characters, followed by 1 $\texttt{b}_i$, followed by $i$ $\texttt{c}_i$s, followed by 1 $\texttt{d}_i$. For example, \[\begin{matrix}h_1^{(3)}=\texttt{a}_3&\texttt{a}_3&\texttt{a}_3&\texttt{b}_3&\texttt{-}&\texttt{-}&\texttt{-} & \texttt{d}_3 \\
		h_2^{(3)}=\texttt{-}&\texttt{-}&\texttt{-}&\texttt{b}_3&\texttt{c}_3&\texttt{c}_3&\texttt{c}_3 & \texttt{d}_3 \end{matrix}.\] The second alignment, which we denote as $\left(\ell_1^{(i)}, \ell_2^{(i)}\right)$, is defined as follows: $\ell_1^{(i)}$ begins with $i$ $\texttt{a}_i$s, followed by 1 $\texttt{b}_i$, followed by $i$ gap characters, followed by 1 $\texttt{d}_i$; $\ell_2^{(i)}$ begins with 1 $\texttt{b}_i$, followed by $i$ gap characters, followed by $i$ $\texttt{c}_i$s, followed by 1 $\texttt{d}_i$. For example, \[\begin{matrix}\ell_1^{(3)}=\texttt{a}_3&\texttt{a}_3&\texttt{a}_3&\texttt{b}_3&\texttt{-}&\texttt{-}&\texttt{-}& \texttt{d}_3\\
		\ell_2^{(3)}=\texttt{b}_3&\texttt{-}&\texttt{-}&\texttt{-}&\texttt{c}_3&\texttt{c}_3&\texttt{c}_3& \texttt{d}_3 \end{matrix}.\]
	
	We now use these $2k$ alignments to define a ground-truth alignment $L_*^{(i)}$ per sequence pair $\left(S_1^{(i)}, S_2^{(i)}\right)$. Beginning with the first pair $\left(S_1^{(1)}, S_2^{(1)}\right)$, where \[S_1^{(1)} = t_1^{(1)} t_1^{(2)} t_1^{(3)}  \cdots t_1^{(k)} \quad \text{ and } \quad S_2^{(1)} = t_2^{(1)} t_2^{(2)} t_2^{(3)}  \cdots t_2^{(k)},\] we define the alignment of $S_1^{(1)}$ to be $\ell_1^{(1)}h_1^{(2)}\ell_1^{(3)}h_1^{(4)}\cdots\ell_1^{(k)}$ and we define the alignment of $S_2^{(1)}$ to be $\ell_2^{(1)}h_2^{(2)}\ell_2^{(3)}h_2^{(4)}\cdots\ell_2^{(k)}$. Moving on to the second pair $\left(S_1^{(2)}, S_2^{(2)}\right)$, where \[S_1^{(2)} = t_1^{(2)} t_1^{(4)} t_1^{(6)}  \cdots t_1^{(k-1)} \quad \text{ and } \quad S_2^{(2)} = t_2^{(2)} t_2^{(4)} t_2^{(6)}  \cdots t_2^{(k-1)},\] we define the alignment of $S_1^{(2)}$ to be $\ell_1^{(2)}h_1^{(4)}\ell_1^{(6)}h_1^{(8)}\cdots\ell_1^{(k-1)}$ and we define the alignment of $S_2^{(1)}$ to be $\ell_2^{(2)}h_2^{(4)}\ell_2^{(6)}h_2^{(8)}\cdots\ell_2^{(k-1)}$. Generally speaking, each pair $\left(S_1^{(i)}, S_2^{(i)}\right)$, where \[S_1^{(i)} = t_1^{(2^{i-1})} t_1^{(2 \cdot 2^{i-1})} t_1^{(3\cdot 2^{i-1})}  \cdots t_1^{(k-2^{i-1} + 1)} \quad \text{ and } \quad S_2^{(i)} = t_2^{(2^{i-1})} t_2^{(2 \cdot 2^{i-1})} t_2^{(3\cdot 2^{i-1})}  \cdots t_2^{(k-2^{i-1} + 1)},\]
	is made up of $\frac{k+1}{2^{i-1}} - 1$ subsequences. Since $k+1$ is a power of two, this number of subsequences is odd. We define the alignment of $S_1^{(i)}$ to alternate between alignments of type $\ell_1^{(j)}$ and alignments of type $h_1^{(j')}$, beginning and ending with alignments of the first type. Specifically, the alignment $S_1^{(i)}$ is $\ell_1^{(2^{i-1})} h_1^{(2 \cdot 2^{i-1})} \ell_1^{(3\cdot 2^{i-1})} h_1^{(4\cdot 2^{i-1})} \cdots \ell_1^{(k-2^{i-1} + 1)}$. Similarly, we define the alignment of $S_2^{(i)}$ to be \[\ell_2^{(2^{i-1})} h_2^{(2 \cdot 2^{i-1})} \ell_2^{(3\cdot 2^{i-1})} h_2^{(4\cdot 2^{i-1})} \cdots \ell_2^{(k-2^{i-1} + 1)}.\]
	
	\paragraph{\textbf{The $N$ witnesses.}} We define the $N$ values that witness the shattering of these $N$ sequence pairs to be $z_1 = z_2 = \cdots = z_N = \frac{3}{4}$.
	
	\paragraph{\textbf{Shattering the $N$ sequence pairs.}} Our goal is to show that for any subset $T \subseteq[N]$, there exists an indel penalty parameter $\rho[T]$ such that if $i \in [T]$, then $u_{0, \rho[T], 0}\left(S_1^{(i)}, S_2^{(i)}\right) < \frac{3}{4}$ and if $i \not\in [T]$, then $u_{0, \rho[T], 0}\left(S_1^{(i)}, S_2^{(i)}\right) \geq \frac{3}{4}$. To prove this, we will use two helpful claims, Claims~\ref{claim:d_match} and \ref{claim:b_match}.
	
	\begin{claim}\label{claim:d_match}
		For any pair $\left(S_1^{(i)}, S_2^{(i)}\right)$ and indel parameter $\rho[2] \geq 0$, there exists an alignment \[\alignment \in \argmax_{\alignment'}  \mt\left(S_1^{(i)}, S_2^{(i)}, \alignment'\right) - \rho[2] \cdot \id\left(S_1^{(i)}, S_2^{(i)}, \alignment'\right)\] such that each $\texttt{d}_j$ character in $S_1^{(i)}$ is matched to $\texttt{d}_j$ in $S_2^{(i)}$.
	\end{claim}
	
	\begin{proof}
		Let	$\alignment_0 \in \argmax_{\alignment'} \mt\left(S_1^{(i)}, S_2^{(i)}, \alignment'\right) - \rho[2] \cdot \id\left(S_1^{(i)}, S_2^{(i)}, \alignment'\right)$ be an alignment such that some $\texttt{d}_j$ character in $S_1^{(i)}$ is not matched to $\texttt{d}_j$ in $S_2^{(i)}$. Denote the alignment $L_0$ as $\left(\tau_1, \tau_2\right)$.
		Let $j \in \Z$ be the smallest integer such that for some indices $\ell_1 \not= \ell_2$, $\tau_1[\ell_1] = \texttt{d}_j$ and $\tau_2[\ell_2] = \texttt{d}_j$.
		Next, let $\ell_0$ be the maximum index smaller than $\ell_{1}$ and $\ell_{2}$ such that $\tau_1[\ell_0] = \texttt{d}_{j'}$ for some $j' \not= j$. We illustrate $\ell_0, \ell_1,$ and $\ell_2$ in Figure~\ref{fig:d_match_1}.
		\begin{figure}[t]
			\centering
			\begin{subfigure}{\textwidth}
				\includegraphics[scale =.8]{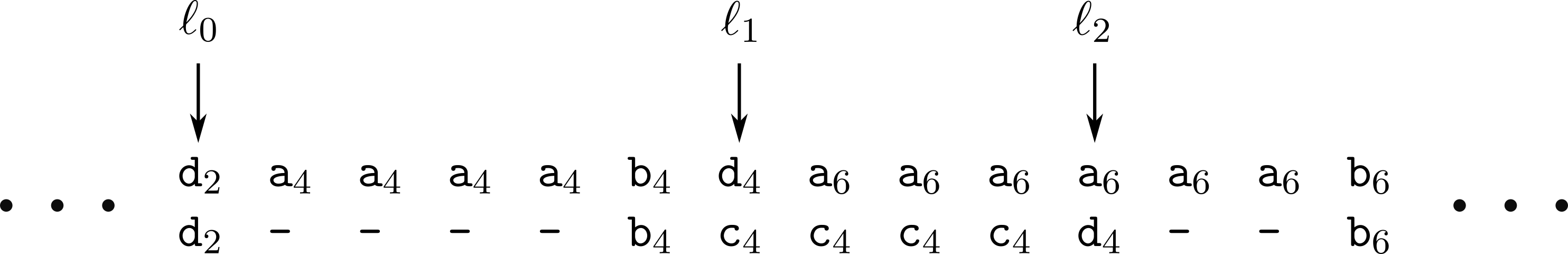}\centering
				\caption{An initial alignment where the $\texttt{d}_j$ characters are not aligned.}\label{fig:d_match_1}\vspace{7mm}
			\end{subfigure}
			\begin{subfigure}{\textwidth}
				\includegraphics[scale =.8]{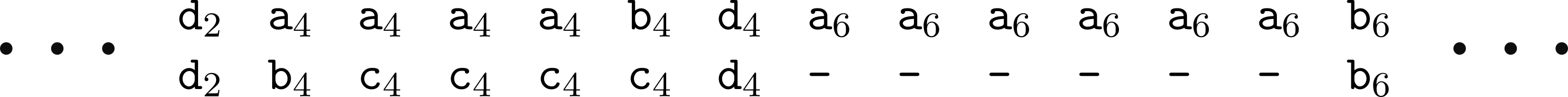}\centering
				\caption{An alignment where the gap characters in Figure~\ref{fig:d_match_1} are shifted such that the $\texttt{d}_j$ characters are aligned. The objective function value of both alignments is the same.}\label{fig:d_match_2}
			\end{subfigure}
			\caption{Illustration of Claim~\ref{claim:d_match}: we can assume that each $\texttt{d}_j$ character in $S_1^{(i)}$ is matched to $\texttt{d}_j$ in $S_2^{(i)}$.}
			\label{fig:d_match}
		\end{figure}
		By definition of $j$, we know that $\tau_1[\ell_0] = \tau_2[\ell_0] = \texttt{d}_{j'}$. We also know there is at least one gap character in $\left\{\tau_1[i] : \ell_0 + 1 \leq i \leq \ell_{1}\right\} \cup \left\{\tau_2[i] : \ell_0 + 1 \leq i \leq \ell_{2}\right\}$ because otherwise, the $\texttt{d}_j$ characters would be aligned in $L_0$. Moreover, we know there is at most one match among these elements between characters other than $\texttt{d}_j$ (namely, between the character $\texttt{b}_j$). If we rearrange all of these gap characters so that they fall directly after $\texttt{d}_j$ in both sequences, as in Figure~\ref{fig:d_match_2}, then we may lose the match between the character $\texttt{b}_j$, but we will gain the match between the character $\texttt{d}_j$. Moreover, the number of indels remains the same, and all matches in the remainder of the alignment will remain unchanged. Therefore, this rearranged alignment has at least as high an objective function value as $L_0$, so the claim holds.
	\end{proof}
	
	Based on this claim, we will assume, without loss of generality, that for any pair $\left(S_1^{(i)}, S_2^{(i)}\right)$ and indel parameter $\rho[2] \geq 0$, under the alignment $\alignment = A_{0, \rho[2], 0}\left(S_1^{(i)}, S_2^{(i)}\right)$ returned by the algorithm $A_{0, \rho[2], 0}$, all $\texttt{d}_j$ characters in $S_1^{(i)}$ are matched to $\texttt{d}_j$ in $S_2^{(i)}$.

	\begin{claim}\label{claim:b_match}
		Suppose that the character $\texttt{b}_j$ is in $S_1^{(i)}$ and $S_2^{(i)}$. The $\texttt{b}_j$ characters will be matched in $\alignment = A_{0, \rho[2], 0}\left(S_1^{(i)}, S_2^{(i)}\right)$ if and only if $\rho[2] \leq \frac{1}{2j}$.
	\end{claim}
	
	\begin{proof}
		Since all $\texttt{d}_j$ characters are matched in $L$,	in order to match $\texttt{b}_j$, it is necessary to add exactly $2j$ gap characters: all $2j$ $\texttt{a}_j$ and $\texttt{c}_j$ characters must be matched with gap characters. Under the objective function $\mt\left(S_1^{(i)}, S_2^{(i)}, \alignment'\right) - \rho[2] \cdot \id\left(S_1^{(i)}, S_2^{(i)}, \alignment'\right)$, this one match will be worth the $2j\rho[2]$ penalty if and only if $1 \geq 2j\rho[2]$, as claimed.
	\end{proof}
	
	We now use Claims~\ref{claim:d_match} and \ref{claim:b_match} to prove that we can shatter the $N$ sequence pairs $\left(S_1^{(1)}, S_2^{(1)}\right)$, \dots, $\left(S_1^{(N)}, S_2^{(N)}\right)$.

	\begin{claim}
		There are $\frac{k+1}{2^{i-1}} - 1$ thresholds $\frac{1}{2(k + 1) - 2^i} < \frac{1}{2(k + 1) - 2 \cdot 2^i} < \frac{1}{2(k + 1) - 3 \cdot 2^i} < \cdots < \frac{1}{2^i}$ such that as $\rho[2]$ ranges from 0 to 1, when $\rho[2]$ crosses one of these thresholds, $u_{0, \rho[2], 0}\left(S_1^{(i)}, S_2^{(i)}\right)$ switches from above $\frac{3}{4}$ to below $\frac{3}{4}$, or vice versa, beginning with $u_{0, 0, 0}\left(S_1^{(i)}, S_2^{(i)}\right) < \frac{3}{4}$ and ending with $u_{0, 1, 0}\left(S_1^{(i)}, S_2^{(i)}\right) > \frac{3}{4}$.
	\end{claim}
	\begin{proof}
		Recall that \[S_1^{(i)} = t_1^{(2^{i-1})} t_1^{(2 \cdot 2^{i-1})} t_1^{(3\cdot 2^{i-1})}  \cdots t_1^{(k-2^{i-1} + 1)} \quad \text{ and } \quad S_2^{(i)} = t_2^{(2^{i-1})} t_2^{(2 \cdot 2^{i-1})} t_2^{(3\cdot 2^{i-1})}  \cdots t_2^{(k-2^{i-1} + 1)},\] so
		in $S_1^{(i)}$ and $S_2^{(i)}$, the $\texttt{b}_j$ characters are $\texttt{b}_{2^{i-1}}, \texttt{b}_{2 \cdot 2^{i-1}}, \texttt{b}_{3 \cdot 2^{i-1}}, \dots, \texttt{b}_{k - 2^{i-1} + 1}$. Also, the reference alignment of $S_1^{(i)}$ is $\ell_1^{(2^{i-1})} h_1^{(2 \cdot 2^{i-1})} \ell_1^{(3\cdot 2^{i-1})} h_1^{(4\cdot 2^{i-1})} \cdots \ell_1^{(k-2^{i-1} + 1)}$ and the reference alignment of $S_2^{(i)}$ is \[\ell_2^{(2^{i-1})} h_2^{(2 \cdot 2^{i-1})} \ell_2^{(3\cdot 2^{i-1})} h_2^{(4\cdot 2^{i-1})} \cdots \ell_2^{(k-2^{i-1} + 1)}.\] We know that when the indel penalty $\rho[2]$ is equal to zero, all $\texttt{d}_j$ characters will be aligned, as will all $\texttt{b}_j$ characters. This means we will correctly align all $\texttt{d}_j$ characters and we will correctly align all $\texttt{b}_j$ characters in the $\left(h_1^{(j)}, h_2^{(j)}\right)$ pairs, but we will incorrectly align the $\texttt{b}_j$ characters in the $\left(\ell_1^{(j)}, \ell_2^{(j)}\right)$ pairs. The number of $\left(h_1^{(j)}, h_2^{(j)}\right)$ pairs in this reference alignment is $\frac{k+1}{2^{i}} -1$ and the number of $\left(\ell_1^{(j)}, \ell_2^{(j)}\right)$ pairs is $\frac{k+1}{2^{i}}$. Therefore, the utility of the alignment that maximizes the number of matches equals the following: \[u_{0, 0, 0}\left(S_1^{(i)}, S_2^{(i)}\right) = \frac{\frac{k+1}{2^{i-1}} -1 + \frac{k+1}{2^{i}} -1}{\frac{k+1}{2^{i-2}} -2} = \frac{3(k +1) - 2^{i+1}}{4(k +1) - 2^{i+1}} < \frac{3}{4},\] where the final inequality holds because $2^{i+1} \leq 2(k+1) < 3 (k+1)$.
		
		Next, suppose we increase $\rho[2]$ to lie in the interval $\left(\frac{1}{2(k + 1) - 2^i}, \frac{1}{2(k + 1) - 2 \cdot 2^i}\right]$. Since it is no longer the case that $\rho[2] \leq \frac{1}{2\left(k - 2^{i-1} + 1\right)}$, we know that the $\texttt{b}_{k - 2^{i-1} + 1}$ characters will no longer be matched, and thus we will correctly align this character according to the reference alignment. This means we will correctly align all $\texttt{d}_j$ characters and we will correctly align all $\texttt{b}_j$ characters in the $\left(h_1^{(j)}, h_2^{(j)}\right)$ pairs, but we will incorrectly align all but one of the $\texttt{b}_j$ characters in the $\left(\ell_1^{(j)}, \ell_2^{(j)}\right)$ pairs. Therefore, the utility of the alignment that maximizes $\mt(S_1, S_2, \alignment) - \rho[2] \cdot \id(S_1, S_2, \alignment)$ is  \[u_{0, \rho[2], 0}\left(S_1^{(i)}, S_2^{(i)}\right) = \frac{\frac{k+1}{2^{i-1}} -1 + \frac{k+1}{2^{i}}}{\frac{k+1}{2^{i-2}} -2} = \frac{3(k +1) - 2^{i}}{4(k +1) - 2^{i+1}} > \frac{3}{4},\] where the final inequality holds because $2^{i+1} \leq 2(k + 1) < 4(k+1)$.
		
		Next, suppose we increase $\rho[2]$ to lie in the interval $\left(\frac{1}{2(k + 1) - 2\cdot 2^i}, \frac{1}{2(k + 1) - 3 \cdot 2^i}\right]$. Since it is no longer the case that $\rho[2] \leq \frac{1}{2\left(k - 2\cdot 2^{i-1} + 1\right)}$, we know that the $\texttt{b}_{k - 2\cdot 2^{i-1} + 1}$ characters will no longer be matched, and thus we will incorrectly align this character according to the reference alignment. This means we will correctly align all $\texttt{d}_j$ characters and we will correctly align the $\texttt{b}_j$ characters in all but one of the $\left(h_1^{(j)}, h_2^{(j)}\right)$ pairs, but we will incorrectly align all but one of the $\texttt{b}_j$ characters in the $\left(\ell_1^{(j)}, \ell_2^{(j)}\right)$ pairs. Therefore, the utility of the alignment that maximizes $\mt(S_1, S_2, \alignment) - \rho[2] \cdot \id(S_1, S_2, \alignment)$ is \[u_{0, \rho[2], 0}\left(S_1^{(i)}, S_2^{(i)}\right) = \frac{\frac{k+1}{2^{i-1}} -1 + \frac{k+1}{2^{i}}}{\frac{k+1}{2^{i-2}} -2} > \frac{3}{4}.\]
		
		In a similar fashion, every time $\rho[2]$ crosses one of the thresholds $\frac{1}{2(k + 1) - 2^i} < \frac{1}{2(k + 1) - 2 \cdot 2^i} < \frac{1}{2(k + 1) - 3 \cdot 2^i} < \cdots < \frac{1}{2^i}$, the utility will shift from above $\frac{3}{4}$ to below or vice versa, as claimed.
	\end{proof}
	
	The above claim demonstrates that the $N$ sequence pairs are shattered, each with the witness $\frac{3}{4}$. After all, for every $i \in \{2, \dots, N\}$ and every interval $\left(\frac{1}{2(k+1) -j2^i}, \frac{1}{2(k+1) -(j + 1)2^i}\right)$ where $u_{0, \rho[2], 0}\left(S_1^{(i)}, S_2^{(i)}\right)$ is uniformly above or below $\frac{3}{4}$, there exists a subpartition of this interval into the two intervals \[\left(\frac{1}{2(k+1) -j2^i}, \frac{1}{2(k+1) -(2j + 1)2^{i-1}}\right) \text{ and }\left(\frac{1}{2(k+1) -(2j+1)2^{i-1}}, \frac{1}{2(k+1) -(j + 1)2^i}\right)\] such that in the first interval, $u_{0, \rho[2], 0}\left(S_1^{(i - 1)}, S_2^{(i - 1)}\right) < \frac{3}{4}$ and in the second, $u_{0, \rho[2], 0}\left(S_1^{(i - 1)}, S_2^{(i - 1)}\right) > \frac{3}{4}$. Therefore, for any subset $T \subseteq[N]$, there exists an indel penalty parameter $\rho[T]$ such that if $i \in [T]$, then $u_{0, \rho[T], 0}\left(S_1^{(i)}, S_2^{(i)}\right) < \frac{3}{4}$ and if $i \not\in [T]$, then $u_{0, \rho[T], 0}\left(S_1^{(i)}, S_2^{(i)}\right) > \frac{3}{4}$.
\end{proof}

\subsection{Tighter guarantees for a structured algorithm subclass: sequence alignment using hidden Markov models}
While we focused on the affine gap model in the previous section, which was inspired by the results in~\citet*{Gusfield94:Parametric}, 
the result in~\citet*{Pachter04:Parametric} helps to provide uniform convergence guarantees for any alignment scoring function that can be modeled as a \emph{hidden Markov model (HMM)}.
A bound on the number of parameter choices that emit distinct sets of co-optimal alignments in that work is found by taking an algebraic view of the alignment HMM with $d$ tunable parameters.
In fact, the bounds provided can be used to provide guarantees for many types of HMMs.

\begin{lemma}\label{lem:sequence_tighter_highd}
Let $\left\{A_{\vec{\rho}} \mid \vec{\rho} \in \R^d\right\}$ be a set of co-optimal-constant algorithms and let $u$ be a utility function mapping tuples $(S_1, S_2, L)$ of sequence pairs and alignments to the interval $[0,1]$. Let $\cU$ be the set of functions $\cU = \left\{u_{\vec{\rho}} : (S_1, S_2) \mapsto u\left(S_1, S_2, A_{\vec{\rho}}\left(S_1, S_2\right)\right) \mid \vec{\rho} \in \R\right\}$ mapping sequence pairs $S_1, S_2 \in \Sigma^n$ to $[0,1]$.
For some constant $c_1 > 0$, the dual class $\cU^*$ is $\left(\cF, \cG, c_1^2n^{2d(d-1)/(d+1)}\right)$-piecewise decomposable, where $\cG = \{g_{\vec{a}} : \cU \to \{0,1\} \mid \vec{a} \in \R^{d+1}\}$ consists of halfspace indicator functions $g_{\vec{a}} : u_{\vec{\rho}} \mapsto \ind{a_1\rho[1] + ... + a_d \rho[d] < a_{d+1}}$ and $\cF = \{f_c : \cU \to \R \mid c \in \R\}$ consists of constant functions $f_c : u_{\vec{\rho}} \mapsto c$.
\end{lemma}

\begin{proof}
Fix a sequence pair $S_1$ and $S_2$ and consider the function $u_{S_1,S_2}^* : \cU \to \R$ from the dual class $\cU^*$, where $u_{S_1, S_2}^*(u_{\vec{\rho}}) = u_{\vec{\rho}}(S_1, S_2)$. 
Consider the set of alignments $\cL_{S_1, S_2} = \{A_{\vec{\rho}}(S_1, S_2) \mid \vec{\rho} \in \R^d\}$.
There are at most 
$O\left(n^{d(d-1)/(d+1)}\right)$ 
sets of co-optimal solutions as we range $\vec{\rho}$ over $\R^d$ \citep{Pachter04:Parametric}.
The remainder of the proof is analogous to that for Lemma~\ref{lem:sequence_tighter}.
\end{proof}

Finally the results of Lemma~\ref{lem:sequence_tighter_highd} imply the following pseudo-dimension bound.

\begin{cor}\label{cor:sequence_tighter_highd}
Let $\left\{A_{\vec{\rho}} \mid \vec{\rho} \in \R^d\right\}$ be a set of co-optimal-constant algorithms and let $u$ be a utility function mapping tuples $(S_1, S_2, L)$ to $[0,H]$. Let $\cU$ be the set of functions \[\cU = \left\{u_{\vec{\rho}} : (S_1, S_2) \mapsto u\left(S_1, S_2, A_{\vec{\rho}}\left(S_1, S_2\right)\right) \mid \vec{\rho}\in \R^d\right\}\] mapping sequence pairs $S_1, S_2 \in \Sigma^n$ to $[0,1]$. Then $\pdim(\cU) = O\left(d^2 \ln n\right).$
\end{cor}

%% file: appendix_multisequence.tex
\begin{algorithm}
\caption{Progressive Alignment Algorithm \textsc{ProgressiveAlignment}}
\label{alg:progressivealignment}
	\begin{algorithmic}
\Require{Binary guide tree $G$, pairwise sequence alignment algorithm $A_{\vec{\rho}}$}
\State {Let $v_1, \dots, v_m$ be an ordering of the nodes in $G$ from deepest to shallowest, with nodes of the same depth ordered arbitrarily}
\For{$i \in \{1, \dots, m\}$}  \Comment {Compute the consensus sequences}
	  \If{$v_i$ is a leaf}
	\State{Set $\sigma'_{v_i}$ to be the leaf's sequence}
	\Else
	\State{Let $c_1$ and $c_2$ be the children of $v_i$}
	\State{Compute the pairwise alignment $L'_{v_i}= A_{\vec{\rho}}\left(\sigma'_{c_1}, \sigma'_{c_2}\right)$}
	\State{Set $\sigma'_{v_i}$ to be the consensus sequence of $L'_{v_i}$ (as in Definition~\ref{def:consensus})}
	\EndIf
\EndFor

\State{set $\sigma_{v_m} = \sigma'_{v_m}$}\Comment{note $v_m$ is the root of $G$}
\For{$i \in \{m, \dots, 1\}$}  \Comment {Compute the alignment sequences}
	\If{$v_i$ is not a leaf}
	\State{Let $L'_{v_i} = \left(\tau'_1, \tau'_2\right)$ be the alignment sequences computed at $v_i$}
	\State{Let $c_1$ and $c_2$ be the children of $v_i$}
	\State{Set $\sigma_{c_1} = \sigma_{c_2} = $ ``''}
	\State{Set $k = 0$}
	\For{$j \in \left[\left|\sigma_{v_i}\right|\right]$}
		\If{$\sigma_{v_i}[j] = $ `-'}
		   \State{Append `-' to the end of both $\sigma_{c_1}$ and $\sigma_{c_2}$}
		\Else
		    \State{Append  $\tau'_1[k]$ to the end of  $\sigma_{c_1}$}
		    \State{Append  $\tau'_2[k]$ to the end of  $\sigma_{c_2}$}
		    \State{Increment $k$ by 1}
		\EndIf
	\EndFor
	\EndIf
\EndFor	
	\For{$i \in \{1, \dots, m\}$} \Comment{Construct the final alignment}
	  \If{$v_i$ is a leaf representing sequence $S_j$}
	    \State{Set $\tau_j = \sigma_{v_i}$}
	   \EndIf
	 \EndFor
\end{algorithmic}
\end{algorithm}

	\begin{definition}\label{def:consensus}
		Let  $\left(\tau_1, \tau_2\right)$ be a sequence alignment. The \emph{consensus sequence} of this alignment is the sequence $\sigma \in \Sigma^*$ where $\sigma[j]$ is the
		most-frequent non-gap character in the $j^\text{th}$ position in the alignment (breaking ties in a fixed but arbitrary way). For example, the consensus sequence of the alignment \[\begin{bmatrix}
			\texttt{A} & \texttt{T} & \texttt{\texttt{-}} & \texttt{C}\\
			\texttt{G} & \texttt{-} &\texttt{C} & \texttt{C}\\\end{bmatrix}\] is \texttt{ATCC} when ties are broken in favor of \texttt{A} over \texttt{G}.
	\end{definition}

\begin{figure}[t]
\centering
\begin{subfigure}{.47\textwidth}
\includegraphics[width=\textwidth]{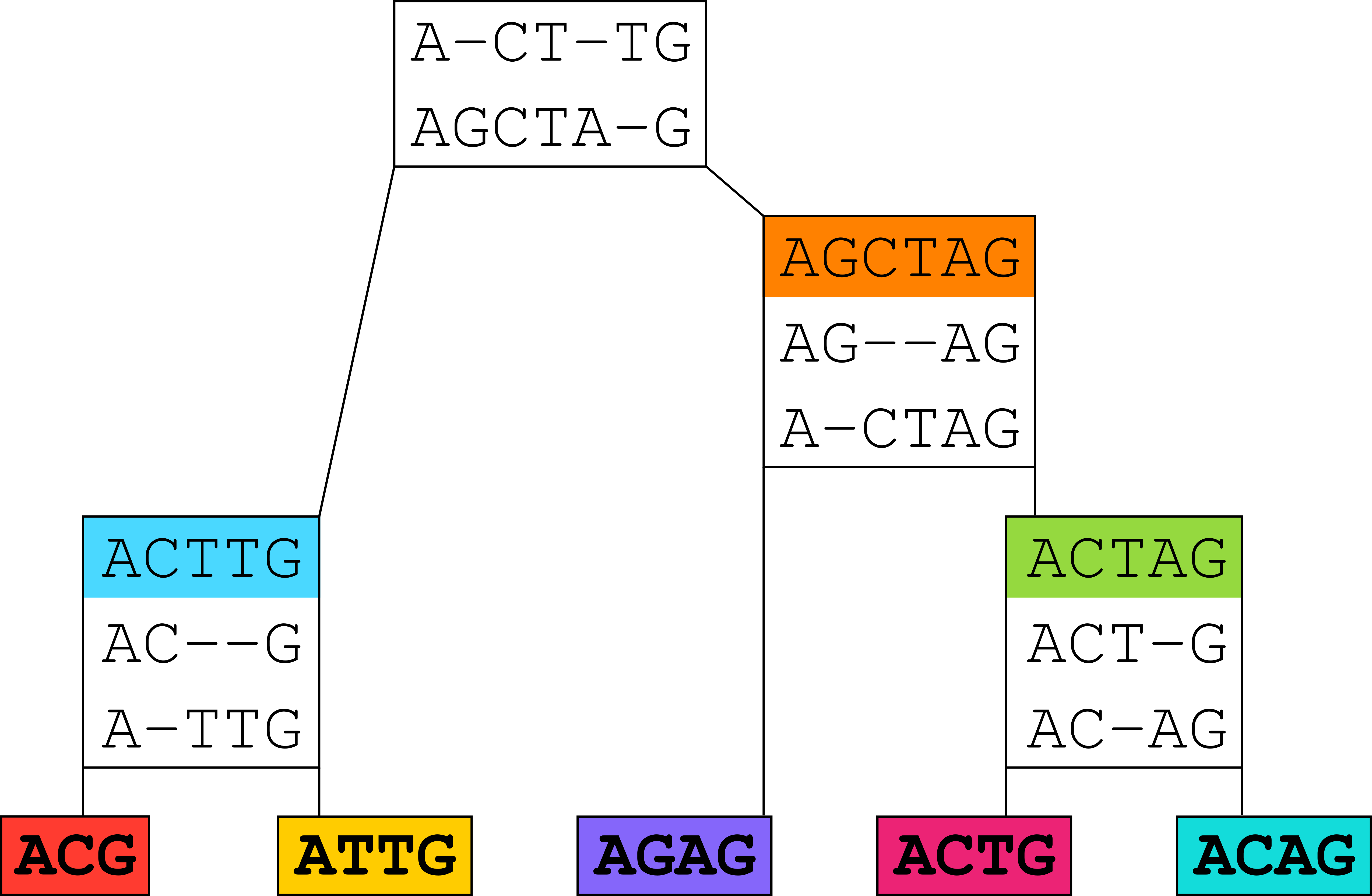}\centering
\caption{A guide tree}\label{fig:tree}
\end{subfigure}\qquad
\begin{subfigure}{.47\textwidth}
\includegraphics[width=\textwidth]{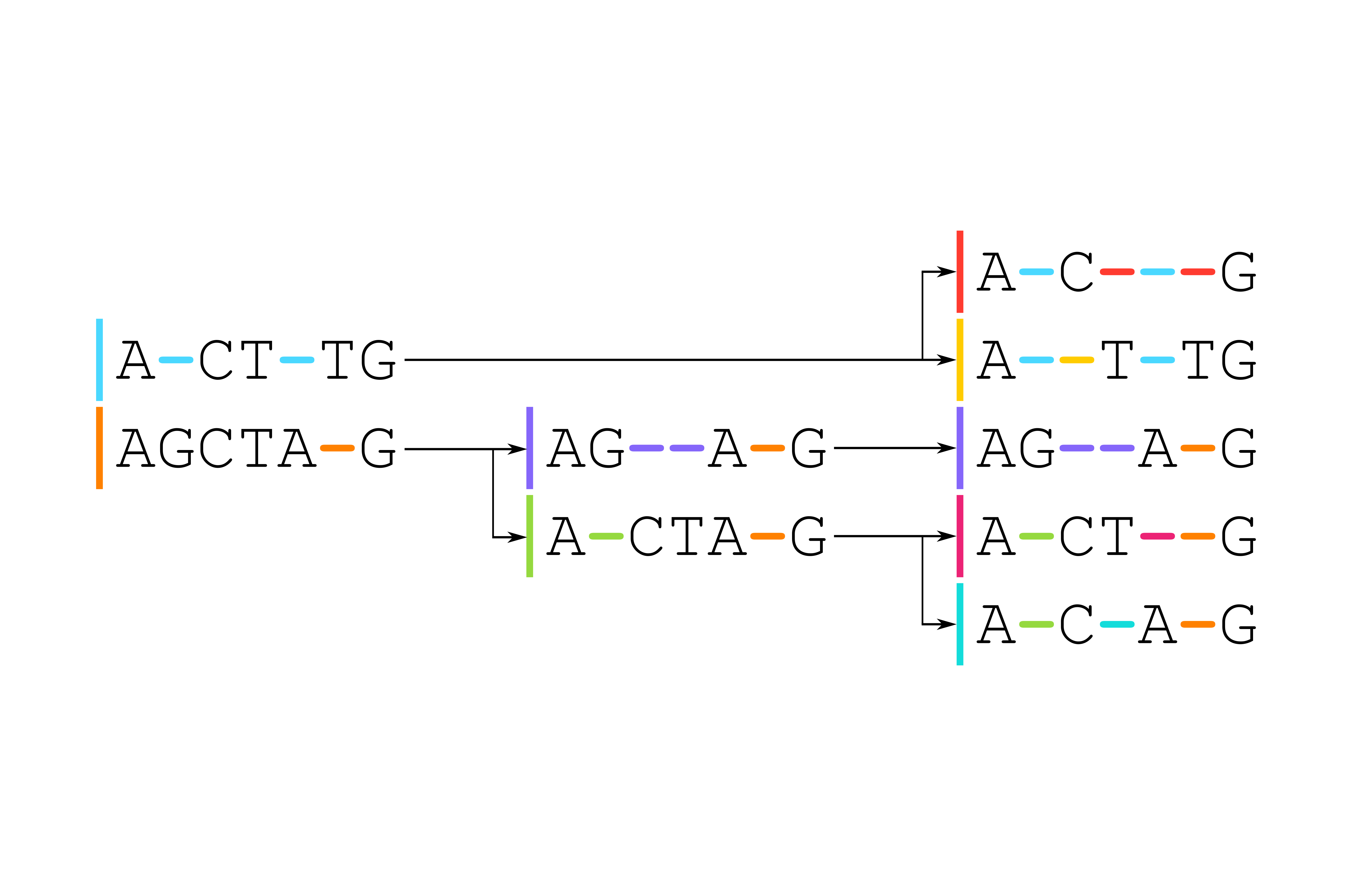}\centering
\caption{Constructing the alignment using the guide tree}\label{fig:alignment_from_tree}
\end{subfigure}
\caption{This figure illustrates an example of the progressive sequence alignment algorithm in action. Figure~\ref{fig:tree} depicts a completed guide tree. The five input sequences are represented by the leaves. Each internal leaf, depicts an alignment of the (consensus) sequences contained in the leaf's children. Each internal leaf other than the root also contains the consensus sequence corresponding to that alignment. Figure~\ref{fig:alignment_from_tree} illustrates how to extract an alignment of the five input strings (as well as the consensus strings) from Figure~\ref{fig:tree}.}
\label{fig:progressive}
\end{figure}
Figure~\ref{fig:progressive} illustrates an example of this algorithm in action, and corresponds to the psuedo-code given in Algorithm~\ref{alg:progressivealignment}.
The first loop matches with Figure~\ref{fig:tree}, the second and third match with Figure~\ref{fig:alignment_from_tree}.

\progressive*

\begin{proof}
	A key step in the proof of Lemma~\ref{lem:sequence} for pairwise alignment shows that for any pair of
	sequences $S_1, S_2 \in \Sigma^n$, we can find a set $\cH$ of $4^n n^{4n + 2}$
	hyperplanes such that for any pair $\vec{\rho}$ and $\vec{\rho'}$
	belonging to the same connected component of $\reals^d \setminus \cH$, we have
	$A_{\vec{\rho}}(S_1, S_2) = A_{\vec{\rho'}}(S_1, S_2)$. We use this result
	to prove
	the following claim.
	
	\begin{claim}\label{claim:induction}
		For each node $v$ in the guide tree, there is a set $\cH_v$ of  hyperplanes where for any connected component $R$ of $\R^d \setminus \cH_v$, the alignment and consensus sequence computed by $M_{\vec{\rho}}$ is fixed across all $\vec{\rho} \in R$. Moreover, the size of $\cH_v$ is bounded as follows:
		\[\left|\cH_v\right| \leq \ell^{d^{\height{v}}}\left(\ell4^d\right)^{\left(d^{\height{v}} - 1\right)/(d - 1)},\] where $\ell:= 4^{n\kappa} (n\kappa)^{4n\kappa + 2}$.
	\end{claim}
	
	Before we prove Claim~\ref{claim:induction}, we remark that the longest consensus
	sequence computed for any node $v$ of the guide tree has length at most
	$n \kappa$, which is a bound on the sum of the lengths of the
	input sequences.
	
	\begin{proof}[Proof of Claim~\ref{claim:induction}]
		We prove this claim by induction on the guide tree $G$. The base case corresponds to the leaves of $G$. On each leaf, the alignment and
		consensus sequence constructed by $M_{\vec{\rho}}$ is constant for all
		$\vec{\rho} \in \reals^d$, since there is only one string to align (i.e.,
		the input string placed at that leaf). Therefore, the claim holds for the
		leaves of $G$. Moving to an internal node $v$, suppose that the inductive hypothesis
		holds for its children $v_1$ and $v_2$. Assume without loss of generality that
		$\height{v_1} \geq \height{v_2}$, so that $\height{v} = \height{v_1} + 1$.
		Let $\cH_{v_1}$ and
		$\cH_{v_2}$ be the sets of hyperplanes corresponding to the children $v_1$ and $v_2$. By the inductive hypothesis, these sets are each of size at most \[s := \ell^{d^{\height{v_1}}}\left(\ell4^d\right)^{\left(d^{\height{v_1}} - 1\right)/(d - 1)}\]   Letting $\cH = \cH_{v_1} \cup \cH_{v_2}$, we are
		guaranteed that for every connected component of $\reals^d \setminus \cH$, the alignment
		and consensus string computed by $M_{\vec{\rho}}$ for both children
		$v_1$ and $v_2$ is constant. Based on work by \citet{Buck43:Partition}, we know that there are at most
		$(2s + 1)^d \leq (3s)^d$
		connected components of $\reals^d \setminus \cH$.
		For each region, by the same argument as in the proof of Lemma~\ref{lem:sequence}, there are an additional
		$\ell$ hyperplanes that partition
		the region into subregions where the outcome of the pairwise merge at node $v$ is
		constant.
		Therefore, there is a set $\cH_v$ of at most
		\begin{align*}
			\ell(3s)^d + 2s &\leq \ell(4s)^d\\
			&= \ell\left(4\ell^{d^{\height{v_1}}}\left(\ell4^d\right)^{\left(d^{\height{v_1}} - 1\right)/(d - 1)}\right)^d\\
			&= \ell^{d^{\height{v_1} + 1}}\left(\ell4^d\right)^{\left(d^{\height{v_1}+1} - d\right)/(d - 1) + 1}\\
			&= \ell^{d^{\height{v_1} + 1}}\left(\ell4^d\right)^{\left(d^{\height{v_1}+1} - 1\right)/(d - 1)}\\
			&= \ell^{d^{\height{v}}}\left(\ell4^d\right)^{\left(d^{\height{v}} - 1\right)/(d - 1)}
		\end{align*}  
		hyperplanes where for every connected component of $\reals^d \setminus \cH$, the alignment
		and consensus string computed by $M_{\vec{\rho}}$ at $v$ is invariant.
	\end{proof}
	Applying Claim~\ref{claim:induction} to the root
	of the guide tree, the function $\vec{\rho} \mapsto
	M_{\vec{\rho}}(S_1, \dots, S_\numseqs, G)$ is piecewise constant with
	$\ell^{d^{\height{G}}}\left(\ell4^d\right)^{\left(d^{\height{G}} - 1\right)/(d - 1)}$  linear boundary
	functions. The lemma then follows from the following chain of inequalities:
	\begin{align*}\ell^{d^{\height{G}}}\left(\ell4^d\right)^{\left(d^{\height{G}} - 1\right)/(d - 1)} &\leq \ell^{d^{\height{G}}}\left(\ell4^d\right)^{d^{\height{G}}}\\
		&= \ell^{2d^{\height{G}}}4^{d^{\height{G}+1}}\\
		&= \left(4^{n\kappa} (n\kappa)^{4n\kappa + 2}\right)^{2d^{\height{G}}}4^{d^{\height{G}+1}}\\
		&= \left(4^{n\kappa} \left(n\kappa\right)^{4n\kappa + 2}\right)^{2d^{\height{G}}}4^{d^{\height{G}+1}}\\
		&\leq \left(4^{n\kappa} \left(n\kappa\right)^{4n\kappa + 2}\right)^{2d^{\eta}}4^{d^{\eta+1}}.\end{align*}
\end{proof}